\documentclass[11pt,letterpaper]{article}
\usepackage[T1]{fontenc}
\usepackage[utf8]{inputenc}
\usepackage{amsmath,amsfonts,amsthm,amssymb}
\usepackage{setspace}
\usepackage{fullpage}
\usepackage{fancyhdr}
\usepackage{xspace, xcolor}

\usepackage{verbatim}
\usepackage{lastpage}
\usepackage{extramarks}
\usepackage{natbib}
\usepackage{diagbox}

\usepackage{titletoc}

\title{Provable Model-based Nonlinear Bandit and Reinforcement Learning: Shelve Optimism, Embrace Virtual Curvature
}

\allowdisplaybreaks
\usepackage{breqn}
\usepackage{times}
\usepackage{setspace}
\usepackage{graphicx,float,wrapfig}
\usepackage{algorithm}
\usepackage{algorithmicx}
\usepackage{enumitem}
\usepackage[noend]{algpseudocode}
\usepackage{subcaption}

\newtheorem{theorem}{Theorem}[section]
\newtheorem{lemma}[theorem]{Lemma}
\newtheorem{assumption}[theorem]{Assumption}
\newtheorem{corollary}[theorem]{Corollary}

\newtheorem{example}[theorem]{Example}
\theoremstyle{definition}

\def\shownotes{0}  %
\ifnum\shownotes=1
\newcommand{\authnote}[2]{[#1: #2]}
\else
\newcommand{\authnote}[2]{}
\fi

\usepackage{amsmath,amsfonts,bm}

\newcommand*{\defeq}{\triangleq}

\def\1{\bm{1}}

\def\eps{{\epsilon}}

\def\vtheta{{\bm{\theta}}}

\def\va{{\bm{a}}}

\makeatletter
\newcommand{\ve}{\@ifnextchar\bgroup{\velong}{{\bm{e}}}}
\newcommand{\velong}[1]{{\bm{#1}}}
\makeatother

\def\vx{{\bm{x}}}
\def\vy{{\bm{y}}}
\def\vz{{\bm{z}}}

\def\mG{{\bm{G}}}

\def\mW{W}

\DeclareMathAlphabet{\mathsfit}{\encodingdefault}{\sfdefault}{m}{sl}
\SetMathAlphabet{\mathsfit}{bold}{\encodingdefault}{\sfdefault}{bx}{n}

\def\calA{{\mathcal{A}}}
\def\calB{{\mathcal{B}}}

\def\calF{{\mathcal{F}}}
\def\calG{{\mathcal{G}}}
\def\calH{{\mathcal{H}}}
\def\calI{{\mathcal{I}}}

\def\calL{{\mathcal{L}}}

\def\calN{{\mathcal{N}}}
\def\calO{{\mathcal{O}}}
\def\calP{{\mathcal{P}}}

\def\calR{{\mathcal{R}}}
\def\calS{{\mathcal{S}}}

\def\calX{{\mathcal{X}}}
\def\calY{{\mathcal{Y}}}
\def\calZ{{\mathcal{Z}}}

\def\sB{{\mathbb{B}}}

\def\sR{{\mathbb{R}}}

\newcommand{\E}{\mathbb{E}}
\newcommand{\R}{\mathbb{R}}

\newcommand{\KL}{D_{\mathrm{KL}}}

\DeclareMathOperator*{\argmax}{argmax}

\DeclareMathOperator{\Tr}{Tr}

\newcommand\numberthis{\addtocounter{equation}{1}\tag{\theequation}}

\newcommand{\poly}{\mathrm{poly}}
\newcommand{\polylog}{\mathrm{polylog}}

\newcommand{\diag}{\mathrm{diag}}

\newcommand{\dotp}[2]{\left<#1, #2\right>}
\newcommand{\inner}[2]{\langle#1, #2\rangle}

\newcommand{\norm}[1]{\left\| #1 \right\|}
\newcommand{\normtwo}[1]{\left\| #1 \right\|_2}
\newcommand{\normone}[1]{\left\| #1 \right\|_1}
\newcommand{\normtwosm}[1]{\| #1 \|_2}
\newcommand{\normF}[1]{\left\| #1 \right\|_{\mathrm{F}}}
\newcommand{\normop}[1]{\left\| #1 \right\|_{\mathrm{op}}}
\newcommand{\normsp}[1]{\left\| #1 \right\|_{\mathrm{sp}}}
\newcommand{\normFsm}[1]{\| #1 \|_{\mathrm{F}}}

\newcommand{\normspsm}[1]{\| #1 \|_{\mathrm{sp}}}

\newcommand{\relu}{\mathrm{ReLU}}

\newcommand{\pbra}[1]{{\left( {#1} \right)}}
\newcommand{\bbra}[1]{{\left[ {#1} \right]}}

\newcommand{\bbrasm}[1]{{[ {#1} ]}}
\newcommand{\abs}[1]{{\left| {#1} \right|}}
\newcommand{\abssm}[1]{{| {#1} |}}

\newcommand{\mintwo}[2]{\min \left( #1, #2 \right)}
\newcommand{\maxtwo}[2]{\max \left( #1, #2 \right)}

\newcommand{\ind}[1]{\mathbb{I}\left[ #1 \right]}
\DeclareMathOperator{\grad}{grad}

\newcommand{\bigO}{\calO}
\newcommand{\tildeO}{\widetilde{\calO}}

\newcommand{\olreg}{\textsc{reg}^{\textsc{ol}}}

\newcommand{\reg}{\textsc{reg}}
\newcommand{\thetas}{\theta^\star}
\newcommand{\vthetas}{\vtheta^\star}

\newcommand{\sosp}{\mathfrak{A}}
\newcommand{\src}{\mathfrak{R}}

\newcommand{\srcrl}{\mathfrak{R}^{dyn}}
\newcommand{\da}{d_{\calA}}

\newcommand{\AlgBan}{ViOlin}

\newcommand{\err}{\Delta}
\newcommand{\emperr}{\tilde{\Delta}}
\newcommand{\emperrrl}{\bar{\Delta}}
\newcommand{\ubg}{\kappa_1}
\newcommand{\ubh}{\kappa_2}

\newcommand{\lmax}{\lambda_{\mathrm max}}
\newcommand{\ga}{\nabla_a}
\newcommand{\ha}{\nabla^2_a}

\newcommand{\gasup}{\zeta_g}
\newcommand{\hasup}{\zeta_h}
\newcommand{\tasup}{\zeta_{\mathrm{3rd}}}

\newcommand{\dy}{T}

\newcommand{\gpsi}{\nabla_\psi}
\newcommand{\hpsi}{\nabla^2_\psi}

\newcommand{\V}{V^{\psi}_{\theta}}
\newcommand{\Vs}[1]{V^{\psi_{#1}}_{\theta_{#1}}}
\newcommand{\G}{G^{\psi}_{\theta}}

\newcommand{\Pg}{\chi_g}
\newcommand{\Ph}{\chi_h}
\newcommand{\Pt}{\chi_f}

\newcommand{\ol}{\calR}

\newcommand{\NN}{\mathrm{N}}
\renewcommand{\vec}{\mathrm{vec}}
\newcommand{\dd}{\mathrm{d}}
\newcommand{\TV}[2]{\mathrm{TV}\pbra{#1, #2}}

\usepackage{hyperref}
\hypersetup{
	colorlinks   = true, %
	urlcolor     = red, %
	linkcolor    = blue, %
	citecolor   = blue %
}
\begin{document}

\renewcommand*{\ttdefault}{cmtt}

\author{Kefan Dong \\ 
Stanford University \\
\texttt{kefandong@stanford.edu}
\and
Jiaqi Yang \\
Tsinghua University\\
\texttt{yangjq17@gmail.com} \\
\and
Tengyu Ma \\
Stanford University \\
\texttt{tengyuma@stanford.edu}
}

\maketitle

\begin{abstract}
This paper studies model-based bandit and reinforcement learning (RL) with nonlinear function approximations. We propose to study convergence to approximate local maxima because we show that global convergence is statistically intractable even for one-layer neural net bandit with a deterministic reward. For both nonlinear bandit and RL, the paper presents a model-based algorithm, Virtual Ascent with Online Model Learner (\AlgBan), which provably converges to a local maximum with sample complexity that only depends on the sequential Rademacher complexity of the model class. 
Our results imply novel global or local regret bounds  on several concrete settings such as linear bandit with finite or sparse model class, and two-layer neural net bandit.
A key algorithmic insight is that optimism may lead to over-exploration even for two-layer neural net model class. On the other hand, for convergence to local maxima, it suffices to maximize the virtual return if the model can also reasonably predict the gradient and Hessian of the real return.
\end{abstract}

\section{Introduction}\label{sec:intro}

Recent progresses demonstrate many successful applications of deep reinforcement learning (RL) in robotics \citep{levine2016end}, games \citep{berner2019dota,alphago17}, computational biology \citep{mahmud2018applications}, etc. 
However, theoretical understanding of deep RL algorithms is limited. Last few years witnessed a plethora of results on linear function approximations in RL~\citep{zanette2020learning,shariff2020efficient,jin2019provably,wang2019optimism, wang2020provably,du2019q,agarwal2020flambe,hao2021online}, but the analysis techniques appear to strongly rely on (approximate) linearity and hard to generalize to neural networks.\footnote{Specifically, \citet{zanette2020learning} rely on closure of the value function class under bootstrapping. \citet{shariff2020efficient} rely on additional properties of the feature map, and \citet{jin2019provably,wang2019optimism,du2019q} use uncertainty quantification for linear regression. 
}

The goal of this paper is to theoretically analyze model-based nonlinear bandit and RL with neural net approximation, which achieves amazing sample-efficiency in practice (see e.g., \citep{janner2019trust,clavera2019model, hafner2019dream,hafner2019learning,dong2020expressivity}). We focus on the setting where the state and action spaces are continuous. %

Past theoretical work on model-based RL studies families of dynamics with restricted complexity measures such as Eluder dimension \citep{osband2014model}, witness rank \citep{sun2019model}, the linear dimensionality \citep{yang2020reinforcement}, and others \citep{modi2020sample,kakade2020information,du2021bilinear}. 
Implications of these complexity measures have been studied, e.g., finite mixture of dynamics \citep{ayoub2020model} and linear models~\citep{russo2013eluder} have bounded Eluder dimensions. However, it turns out that none of the complexity measures apply to the family of MDPs with even barely nonlinear dynamics, e.g., MDPs with dynamics parameterized by all one-layer neural network with a single activation unit (and with bounded weight norms).
For example, in Theorem~\ref{thm:lowerbound-eluder}, we will prove that one-layer neural nets do not have polynomially-bounded Eluder dimension.\footnote{This result is also proved by the concurrent \citet[Theorem 8]{li2021eluder} independently.} (See more evidence below.)

The limited progress on neural net approximation is to some extent not surprising. Given a deterministic dynamics with \textit{known} neural net parameters, finding the best parameterized policy still involves optimizing a complex non-concave function, which is in general computationally intractable. More fundamentally, we find that it is also \textit{statistically intractable} for solving the one-hidden-layer neural net bandit problem (which is a strict sub-case of deep RL). In other words, it requires exponential (in the input dimension) samples to find the global maximum (see Theorems~\ref{thm:lowerbound-sample-comp}).
This also shows that conditions in past work that guarantee global convergence cannot apply to neural nets.

Given these strong impossibility results, we propose to reformulate the problem to finding an approximate local maximum policy with guarantees. This is in the same vein as the recent fruitful paradigm in non-convex optimization where researchers disentangle the problem into showing that all local minima are good and fast convergence to local minima (e.g., see~\citep{ge2016matrix,ge2015escaping,ge2017no,ge2020optimization,lee2016gradient}). In RL, local maxima can often be global as well for many cases~\citep{agarwal2019theory}.\footnote{The all-local-maxima-are-global condition only needs to hold to the ground-truth total expected reward function. This potentially can allow disentangled assumptions on the ground-truth instance and the hypothesis class.}
 This paper focuses on sample-efficient convergence to an approximate local maximum. We consider the notion of local regret, which is measured against the worst $\eps$-approximate local maximum of the reward function (see Eq.~\eqref{equ:local-regret}).

Zero-order optimization or policy gradient algorithms can converge to local maxima and become natural potential competitors. They are widely believed to be less sample-efficient than the model-based approach because the latter can leverage the extrapolation power of the parameterized models. Theoretically, our formulation aims to characterize this phenomenon with results showing that the model-based approach's  sample complexity mostly depends (polynomially) on the complexity of the model class, whereas policy gradient algorithms' sample complexity polynomially depend on the dimensionality of policy parameters (in RL) or actions (in bandit).  
Our technical goal is to answer the following question:
\begin{center}
Can we design algorithms that converge to approximate local maxima with sample complexities that depend \textit{only and polynomially} on the complexity measure of the \textit{dynamics/reward} class?
\end{center}
We note that this question is open even if the dynamics hypothesis class is finite, and the complexity measure is the logarithm of its size. The question is also open even for nonlinear bandit problems (where dynamics class is replaced by reward function class), with which we start our research. We consider first nonlinear bandit with \textit{deterministic} reward where the reward function is given by $\eta(\theta,a)$ for action $a\in \calA$ under instance $\theta\in \Theta$.  We use sequential Rademacher complexity~\citep{rakhlin2015online,rakhlin2015sequential} to capture the complexity of the reward function $\eta$. Our main result for nonlinear bandit is stated as follows.

\begin{theorem}[Informal version of Theorem~\ref{thm:bandit-main}]
	There exists a model-based algorithm (\AlgBan, Alg.~\ref{alg:ban}) whose local regret, compared to $\Omega(\epsilon)$-approximate local maxima, is bounded by $\bigO(\sqrt{T\src_T}/\epsilon^2),$ where $\src_T$ is the sequential Rademacher complexity of a bounded loss function induced by the reward function class $\{\eta(\theta, \cdot): \theta\in \Theta\}$.
\end{theorem}
The sequential Rademacher complexity $\src_T$ is often bounded by the form $\tildeO(\sqrt{RT})$ for some parameter $R$ that measures the complexity of the hypothesis. When this happens, we have $\bigO(\sqrt{T\src_T}) = \tildeO(T^{3/4}) = o(T)$ local regret. 

In contrast to zero-order optimization, which does not use the parameterization of $\eta$ and has regret bounds depending on the action dimension, our regret only depends on the complexity of the reward function class. 
This suggests that our algorithm exploits the extrapolation power of the reward function class. 
To the best of our knowledge, this is the first action-dimension-free result for both linear and nonlinear bandit problems. More concretely, we instantiate our theorem to the following settings and get new results that leverage the model complexity (more in Section~\ref{sec:instance}). 
	
\begin{itemize}
	\item[1.] Linear bandit with finite parameter space $\Theta$. Because $\eta$ is concave in action $a$, our result leads to a standard regret bound $\bigO\pbra{T^{15/16}\pbra{\log|\Theta|}^{1/16}}$. In this case both zero-order optimization and the SquareCB algorithm in \citet{foster2020beyond} have regrets that depend on the dimension of action space $\da$.
	\item[2.] Linear bandit with $s$-sparse or structured instance parameters. Our algorithm {\AlgBan} achieves an  $\tildeO\pbra{T^{15/16}s^{1/16}}$  \sloppy standard regret bound when the instance/model parameter is $s$-sparse and the reward is deterministic. The regret bound of zero-order optimization depends polynomially on $\da$, so do Eluder dimension based bounds because the Eluder dimension for this class is $\Omega\pbra{\da}.$ The same bound also applies to linear bandit problems where the instance parameter has low-dimensional structure with $s$ degree of freedom. The prior work of \citet{carpentier2012bandit} achieves a stronger $\tildeO(s\sqrt{T})$ regret bound for $s$-sparse linear bandits with actions set $\mathcal{A}=S^{d-1}$. In contrast, our \AlgBan~algorithm applies more generally to any structured instance parameter set.  
	Other related results either leverage the rather strong anti-concentration assumption on the action set \citep{wang2020nearly}, or have implicit dimension dependency \citep[Remark 4.3]{hao2020high}.
	
	\item[3.] Two-layer neural nets bandit.  The local regret of our algorithm is bounded by $\tildeO\pbra{\epsilon^{-2}T^{3/4}}$. Zero-order optimization can also find a local maximum but with $\Omega(\da)$ samples. Optimistic algorithms in this case have an exponential sample complexity (see Theorem~\ref{thm:lowerbound-optimism}). Moreover, when the second layer of the ground-truth network contains all negative weights and the activation is convex and monotone, the local regret guarantee translates to a $\tildeO(T^{7/8})$ \textit{global regret} guarantee, because the reward is concave in the input (action) \citep{amos2017input}. 
\end{itemize}

The results for bandit can be extended to model-based RL with deterministic nonlinear dynamics and deterministic reward.
Our algorithm can find an approximate locally maximal stochastic policy (under additional Lipschitz assumptions):
\begin{theorem}[Informal version of Theorem~\ref{thm:rl-main}]
	For RL problems with  deterministic dynamics class and stochastic policy class that satisfy some Lipschitz properties, the  local regret of a model-based algorithm (\AlgBan~for RL, Algo~\ref{alg:rl}), compared to $\Omega(\epsilon)$-approximate local maxima, is bounded by $\bigO(\sqrt{T\src_T}/\epsilon^2),$ where $\src_T$ is the sequential Rademacher complexity of $\ell_2$ losses of the dynamics class.
\end{theorem}
To the best of our knowledge, this is the first model-based RL algorithms with provable finite sample complexity guarantees (for local convergence) for general nonlinear dynamics. The work of~\cite{luo2019algorithmic} is the closest prior work which also shows local convergence, but its conditions likely cannot be satisfied by any parameterized models (including linear models). We also present a concrete example of RL problems with nonlinear models satisfying our Lipschitz assumptions in Example~\ref{example:1} of Section~\ref{sec:main-results-rl}, which may also serve as a testbed for future model-based deep RL analysis. %
As discussed, other prior works on model-based RL do not apply to one-hidden-layer neural nets because they conclude global convergence which is not possible for one-hidden-layer neural nets in the worst case. 

\paragraph{Optimism vs. Exploring by Model-based Curvature Estimate.} The key algorithmic idea is to avoid exploration using the optimism-in-face-of-uncertainty principle because we show that optimism over a barely nonlinear model class is already statistically too aggressive, even if the ground-truth model is linear (see Theorem~\ref{thm:lowerbound-optimism} in Section~\ref{sec:lowerbounds}). Indeed, empirical model-based deep RL research has also not found optimism to be useful, partly because with neural nets dynamics, optimism will lead to huge virtual returns on the optimistic dynamics~\citep{luo2019algorithmic}. The work of \citet{foster2020beyond} also proposes algorithms that do not rely on UCB---their exploration strategy either relies on the discrete action space, or leverages the linear structure in the action space and has action-dimension dependency. In contrast, our algorithms' exploration relies more on the learning of the model (or the model's capability of predicting the curvature of the reward, as discussed more below).  Consequently, our regret bounds can be action-dimension-free. 

Our algorithm is conceptually very simple---it alternates between maximizing virtual return (over action or policy) and learning the model parameters by an online learner. The key insight is that, in order to ensure sufficient exploration for converging to local maxima,  it suffices for the model to predict the gradient and Hessian of the return reasonably accurately, and then follow the virtual return. We achieve reasonable curvature prediction by modifying the loss function of the online learner. 
We refer to the approach as  ``model-based curvature estimate''. 
Because we leverage model extrapolation, the sample complexity of model-based curvature prediction depends on the model complexity instead of action dimension in the zero-optimization approach for bandit.

We remark that many prior works also leverage the gradient or curvature information to explore without optimism, e.g., the work of \citet{dudik2011efficient,agarwal2014taming} as well as the EXP3/4 algorithms \citep{lattimore2020bandit}\footnote{In fact, these algorithms are instances of mirror descent. See \citet[Chapter 28]{lattimore2020bandit} for example.}. These algorithms exploit the structure of action space (which oftentimes leads to regret linear in the number of actions) and are closer to  zero-order optimization algorithms. In contrast,  our algorithm is model-based---it leverages the model extrapolation and results in regret bounds independent with the action space complexity. 

\paragraph{Organization. } 
This paper is organized as follows. Section~\ref{sec:preliminary} introduces the problem setup including the definition of local regret. In Section~\ref{sec:main-results-bandit} and Section~\ref{sec:main-results-rl} we present our main results for model-based nonlinear bandit and reinforcement learning respectively. Section~\ref{sec:lowerbounds} lists our negative results showing the inefficiency of optimism-in-face-of-uncertainty principle, as well as the hardness of finding global optimum in nonlinear bandit problems. Proofs of the lower bounds are deferred to Appendix~\ref{sec:app:pf5}. Appendix~\ref{sec:app:pf3} and \ref{app:rl} shows the proofs of main results for nonlinear bandit and reinforcement learning respectively. In Appendix~\ref{app:rl-instance} we present the analysis of our \AlgBan~for RL algorithm on a concrete example.
Finally, in Appendix~\ref{sec:app:helperlemmas} we list the proofs of helper lemmas.

\section{Problem Setup and Preliminaries}\label{sec:preliminary}
In this section, we first introduce our problem setup for nonlinear bandit and reinforcement learning, and then the preliminary for online learning and sequential Rademacher complexity.

\subsection{Nonlinear Bandit Problem with Deterministic Reward}\label{subsec:prelim-bandit}
We consider \textit{deterministic} nonlinear bandit problem with continuous actions. Let $\theta \in \Theta$ be the parameter that specifies the bandit instance,  $a\in \R^{\da}$ the action, and $\eta(\theta,a)\in [0,1]$ the reward function. Let $\thetas$ denote the unknown ground-truth parameter. Throughout the paper, we work under the \textit{realizability assumption} that $\thetas\in \Theta.$
A bandit algorithm aims to maximize the reward under $\thetas$, that is,  $\eta(\thetas,a).$ Let $a^\star=\argmax_a \eta(\thetas,a)$ be the optimal action (breaking tie arbitrarily). 
Let $\normsp{H}$ be the spectral norm of a matrix $H$. 
We also assume that the reward function, its gradient and Hessian matrix are Lipschitz, which are somewhat standard assumptions in the optimization literature (e.g., the work of \citet{johnson2013accelerating,ge2015escaping}).
\begin{assumption}\label{assumption:lipschitz} We assume that for all $\theta\in \Theta$, $\sup_a \normtwo{\ga \eta(\theta, a)}\le \gasup$ and $\sup_a \normsp{\ha \eta(\theta, a)}\le \hasup.$ And for every $\theta\in \Theta$ and $a_1,a_2\in \R^{\da}$, $\normsp{\ha \eta(\theta, a_1)-\ha \eta(\theta, a_2)}\le \tasup \normtwo{a_1-a_2}$.
\end{assumption}
As a motivation to consider deterministic rewards, we prove in Theorem~\ref{thm:lowerbound-optimism} for a special case that no algorithm can find a local maximum in less than $\sqrt{\da}$ steps. The result implies that an action-dimension-free regret bound is impossible under reasonably stochastic environments.

\paragraph{Approximate Local Maxima.}
In this paper, we aim to find a local maximum of the real reward function $\eta(\thetas,\cdot)$. 
A point $x$ is an $(\epsilon_g, \epsilon_h)$-approximate local maximum of a twice-differentiable function $f(x)$ if $\normtwo{\nabla f(x)}\le \epsilon_g$, and $\lmax(\nabla^2 f(x))\le \epsilon_h$. As argued in Sec.~\ref{sec:intro} and proved in Sec.~\ref{sec:lowerbounds}, because reaching a global maximum is computational and statistically intractable  for nonlinear problems, we only aim to reach a local maximum. 

\paragraph{Sample Complexity (for Converging to Local Maxima) and Local Regret.} Let $a_t$ be the action that the algorithm takes at time step $t$. The sample complexity for converging to approximate local maxima is defined to be the minimal number of steps $T$ such that there exists $t\in [T]$ where $a_t$ is an $(\epsilon_g,\epsilon_h)$ approximate local maximum with probability at least $1-\delta$.

On the other hand, we also define the ``local regret'' by  comparing with an approximate local maximum. Formally speaking, let $\sosp_{\epsilon_g,\epsilon_h}$ be the set of all $(\epsilon_g,\epsilon_h)$-approximate local maximum of $\eta(\thetas,\cdot)$. The $(\eps_g,\eps_h)$-local regret of a sequence of actions $a_1,\dots, a_T$ is defined as
\begin{equation}\label{equ:local-regret}
	\reg_{\epsilon_g,\epsilon_h}(T)=\sum_{t=1}^{T}\pbra{\inf_{a\in \sosp_{\epsilon_g,\epsilon_h}}\eta(\thetas,a)-\eta(\thetas,a_t)}.
\end{equation}
Our goal is to achieve a $(\eps_g,\eps_h)$-local regret that is sublinear in $T$ and inverse polynomial in $\eps_g$ and $\eps_h$. With a sublinear regret (i.e., $\reg_{\epsilon_g,\epsilon_h}(T)=o(T)$), the average performance $\frac{1}{T}\sum_{t=1}^{T}\eta(\thetas,a_t)$, converges to that of an approximate local maximum of $\eta(\thetas,\cdot)$.

\subsection{Reinforcement Learning}
We consider finite horizon Markov decision process (MDP) with deterministic dynamics, defined by a tuple $\left<\dy,r,H,\mu_1\right>$, where the dynamics $\dy$ maps from a state action pair $(s,a)$ to next state $s'$, $r:\calS\times\calA\to [0,1]$ is the reward function, and $H$ and $\mu_1$ denote the horizon and distribution of initial state respectively.  Let $\calS$ and $\calA$ be the state and action spaces. Without loss of generality, we make the standard assumption that the state space is disjoint for different time steps. That is, there exists disjoint sets $\calS_1,\cdots,\calS_H$ such that $\calS=\cup_{h=1}^{H}\calS_h$, and for any $s_h\in \calS,a_h\in a$, $\dy(s_h,a_h)\in \calS_{h+1}$.

In this paper consider parameterized policy and dynamics. Formally speaking, the policy class is given by $\Pi = \{\pi_\psi:\psi\in \Psi\}$, and the dynamics class is given by $\{\dy_\theta:\theta\in\Theta\}.$ The value function is defined as $V^\pi_\dy(s_h)\defeq \E\bbrasm{\sum_{h'=h}^{H}r(s_{h'},a_{h'})},$
where $a_h\sim \pi(\cdot\mid s_h), s_{h+1}=\dy(s_h,a_h).$
Sharing the notation with the bandit setting, let $\eta(\theta,\psi)=\E_{s_1\sim \mu_1}V^{\pi_\psi}_{\dy_\theta}(s_1)$ be the expected return of policy $\pi_\psi$ under dynamics $\dy_\theta.$ Also, we use $\rho_\dy^\pi$ to denote the distribution of state action pairs when running policy $\pi$ in dynamics $\dy$. For simplicity, we do not distinguish $\psi,\theta$ from $\pi_\psi$, $T_\theta$ when the context is clear. For example, we write $\V=V^{\pi_\psi}_{\dy_\theta}.$

The approximate local regret is defined in the same as in the bandit setting, except that the gradient and Hessian matrix are taken w.r.t to the policy parameter space $\psi$. We also assume realizability ($\thetas\in \Theta$) and the Lipschitz assumptions as in Assumption~\ref{assumption:lipschitz} (with action $a$ replaced by policy parameter $\psi$). 

\subsection{Preliminary on Online Learning with Stochastic Input Components}

Consider a prediction problem where we aim to learn a function that maps from $\calX$ to $\calY$ parameterized by parameters in $\Theta$. Let $\ell((x,y);\theta)$ be a loss function that maps $ (\calX\times \calY)\times \Theta \to \R_+$. An online learner $\ol$ aims to solve the prediction tasks under the presence of an adversarial nature iteratively. At time step $t$, the following happens. 

\begin{itemize}
	\item[1.] The learner computes a distribution $p_t = \ol(\{(x_i,y_i)\}_{i=1}^{t-1})$ over the parameter space $\Theta$. 
	\item[2.] The adversary selects a point $\bar{x}_t \in \bar{\calX}$ (which may depend on $p_t$) and generates a sample $\xi_t$ from some fixed distribution $q$. Let $x_t \defeq (\bar{x}_t, \xi_t)$, and the adversary picks a label $y_t\in \calY$.
	\item[3.] The data point $(x_t,y_t)$ is revealed to the online learner. 
\end{itemize} 
The online learner aims to minimize the expected regret in $T$ rounds of interactions, defined as
\begin{align}
	\olreg_T\triangleq\mathop{\E}\limits_{\substack{\xi_t\sim q,\theta_t\sim p_t\\ \forall 1\le t\le T}}\bbra{\sum_{t=1}^{T}\ell((x_t,y_t);\theta_t)-\inf_{\theta\in \Theta}\sum_{t=1}^{T}\ell((x_t,y_t);\theta)}.
	\label{equ:def-online-learning}
\end{align}
The difference of the formulation from the most standard online learning setup is that the $\xi_t$ part of the input is randomized instead of adversarially chosen (and the learner knows the distribution of $\xi_t$ before making the prediction $p_t$). It was introduced by \citet{rakhlin2011online}, who considered a more generalized setting where the distribution $q$ in round $t$ can depend on $\{x_{1}, \cdots, x_{t-1}\}$. 

We adopt the notation from \citet{rakhlin2011online,rakhlin2015online} to define the (distribution-dependent) sequential Rademacher complexity of the loss function class  $\calL = \{(x,y)\mapsto \ell((x,y);\theta):\theta\in \Theta\}$. For any set $\calZ$, a $\calZ$-valued tree with length $T$ is a set of functions $\{\vz_i : \{\pm 1\}^{i-1} \to \calZ\}_{i=1}^{T}$. For a sequence of Rademacher random variables $\epsilon=(\epsilon_1,\cdots,\epsilon_T)$ and for every $1 \le t \le T,$ we denote $\vz_t(\epsilon)\defeq \vz_t(\epsilon_1,\cdots,\epsilon_{t-1})$. For any $\bar{\calX}$-valued tree $\vx$ and any $\calY$-valued tree $\vy$, we define the sequential Rademacher complexity as

\begin{equation}
	\src_T(\calL;\vx,\vy)\triangleq \E_{\xi_1, \cdots, \xi_t}\E_{\epsilon}\bbra{\sup_{\ell\in \calL}\sum_{t=1}^{T}\epsilon_t \ell\big((\vx(\epsilon), \xi_t),\vy(\epsilon)\big)}.
\end{equation}
We also define $\src_T(\calL)=\sup_{\vx,\vy}\src_T(\calL;\vx,\vy)$, where the supremum is taken over all $\bar{\calX}$-valued and $\calY$-valued trees. \citet{rakhlin2011online} proved the existence of an algorithm whose online learning regret satisfies
$\olreg_T\le 2\src_T(\calL).$

\section{Model-based  Algorithms for nonlinear Bandit }
\label{sec:main-results-bandit}

We first study model-based algorithms for nonlinear continuous bandits problem, which is a simplification of model-based reinforcement learning. We use the notations and setup in Section~\ref{subsec:prelim-bandit}.

\newcommand{\hattheta}{\hat{\theta}}

\paragraph{Abstraction of analysis for model-based algorithms. } Typically, a model-based algorithm explicitly maintains an estimated model $\hattheta_t$, and sometimes maintains a distribution, posterior, or confidence region of $\hattheta_t$. We will call $\eta(\thetas, a)$ the \textit{real reward} of action $a$, and $\eta(\hattheta_t, a)$ the \textit{virtual reward}.  
Most analysis for model-based algorithms (including UCB and ours) can be abstracted as showing the following two properties: 

\noindent(i)~ the virtual reward $\eta(\hattheta_t,a_t)$ is sufficiently high. \\
(ii) the virtual reward $\eta(\hattheta_t,a_t)$ is close to the real reward $\eta(\thetas,a_t)$ in the long run. 

One can expect that a proper combination of property (i) and (ii) leads to showing the real reward $\eta(\thetas,a_t)$ is high in the long run. Before describing our algorithms, we start by inspecting and summarizing the pros and cons of UCB from this viewpoint. 

\paragraph{Pros and cons of UCB. } The UCB algorithm chooses an action $a_t$ and an estimated model $\hattheta_t$ that maximize the virtual reward $\eta(\hattheta_t,a_t)$ among those models agreeing with the observed data. The pro is that it satisfies property (i) by definition---$\eta(\hattheta_t, a_t)$ is higher than the optimal real reward $\eta(\thetas, a^\star)$. The downside is that ensuring (ii) is challenging and often requires strong complexity measure bound such as Eluder dimension (which is not polynomial for even barely nonlinear models, as shown in Theorem~\ref{thm:lowerbound-eluder}). 
The difficulty largely stems from our very limited control of $\hattheta_t$ except its consistency with the observed data. In order to bound the difference between the real and virtual rewards, we essentially require that \textit{any} model that agrees with the past history should extrapolate to any future data accurately (as quantitatively formulated in Eluder dimension). Moreover, the difficulty of satisfying property (ii) is fundamentally caused by the over-exploration of UCB---As shown in the Theorem~\ref{thm:lowerbound-optimism}, UCB suffers from bad regrets with barely nonlinear  family of models. %

\paragraph{Our key idea: natural exploration via model-based curvature estimate. } We deviate from UCB by readjusting the priority of the two desiderata. First, we focus more on ensuring property (ii)  on large model class by leveraging strong online learners. We use an online learning algorithm to predict $\hattheta_t$ with the objective that $\eta(\hattheta_t,a_t)$ matches $\eta(\thetas,a_t)$ . As a result,  the difference between the virtual and real reward depends on the online learnability or the sequential Rademacher complexity of the model class. Sequential Rademacher complexity turns out to be a fundamentally more relaxed complexity measure than Eluder dimension---e.g., two-layer neural networks'
sequential Rademacher complexity is polynomial in parameter norm and dimension, but their Eluder dimension is at least exponential in  dimension (even with a constant parameter norm). 
However, an immediate consequence of using online-learned $\hattheta_t$ is that we lose  optimism/exploration that ensured property (i).\footnote{More concretely, the algorithm can get stuck when (1) $a_t$ is optimal for $\hattheta_t$, (2) $\hattheta_t$ fits actions $a_t$ (and history) accurately, but (3) $\hattheta_t$ does not fit $a^\star$ (because online learner never sees $a^\star$). The passivity of online learning formulation causes this issue---the online learner is only required to predict well for the point that it saw and will see, but not for those points that it never observes. This limitation, on the other hand, allows more relaxed complexity measure of the model class (that is, sequential Rademacher complexity instead of Eluder dimension).}

\begin{algorithm}[ht]
	\setstretch{0.9}
	\caption{\AlgBan: \textbf{Vi}rtual Ascent with \textbf{O}n\textbf{lin}e Model Learner (for Bandit)}
	\label{alg:ban}
	\begin{algorithmic}[1]
		\State Set parameter $\ubg=2\gasup$ and $\ubh=640\sqrt{2}\hasup$. Let $\calH_0 = \emptyset$; choose $a_0\in \calA$ arbitrarily. 
		\For{$t=1,2,\cdots$}
		\State Run $\ol$ on $\calH_{t-1}$ with loss function $\ell$ (defined in equation~\eqref{equ:bandit-loss}) and obtain $p_t = \ol(\calH_{t-1})$. 
		\State Let $a_t\gets \argmax_{a}\E_{\theta_t\sim p_t}\bbra{\eta(\theta_t,a)}$.
		\State Sample $u_t,v_t\sim \calN(0,I_{\da\times \da})$ independently. 
		\State Let $\xi_t = (u_t,v_t)$, $\bar{x}_t = (a_{t}, a_{t-1})$, and $x_t = (\bar{x}_t, \xi_t)$
		\State Compute $y_t = [\eta(\thetas, a_t), \eta(\thetas,a_{t-1}), \inner{\ga \eta(\thetas,a_{t-1})}{u_t}, \inner{\ha \eta(\thetas,a_{t-1})u_t}{v_t}]\in \R^4$ by applying a finite number of actions in the real environments using equation~\eqref{equ:finite-difference-grad} and~\eqref{equ:finite-differnece-hessian} with infinitesimal $\alpha_1$ and $\alpha_2$. 
		\State Update $\calH_t=\calH_{t-1}\cup \left\{(x_t,y_t)\right\}$
		\EndFor
	\end{algorithmic}
\end{algorithm}

Our approach realizes property (i) in a sense that the virtual reward will improve iteratively if the real reward is not yet near a local maximum. This is much weaker than what UCB offers (i.e., that the virtual reward is higher than the optimal real reward), but suffices to show the convergence to a local maximum of the real reward function. We achieve this by demanding the estimated model $\hattheta_t$ not only to predict the real reward accurately, but also to predict the gradient $\nabla_a \eta(\thetas, a)$ and Hessian $\nabla_a^2 \eta(\thetas, a)$ accurately. In other words, we augment the loss function for the online learner so that the estimated model satisfies $\eta(\hattheta_t, a_t)\approx \eta(\thetas, a_t)$, $\nabla_a \eta(\hattheta_t, a_t)\approx \nabla_a \eta(\thetas, a_t)$, and $\nabla_a^2 \eta(\hattheta_t, a_t)\approx \nabla_a^2 \eta(\thetas, a_t)$ in the long run. This implies that when $a_t$ is not at a local maximum of the real reward function $\eta(\thetas, \cdot)$, then it's not at a maximum of the virtual reward $\eta(\hattheta_t, \cdot)$, and hence the virtual reward will improve in the next round if we take the greedy action that maximizes it. 

\paragraph{Estimating projections of gradients and Hessians. }To guide the online learner to predict $\nabla_a \eta(\thetas, a_t)$ correctly, we need a supervision for it. However, we only observe the reward $\eta(\thetas, a_t)$.  Leveraging the deterministic reward property, we use rewards at $a$ and $a+\alpha_1 u$ to estimate the projection of the gradient at a random direction $u$: 
\begin{align}\label{equ:finite-difference-grad}
\dotp{\ga \eta(\thetas,a)}{u} = \lim_{\alpha_1\to 0}\left(\eta(\thetas,a+\alpha_1 u)-\eta(\thetas,a)\right)/\alpha_1
\end{align}
It turns out that the number of random projections $\dotp{\ga \eta(\thetas,a)}{u}$ needed for ensuring a large virtual gradient \textit{does not} depend on the dimension, because we only use these projections to estimate the norm of the gradient but not necessarily the exact direction of the gradient (which may require $d$ samples.) Similarly, we can also estimate the projection of Hessian to two random directions $u, v\in \da$ by:
\begin{align}\label{equ:finite-differnece-hessian}
& \dotp{\ha \eta(\thetas,a)v}{u} = \lim_{\alpha_2\to 0}\left(\dotp{\ga \eta(\thetas,a+\alpha_2 v)}{u}-\dotp{\ga \eta(\thetas,a)}{u}\right)/\alpha_2\\ 
&\quad =\lim_{\alpha_2\to 0}\lim_{\alpha_1\to 0}\left(\left(\eta(\thetas,a+\alpha_1 u+\alpha_2 v)-\eta(\thetas,a+\alpha_2 v)\right)-\left(\eta(\thetas,a+\alpha_1 u)-\eta(\thetas,a)\right)\right)/ (\alpha_1\alpha_2)\nonumber
\end{align}
Algorithmically, we can choose infinitesimal $\alpha_1$ and $\alpha_2$. Note that $\alpha_1$ should be at least an order smaller than $\alpha_2$ because the limitations are taken sequentially.

We create the following prediction task for an online learner: let $\theta$ be the parameter, $x = (a, a', u, v)$ be the input,  $$\hat{y}= [\eta(\theta, a), \eta(\theta,a'), \inner{\ga \eta(\theta,a')}{u}, \inner{\ha \eta(\theta,a')u}{v}]\in \R^4 $$ be the output, and $$y = [\eta(\thetas, a), \eta(\thetas,a'), \inner{\ga \eta(\thetas,a')}{u}, \inner{\ha \eta(\thetas,a')u}{v}]\in \R^4$$ be the supervision, and  the loss function be
\begin{align}\label{equ:bandit-loss}
\ell(((a,a',u,v), y);\theta)\triangleq&\pbra{[\hat{y}]_1-[y]_1}^2+\pbra{[\hat{y}]_2-[y]_2}^2+\mintwo{\ubg^2}{\pbra{[\hat{y}]_3-[y]_3}^2}\nonumber \\
&+ \mintwo{\ubh^2}{\pbra{[\hat{y}]_4-[y]_4}^2}
\end{align}
Here we used $[y]_i$ to denote the $i$-th coordinate of $y\in \R^4$ to avoid confusing with $y_t$ (the supervision at time $t$.) Our model-based bandit algorithm is formally stated in Alg.~\ref{alg:ban} with its regret bound below. 
\begin{theorem}\label{thm:bandit-main}
	Let $\src_T$ be the sequential Rademacher complexity of the family of the losses defined in Eq.~\eqref{equ:bandit-loss}. Let $C_1=2+\gasup/\hasup$. Under Assumption~\ref{assumption:lipschitz}, for any $\epsilon\le \mintwo{1}{\tasup/16},$ we can bound the $(\epsilon,6\sqrt{\tasup\epsilon})$-local regret of Alg.~\ref{alg:ban} from above by
	\begin{equation}
	\E\bbra{\reg_{\epsilon,6\sqrt{\tasup\epsilon}}(T)}\le
	\pbra{1+C_1\sqrt{4T\src_T}}\maxtwo{4\hasup\epsilon^{-2}}{\sqrt{\tasup}\epsilon^{-3/2}}.
	\end{equation}
\end{theorem}

Note that when the sequential Rademacher complexity $\src_T$ is bounded by $\tildeO(R\sqrt{T})$ (which is typical), we have $\bigO(\sqrt{T\src_T}) = \tildeO(T^{3/4}) = o(T)$ regret. As a result, Alg.~\ref{alg:ban} achieves a $\bigO(\poly(1/\epsilon))$ sample complexity by the sample complexity-regret reduction \citep[Section 3.1]{jin2018q}.

Theorem~\ref{thm:bandit-main} states that the reward of Alg.~\ref{alg:ban} converges to the reward of a local maximum. In addition, with a little modification of the proof, we can also show that  Alg.~\ref{alg:ban} can find a local maximum \emph{action} in polynomial steps.

\begin{theorem}\label{thm:ban-action-space}
	Suppose the sequential Rademacher complexity of the loss function (defined in Eq.~\eqref{equ:bandit-loss}) is bounded by $\sqrt{R(\Theta)T\polylog(T)}.$ When $T\gtrsim R(\Theta)\epsilon^{-8}\polylog(R(\Theta),1/\epsilon)$, if we run Alg.~\ref{alg:ban} for $T$ steps, there exists $t\in [T]$ such that $a_t\in \sosp_{\epsilon,6\sqrt{\tasup\epsilon}}$ with probability at least $0.5$.
\end{theorem}
Proof of Theorem~\ref{thm:ban-action-space} is deferred to Appendix~\ref{app:action-convergence}. We can also boost the success probability by running Alg.~\ref{alg:ban} multiple times.

\subsection{Instantiations of Theorem~\ref{thm:bandit-main}}\label{sec:instance}
In the sequel we sketch some instantiations of our main theorem, whose proofs are deferred to Appendix~\ref{app:instance}.

\paragraph{Linear bandit with finite model class.} Consider the problem with action set $\calA = \{a\in \R^{d}: \|a\|_2\le 1 \}$ and finite model class $\Theta \subset \{\theta\in \R^d : \|\theta\|_2=1\}$. Suppose the reward is linear, that is, $\eta(\theta,a) = \langle \theta, a\rangle$. We deal with the constrained action set by using a surrogate loss $\tilde{\eta}(\theta, a) \defeq \langle \theta, a\rangle - \frac{1}{2}\|a\|_2^2$ and apply Theorem 3 with reward $\tilde{\eta}$. 
 We claim that the expected (global) regret is bounded by $\bigO\pbra{T^{15/16}\pbra{\log|\Theta|}^{1/16}}$.
Note that here the regret bound is \textit{independent} of the dimension $d$, whereas, by contrast, the SquareCB algorithm in \citet{foster2020beyond} depends polynomially on $d$ (see Theorem 7 of \citet{foster2020beyond}). 
Zero-order optimization approach \citep{duchi2015optimal} in this case also gives a $\poly(d)$ regret bound. This and examples below demonstrate that our results fully leverage the low-complexity model class to eliminate the dependency on the action dimension. 

A full proof of this claim needs a few steps: (i) realizing that $\eta(\theta^\star, a)$ is concave in $a$ with no bad local maxima,  and therefore our local regret and the standard regret coincide (up to some conversion of the errors); (ii) invoking \citet[Lemma 3]{rakhlin2015sequential} to show that the sequential Rademacher complexity $\src_T$ is bounded by $\bigO\pbra{\sqrt{(2\log |\Theta|)/T}}$, and (iii) verifying $\tilde{\eta}$ satisfies the conditions (Assumption~\ref{assumption:lipschitz}) on the actions that the algorithm will visit.

\paragraph{Linear bandit with sparse or structured model vectors.} 
We consider the deterministic linear bandit setting where the model class $\Theta=\{\theta\in \R^{d}:\|\theta\|_0 \le s, \|\theta\|_2 = 1\}$ consists of all $s$-sparse vectors on the unit sphere.
Similarly to finite hypothesis case, we claim that the global regret of Alg.~\ref{alg:ban} is $\E[\reg(T)]=\tildeO\pbra{T^{15/16}s^{1/16}}.$ The regret of our algorithm only depends on the sparsity level $s$ (up to logarithmic factors), whereas the Eluder dimension of sparse linear hypothesis is still $\Omega(d)$ (see Lemma~\ref{lem:sparse-eluder}), and the regret in~\citet{lattimore2020bandit} also depends on $d$. %
The proof follows from discretizing the space $\Theta$ into roughly $(dT)^{O(s)}$ points and applying the finite model class result above.
\sloppy We remark that ~\citet{lattimore2020bandit} showed a $\Omega(d)$ sample complexity lower bound for the sparse linear bandit problem with stochastic reward. But here we only consider a deterministic reward and continuous action.\footnote{In fact, if the reward is deterministic, there exists a simple ad-hoc algorithm that solve $s$-sparse linear bandit with $O(s\log d)$ sample complexity. First, the algorithm plays the action $a_1=(1,\cdots,1,0,\cdots,0)/Z$ with first half of coordinates being non-zero, normalized by $Z$. The return $\dotp{\thetas}{a_1}$ reveals whether there exists a non-zero entry in the first half coordinates of $\thetas$. Then, we can proceed with the binary search on the non-zero half of the coordinates. Iteratively, we can identify a non-zero entry with $\log d$ samples. Running $s$ rounds of binary search reveals all the non-zero entries of $\thetas$, which solves the $s$-sparse linear bandit problem. However, we note that this algorithm seems to be adhoc and does not extend to other cases, e.g., finite or structured model class.}

Moreover, we can further extend the result to other linear bandit settings where $\theta$ has an additional structure. Suppose $\Theta = \{\theta = \phi(z): z\in \R^s \}$ for some Lipschitz function $\phi$. Then, a simlar approach gives regret bound that only depends on $s$ but not $d$ (up to logarithmic factors). 

\paragraph{Deterministic logistic bandits.} For deterministic logistic bandits, the reward function is given by $\eta(\theta,a)=(1+e^{-\dotp{\theta}{a}})^{-1}.$ The model class is $\Theta\subseteq S^{d-1}$ and the action space is $\calA=S^{d-1}$. Similarly, we run Alg.~\ref{alg:ban} on an unbounded action space with regularized loss $\tilde{\eta}(\theta,a)=\eta(\theta,a)-\frac{c}{2}\normtwo{a}^2$ where $c=e(e+1)^{-2}$ is a constant. The optimal action in this case is $a^\star=\thetas$. Note that the loss function is not concave, but it satisfies that all local maxima are global. As a result, translating our local regret bound to global regret, the expected standard regret of Alg.~\ref{alg:ban} is bounded by $\bigO\pbra{T^{15/16}(\log|\Theta|)^{1/16}}.$ Compared with algorithms that specially designed for logistic bandits \citep{faury2020improved,dong2018information,filippi2010parametric,li2017provably}, our regret bound obtained by reduction is not optimal.

\paragraph{Two-layer neural nets.} 
We consider the reward function given by two-layer neural networks with width $m$. For matrices $W_1\in \R^{m\times d}$ and $W_2\in \R^{1\times m},$
let $\eta((\mW_1,\mW_2),a)=\mW_2\sigma(\mW_1 a)-\frac{1}{2}\normtwo{a}^2$ for some nonlinear link function $\sigma:\R\to [0,1]$ with bounded derivatives up to the third order. %
Recall that the $(1,\infty)$-norm of $W_1$ is defined by $\max_{i\in [m]}\sum_{j=1}^{d}\abs{[\mW_1]_{i,j}}.$ %
Let the model hypothesis space be $\Theta=\{(\mW_1,\mW_2):\norm{\mW_1}_{1,\infty}\le 1, \normone{\mW_2}\le 1\}$ %
and $\theta\defeq (\mW_1,\mW_2)$. We claim that the expected local regret of Alg~.\ref{alg:ban} is bounded by $\tildeO\pbra{\epsilon^{-2}T^{3/4}\polylog(d)}$. To the best of our knowledge, this is the first result analyzing nonlinear bandit with neural network parameterization. The result follows from  analyzing the sequential Rademacher complexity for $\eta$, $\langle \nabla_a \eta, u\rangle $, and $\langle u, \nabla_a^2 \eta \cdot v\rangle$, and finally the resulting loss function $\ell$. See Theorem~\ref{thm:two-layer} in Section~\ref{app:instance} for details.
We remark here that zero-order optimization in this case gives a $\poly(d)$ local regret bound.

We note that if the second layer of the neural network $W_2$ contains all negative entries, and the activation function $\sigma$ is monotone and convex, then $\eta((W_1,W_2), a)$ is \textit{concave} in the action. (This is a special case of input convex neural networks \citep{amos2017input}.) Therefore, in this case, the  local regret is the same as the global regret, and we can obtain global regret guarantee (see Theorem~\ref{thm:two-layer}.) 
We note that loss function for learning  input convex neural networks is still nonconvex, but the statistical global regret result does not  rely on the convexity of the loss for learning.

\subsection{Proof Sketch for Theorem~\ref{thm:bandit-main}}

Proof of Theorem~\ref{thm:bandit-main} consists of the following parts:

\begin{itemize}
	\item[i.] Because of the design of the loss function (Eq.~\ref{equ:bandit-loss}), the online learner guarantees that $\theta_t$ can estimate the reward, its gradient and hessian accurately, that is, for $\theta_t\sim p_t$,  $\eta(\thetas,a_t)\approx \eta(\theta_t,a_t)$, $\nabla_a \eta(\thetas,a_{t-1})\approx\nabla_a \eta(\theta_t,a_{t-1})$, and $\nabla_a^2 \eta(\thetas,a_{t-1})\approx\nabla_a^2 \eta(\theta_t,a_{t-1})$.
	\item[ii.] Because of (i), maximizing the virtual reward $\E_{\theta_t} \eta(\theta_t, a)$ w.r.t $a$ leads to improving the real reward function $\eta(\thetas, a)$ iteratively (in terms of finding second-order local improvement direction.)
\end{itemize}

\sloppy Concretely, define the errors in rewards and its derivatives:  $\err_{t,1}=\abs{\eta(\theta_t,a_t)-\eta(\thetas,a_t)}$, 
$\err_{t,2}=\abs{\eta(\theta_t,a_{t-1})-\eta(\thetas,a_{t-1})},$ 
$\err_{t,3}=\normtwo{\ga \eta(\theta_t,a_{t-1})-\ga \eta(\thetas,a_{t-1})}$, and
$\err_{t,4}=\normsp{\ha \eta(\theta_t,a_{t-1})-\ha \eta(\thetas,a_{t-1})}.$ Let
$\err_{t}^2=\sum_{i=1}^{4}\err_{t,i}^2$  be the total error which measures how closeness between $\theta_t$ and $\thetas$.

Assuming that $\Delta_{t,j}$'s are small, to show (ii), we essentially view $a_t=\argmax_{a\in\calA}\E_{\theta_t}\eta(\theta_t,a)$ as an approximate update on the real reward $\eta(\thetas, \cdot)$ and show it has local improvements if $a_{t-1}$ is not a critical point of the real reward:
\begin{align}
&\eta(\thetas,a_t)\gtrapprox_{\Delta_t} \E_{\theta_t}\bbra{\eta(\theta_t,a_t)}  \label{equ:poc-1}\\[-3pt]
&\quad\ge ~~~\sup_a\E_{\theta_t}\bbra{\eta(\theta_t,a_{t-1})+\dotp{a-a_{t-1}}{\nabla_a \eta(\theta_t,a_{t-1})}-\frac{\hasup}{2}\normtwo{a-a_{t-1}}^2}\label{equ:poc-2}\\[-3pt]
&\quad\gtrapprox_{\Delta_t} \sup_a\E_{\theta_t}\bbra{\eta(\thetas,a_{t-1})+\dotp{a-a_{t-1}}{\nabla_a \eta(\thetas,a_{t-1})}-\frac{\hasup}{2}\normtwo{a-a_{t-1}}^2}\label{equ:poc-3}\\[-5pt]
&\quad\ge \eta(\thetas,a_{t-1})+\frac{1}{2\hasup}\normtwo{\nabla_a \eta(\thetas,a_{t-1})}^2\label{equ:poc-4}.
\end{align}
Here in equations~\eqref{equ:poc-1} and~\eqref{equ:poc-3}, we use the symbol $\gtrapprox_{\Delta_t}$ to present \textit{informal} inequalities that are true up to some additive errors that depend on $\Delta_t$.   This is because equation~\eqref{equ:poc-1}  holds up to errors related to $\err_{t,1}=\abs{\eta(\theta_t,a_t)-\eta(\thetas,a_t)}$, and equation~\eqref{equ:poc-3} holds up to errors related to $\err_{t,2}=\abs{\eta(\theta_t,a_{t-1})-\eta(\thetas,a_{t-1})}$ and $\err_{t,3}=\normtwo{\ga \eta(\theta_t,a_{t-1})-\ga \eta(\thetas,a_{t-1})}$. Eq.~\eqref{equ:poc-2} is a second-order Taylor expansion around the previous iteration $a_{t-1}$ and utilizes the definition $a_t=\argmax_{a\in\calA}\E_{\theta_t}\eta(\theta_t,a)$. Eq.~\eqref{equ:poc-4} is a standard step to show the first-order improvement of gradient descent (the so-called ``descent lemma''). We also remark that $a_t$ is the maximizer of the expected reward $\E_{\theta_t}\eta(\theta_t,a)$ instead of $\eta(\theta_t,a)$ because the adversary in online learning cannot see $\theta_t$ when choosing adversarial point $a_t$.  

The following lemma formalizes the proof sketch above, and also extends it to considering second-order improvement. 
The proof can be found in Appendix~\ref{app:proof-bandit-improvement}.
\begin{lemma}\label{lem:bandit-improvement}
	In the setting of Theorem~\ref{thm:bandit-main}, when $a_{t-1}$ is \emph{not} an  $(\epsilon,6\sqrt{\tasup\epsilon})$-approximate second order stationary point, we have
	\begin{align}
	\eta(\thetas, a_t)\ge \eta(\thetas, a_{t-1})+\min\pbra{\hasup^{-1}\eps^2/4,\tasup^{-1/2}\epsilon^{3/2}}-C_1\E_{\theta_t\sim p_t}\bbra{\err_t}.
	\end{align}
\end{lemma}

Next, we show part (i) by linking the error $\Delta_t$ to the loss function $\ell$ (Eq.~\eqref{equ:bandit-loss}) used by the online learner. The errors $\Delta_{t,1}, \Delta_{t,2}$ are already part of the loss function. 
Let $\emperr_{t,3}=\dotp{\ga \eta(\theta_t,a_{t-1})-\ga \eta(\thetas,a_{t-1})}{u_t}$ and 
$\emperr_{t,4}=\dotp{{\ha \eta(\theta_t,a_{t-1})-\ha \eta(\thetas,a_{t-1})}u_t}{v_t}$ be the remaining two terms (without the clipping) in the loss (Eq.~\eqref{equ:bandit-loss}). Note that $\emperr_{t,3}$ is supposed to bound $\Delta_{t,3}$ because $\E_{u_t}[\emperr_{t,3}^2] = \Delta_{t,3}^2$. Similarly, $\E_{u_t,v_t}[\emperr_{t,4}^2] = \normFsm{\ha \eta(\theta_t,a_{t-1})-\ha \eta(\thetas,a_{t-1})}^2 \ge \Delta_{t,4}^2$. We clip $\emperr_{t,3}$ and $\emperr_{t,4}$ to make them uniformly bounded and improve the concentration with respect to the randomness of $u$ and $v$ (the clipping is conservative and is often not active). Let 
$\emperr_{t}^2=\err_{t,1}^2+\err_{t,2}^2+\mintwo{\ubg^2}{\emperr_{t,3}^2}+\mintwo{\ubh^2}{\emperr_{t,4}^2}$ be the error received by the online learner at time $t$. The argument above can be rigorously formalized into a lemma that upper bound $\Delta_t$ by $\emperr_t$, which will be bounded by the sequential Rademacher complexity.

\begin{lemma}\label{lem:bandit-concentration} 
	By choosing $\ubg=2\gasup$ and $\ubh=640\sqrt{2}\hasup$, we have
	\begin{align}
		\E_{u_{1:T},v_{1:T},\theta_{1:T}}\bbra{\sum_{t=1}^{T}\emperr_{t}^2}\ge \frac{1}{2}\E_{\theta_{1:T}}\bbra{\sum_{t=1}^{T}\err_t^2}.
	\end{align}
\end{lemma}
We defer the proof to Appendix~\ref{app:proof-bandit-concentration}.
With Lemma~\ref{lem:bandit-improvement} and Lemma~\ref{lem:bandit-concentration}, we can prove Theorem~\ref{thm:bandit-main} by keeping track of the performance $\eta(\thetas,a_t)$. %
The full proof can be found in Appendix~\ref{app:proof-bandit-main}.

\section{Model-based Reinforcement Learning}\label{sec:main-results-rl}

In this section, we extend the results in Section~\ref{sec:main-results-bandit} to model-based reinforcement learning with deterministic dynamics and reward function. 

We can always view a model-based reinforcement learning problem with parameterized dynamics and policy as a nonlinear bandit problem in the following way. The policy parameter $\psi$ corresponds to the action $a$ in bandit, and the dynamics parameter $\theta$ corresponds to the model parameter $\theta$ in bandit. The expected total return $\eta(\theta,\psi) = \E_{s_1\sim \mu_1}V^{\pi_\psi}_{\dy_\theta}(s_1)$ is the analogue of reward function in bandit. We intend to make the same regularity assumptions on $\eta$ as in the bandit case (that is, Assumption~\ref{assumption:lipschitz}) with $a$ being replaced by $\psi$. However, when the policy is deterministic, the reward function $\eta$ has Lipschitz constant with respect to $\psi$ that is exponential in $H$ (even if dynamics and policy are both deterministic with good Lipschitzness). This prohibits efficient optimization over policy parameters. Therefore we focus on \textit{stochastic policies} in this section, for which we expect $\eta$ and its derivatives to be Lipschitz with respect to $\psi$. 

Blindly treating RL as a bandit only utilizes the reward but not the state observations. %
In fact, one major reason why model-based methods are more sample efficient is that it supervises the learning of dynamics by state observations. 
To reason about the learning about local steps and the dynamics, 
we make the following additional Lipschitzness of value functions w.r.t to the states and Lipschitzness of policies w.r.t to its parameters, beyond those assumptions for the total reward $\eta(\theta,\psi)$ in Assumption~\ref{assumption:lipschitz}. 

\begin{assumption}\label{assumption:RL-lipschitz}
	We assume the following (analogous to Assumption~\ref{assumption:lipschitz}) on the value function: $\forall \psi\in \Psi,\theta\in \Theta,s,s'\in \calS$ we have
	\begin{itemize}
		\item $\abssm{\V(s)-\V(s')}\le L_0\normtwosm{s-s'}$;
		\item $\normtwosm{\gpsi \V(s)-\gpsi \V(s')}\le L_1\normtwosm{s-s'}$;
		\item$\normspsm{\hpsi \V(s)-\hpsi \V(s')}\le L_2\normtwosm{s-s'}.$
	\end{itemize}
\end{assumption}
\begin{assumption}\label{assumption:RL-smoothness}
	We assume the following Lipschitzness assumptions on the stochastic policies parameterization $\pi_\psi$.\footnote{Recall that the injective norm of a $k$-th order tensor $A\in {\R^{d}}^{\otimes k}$ is defined as $\normsp{A^{\otimes k}}=\sup_{u\in S^{d-1}}\dotp{A}{u^{\otimes k}}.$ } 
	\begin{itemize}
		\item $\normspsm{\E_{a\sim \pi_\psi(\cdot\mid s)}[(\gpsi \log \pi_\psi(a\mid s))(\gpsi \log \pi_\psi(a\mid s))^\top]}\le \Pg;$
		\item $\normspsm{\E_{a\sim \pi_\psi(\cdot\mid s)}[(\gpsi \log \pi_\psi(a\mid s))^{\otimes 4}]}\le \Pt;$
		\item $\normspsm{\E_{a\sim \pi_\psi(\cdot\mid s)}[(\hpsi \log \pi_\psi(a\mid s))(\hpsi \log \pi_\psi(a\mid s))^\top]}\le \Ph.$
	\end{itemize}
\end{assumption}

Our results will depend polynomially on the parameters $L_0, L_1, L_2, \Pg,\Pt$ and $\Ph$.  To demonstrate that the Assumption~\ref{assumption:RL-lipschitz} and~\ref{assumption:RL-smoothness} can contain interesting RL problems with nonlinear models and stochastic policies, we give the following example where these parameters are all on the order of $O(1)$.

\begin{example}\label{example:1}
Let state space $\calS$ be the unit ball in $\R^d$ and action space $\calA$ be $\R^d$.  The (deterministic) dynamics $T$ is given by $\dy(s,a)=\NN_\theta(s+a),$ where $\NN$ is a nonlinear model parameterized by $\theta$, e.g., a neural network. 
We assume that $\theta$ belongs to a finite hypothesis class $\Theta$ that satisfies $\normtwo{\NN_\theta(s+a)}\le 1$ for all $\theta\in \Theta,s\in\calS,a\in\calA.$
Assume that the reward function $r(s,a)$ is $L_r$-Lipschitz w.r.t $\ell_2$-norm, that is, satisfying $\abs{r(s_1,a_1)-r(s_2,a_2)}\le L_r(\normtwo{s_1-s_2}+\normtwo{a_1-a_2}).$ 
We consider a family of stochastic Gaussian policies with the mean being linear in the state: $\pi_\psi(s)=\calN(\psi s,\sigma^2 I)$, parameterized by $\psi\in \R^{d\times d}$ with $\normop{\psi}\le 1$. We consider $\sigma\in (0,1)$ as a small constant on the order of 1.\footnote{At the first sight, the noise level appears to be quite large because the norm of the noise in the action dominates the norm of the mean. However, this can make sense because the model $\NN_\theta$ can average out the noise by, e.g., taking a weighted sum of its input $s+a$ first before doing other computations. In other words, the scaling of the noise here implicitly assumes that the model $\NN_\theta$ should typically average out the noise by looking at the all the coordinates.}

In this setting, Assumption~\ref{assumption:lipschitz}, ~\ref{assumption:RL-lipschitz}, and~\ref{assumption:RL-smoothness} hold with  all parameters $\gasup, \hasup,\tasup, L_0, L_1, L_2, \Pg,\Pt$ and $\Ph$ bounded by $\poly(\sigma,1/\sigma,H,L_r)$. 
\end{example}
The proofs for the bounds on the Lipschitz parameters are highly nontrivial and deferred to Section~\ref{app:rl-instance}. 

We will show that the difference of gradient and Hessian of the total reward can be upper-bounded by the difference of dynamics. Let $\tau_t=(s_1,a_1,\cdots,s_H,a_H)$ be a trajectory sampled from policy $\pi_{\psi_{t}}$ under the ground-truth dynamics $\dy_{\thetas}.$ 
Similarly to~\cite{yu2020mopo}, using the simulation lemma and Lipschitzness of the value function, we can easily upper bound $\Delta_{t,1} = 	|\eta(\theta_t, \psi_t)-\eta(\thetas, \psi_t)|$ by the one-step model prediction errors.  
Thanks to policies' stochasticity, using the REINFORCE formula, we can also bound  the gradient  errors  by the model errors:
\begin{align*}
&\err_{t,3}^2=\normtwo{\gpsi \eta(\theta_t, \psi_{t-1})-\gpsi \eta(\thetas, \psi_{t-1})}^2\lesssim\E_{\tau\sim \rho^{\psi_{t-1}}_{\thetas}}\bbra{\sum_{h=1}^{H}\normtwo{\dy_{\theta_t}(s_h,a_h)-\dy_{\thetas}(s_h,a_h)}^2}.
\end{align*}
Similarly, we can upper bound the Hessian errors by the errors of dynamics. As a result, the loss function simply can be set to
\begin{align}\label{equ:rl-loss}
	\ell((\tau_t,\tau_t');\theta)=\hspace{-10pt}\sum_{(s_h,a_h)\in \tau_t}\normtwo{\dy_{\theta}(s_h,a_h)-\dy_{\thetas}(s_h,a_h)}^2+\hspace{-10pt}\sum_{(s_h',a_h')\in \tau_t'}\normtwo{\dy_{\theta}(s_h',a_h')-\dy_{\thetas}(s_h',a_h')}^2
\end{align}
for two trajectories $\tau,\tau'$ sampled from policy $\pi_{\psi_t}$ and $\pi_{\psi_{t-1}}$ respectively.
Compared to Alg.~\ref{alg:ban}, the loss function is here simpler without  relying on finite difference techniques to query gradients projections. 
Our algorithm for RL is analogous to Alg.~\ref{alg:ban} by using the loss function in Eq.~\eqref{equ:rl-loss}. Our algorithm is presented in Alg.~\ref{alg:rl} in Appendix~\ref{app:rl}. 
Main theorem for Alg.~\ref{alg:rl} is shown below.
\begin{theorem}\label{thm:rl-main} 
	Let $c_1=HL_0^2(4H^2\Ph+4H^4\Pt+2H^2\Pg+1)+HL_1^2(8H^2\Pg+2)+4HL_2^2$ and $C_1=2+\frac{\gasup}{\hasup}$. Let $\srcrl_T$ be the sequential Rademacher complexity for the loss function defined in Eq.~\eqref{equ:rl-loss}. 
	Under Assumption~\ref{assumption:lipschitz}-\ref{assumption:RL-smoothness}, for any $\epsilon\le \mintwo{1}{\frac{\tasup}{16}},$ we can bound the  $(\epsilon,6\sqrt{\tasup\epsilon})$-regret of Alg.~\ref{alg:rl} by
	\begin{equation}
		\E\bbra{\reg_{\epsilon,6\sqrt{\tasup\epsilon}}(T)}\le \pbra{1+C_1\sqrt{2c_1T\srcrl_T}}\maxtwo{2\hasup\epsilon^{-2}}{\sqrt{\tasup}\epsilon^{-3/2}}.
	\end{equation}
\end{theorem}

\paragraph{Instantiation of Theorem~\ref{thm:rl-main} on Example~\ref{example:1}.} Applying Theorem~\ref{thm:rl-main} to the Example~\ref{example:1} with $\sigma=\Theta(1)$ we get the local regret guarantee\footnote{Since local regret is monotonic w.r.t. $\epsilon$, we can rescale $\epsilon$ to get a $(\epsilon,\sqrt{\epsilon})$-local regret bound by paying an additional $\sqrt{\tasup}$ factor.}: 
\begin{equation}
\E\bbra{\reg_{\epsilon,\sqrt{\epsilon}}(T)}=\bigO\pbra{\poly(\sigma,1/\sigma,H,L_r)T^{3/4}\log|\Theta|^{1/4}\epsilon^{-2}},\label{eqn:3}
\end{equation}

\paragraph{Comparison with policy gradient. } 
To the best of our knowledge, the best analysis for policy gradient~\citep{williams1992simple} shows convergence to a local maximum with a sample complexity that depends polynomially on $\normtwo{\gpsi \log\pi_\psi(a\mid s)}$ \citep{agarwal2019theory}.  For the instance in in Example~\ref{example:1}, this translates to a sample complexity guarantee on the order of $\sqrt{d}/\sigma$. In contrast, our local regret bound in equation ~\eqref{eqn:3} is independent of the dimension $d$. Instead, our bound depends on the complexity of the model family $\Theta$ which could be much smaller than the ambient dimension---this demonstrates that we leverage the model extrapolation.

More generally, %
letting $g(s,a)=\gpsi \log\pi_\psi (a\mid s)$, the variance of REINFORCE estimator is given  by $\E[\normtwo{g(s,a)}^2]$, which eventually shows up in the sample complexity bound. 
In contrast, our bound depends on $\normop{\E[g(s,a)g(s,a)^\top]}$. 
The difference between $\E[\normtwo{g(s,a)}^2]$ and $\normop{\E[g(s,a)g(s,a)^\top]}$ can be as large as a factor of $\da$ when $g(s,a)$ is isotropic. It's possible that our bound is dimension-free and the bound for policy gradient is not (e.g., as in Example~\ref{example:1}). 
We can also consider a more general Gaussian stochastic policy in deep RL~\citep{schulman2017proximal,schulman2015trust}: $\pi_\psi(s)\sim \mu_\psi(s)+\calN(0,\sigma^2I),$ where $\mu_\psi$ is a neural network and $\sigma$ is a constant . We have $g(s,a)=\frac{\partial\mu_\psi(s)}{\partial \psi}\frac{1}{\sigma^2}(\mu_\psi(s)-a).$ It follows that if $\normsp{\frac{\partial\mu_\psi(s)}{\partial \psi}} \approx 1$, then $\E[\normtwo{g(s,a)}^2]\approx \da$. On the other hand, %
$\normop{\E_a[g(s,a)g(s,a)^\top]}$ can be bounded by $O(1)$ if  $g(s,a)$ is isotropic.

\section{Lower Bounds}\label{sec:lowerbounds}
We prove several lower bounds to show (a) the hardness of finding global maxima, (b) the inefficiency of using optimism in nonlinear bandit, and (c) the hardness of stochastic environments.

\paragraph{Hardness of Global Optimality.}
In the following theorem, we show it statistically intractable to find the global optimal policy when the function class is chosen to be the neural networks with ReLU activation. That is, the reward function can be written in the form of $\eta((w,b),a)=\mathrm{ReLU}(\dotp{w}{a}+b).$ Note that the reward function can also be made smooth by replacing the activation by a smoothed version. For example, $\eta((w,b),a)=\mathrm{ReLU}(\dotp{w}{a}+b)^2$. We visualize the reward function in Fig.~\ref{fig:local-1}.

\begin{theorem} \label{thm:lowerbound-sample-comp}
    When the function class is chosen to be one-layer neural networks with ReLU activation, the minimax sample complexity is $\Omega(\varepsilon^{-(d-2)}).$
\end{theorem}
 We can also prove that the eluder dimension of the constructed reward function class is exponential. 
\begin{theorem}\label{thm:lowerbound-eluder}  The  $\varepsilon$-eluder dimension of one-layer neural networks is at least $\Omega(\varepsilon^{-(d-1)}).$
\end{theorem} 
 This result is concurrently established by \citet[Theorem 8]{li2021eluder}. The proofs of both theorems are deferred to Appendices~\ref{app:proof-sample-comp} and~\ref{app:proof-hardness-eluder}, respectively. We also note that Theorem~\ref{thm:lowerbound-sample-comp} does require ReLU activation, because if the ReLU function is replaced by a \textit{strictly} monotone link function with bounded derivatives (up to third order), then this is the setting of deterministic generalized linear bandit problem, which does allow a global regret that depends polynomially on dimension \citep{filippi2010parametric,dong2019performance,li2017provably}.
In this case, our Theorem~\ref{thm:bandit-main} can also give polynomial global regret result: because all local maxima of the reward function is global maximum~\citep{hazan2015beyond,kakade2011efficient} and it also satisfies the strict-saddle property~\citep{ge2015escaping}, the local regret result translates to a global regret result. This shows that our framework does separate the intractable cases from the tractable by the notions of local and global regrets.

With two-layer neural networks, we can relax the use of ReLU activation---Theorem~\ref{thm:lowerbound-eluder} holds with two-layer neural networks and leaky-ReLU activations~\citep{xu2015empirical} because $O(1)$ leaky-ReLU can implement a ReLU activation. We conjecture that with more layers, the impossibility result also holds for a broader sets of activations.  

\begin{figure}
	\centering
	\begin{subfigure}[b]{0.3\textwidth}
		\centering
		\includegraphics[width=.8\textwidth]{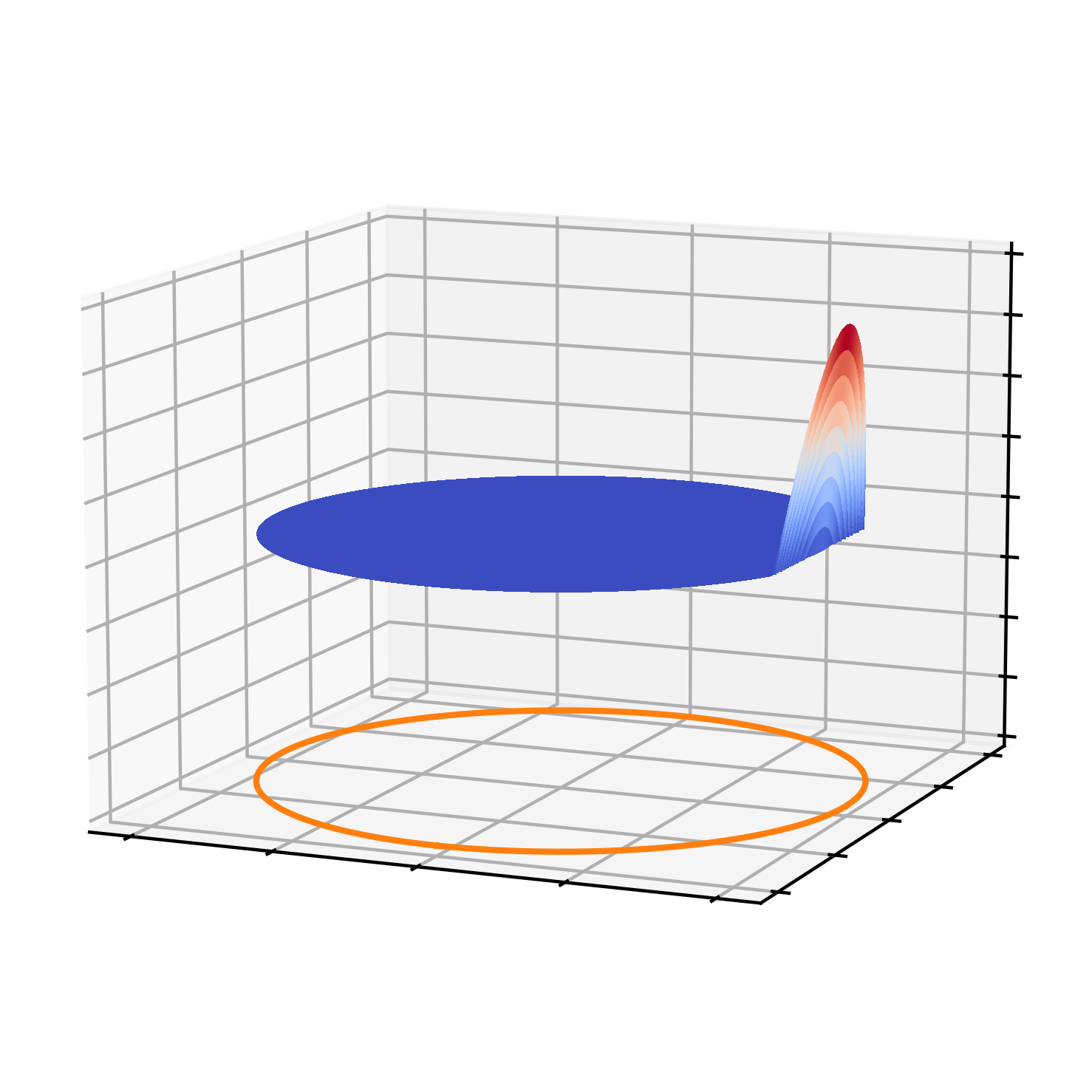}
		\caption{\footnotesize$\relu(\langle\theta,a\rangle-0.9)$}
		\label{fig:local-1a}
	\end{subfigure}
	\begin{subfigure}[b]{0.3\textwidth}
		\centering
		\includegraphics[width=.8\textwidth]{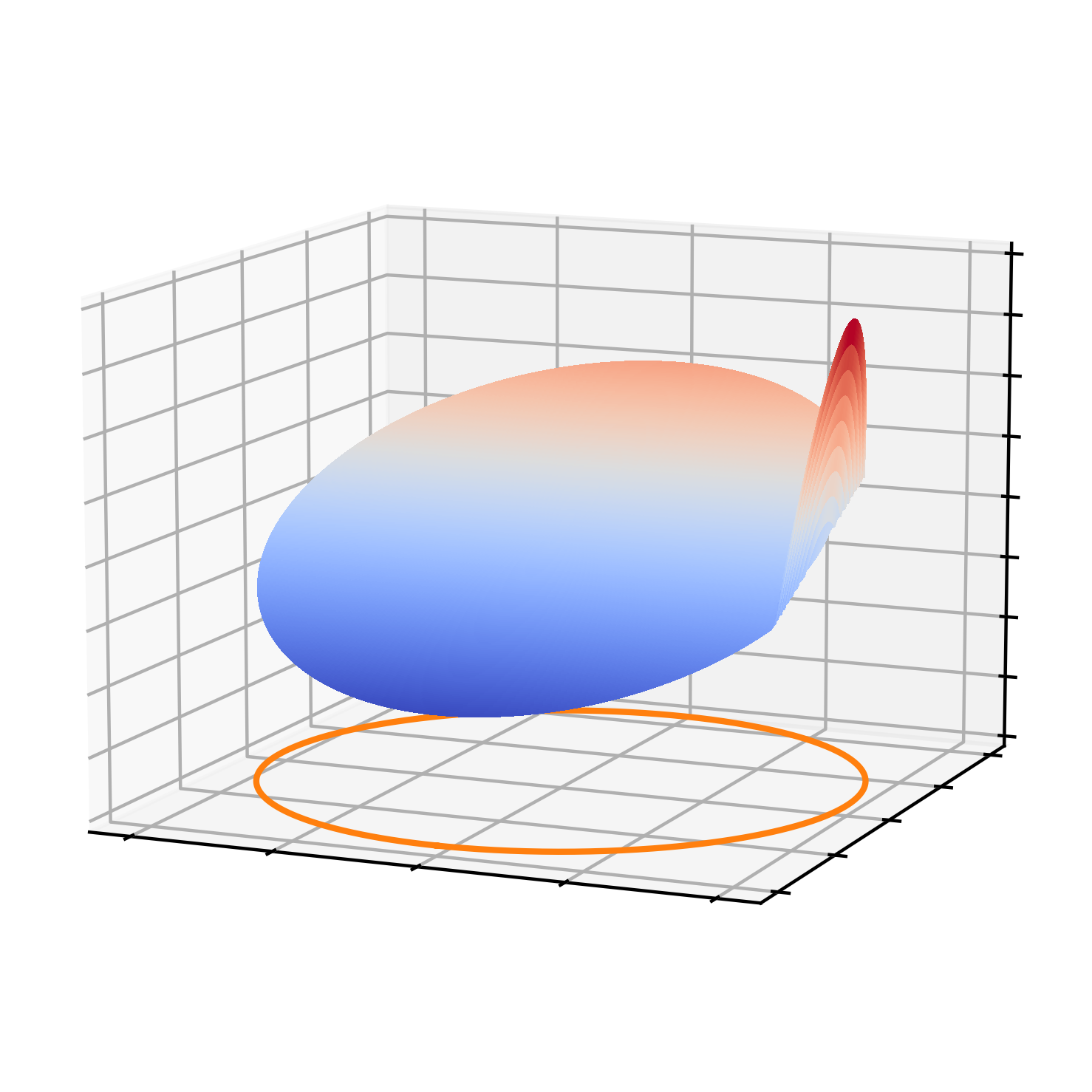}
		\caption{\footnotesize$\langle\theta_1,a\rangle+c\cdot \relu(\langle\theta_2,a\rangle-0.9)$}
		\label{fig:local-1b}
	\end{subfigure}
	\begin{subfigure}[b]{0.3\textwidth}
		\centering
		\includegraphics[width=.8\textwidth]{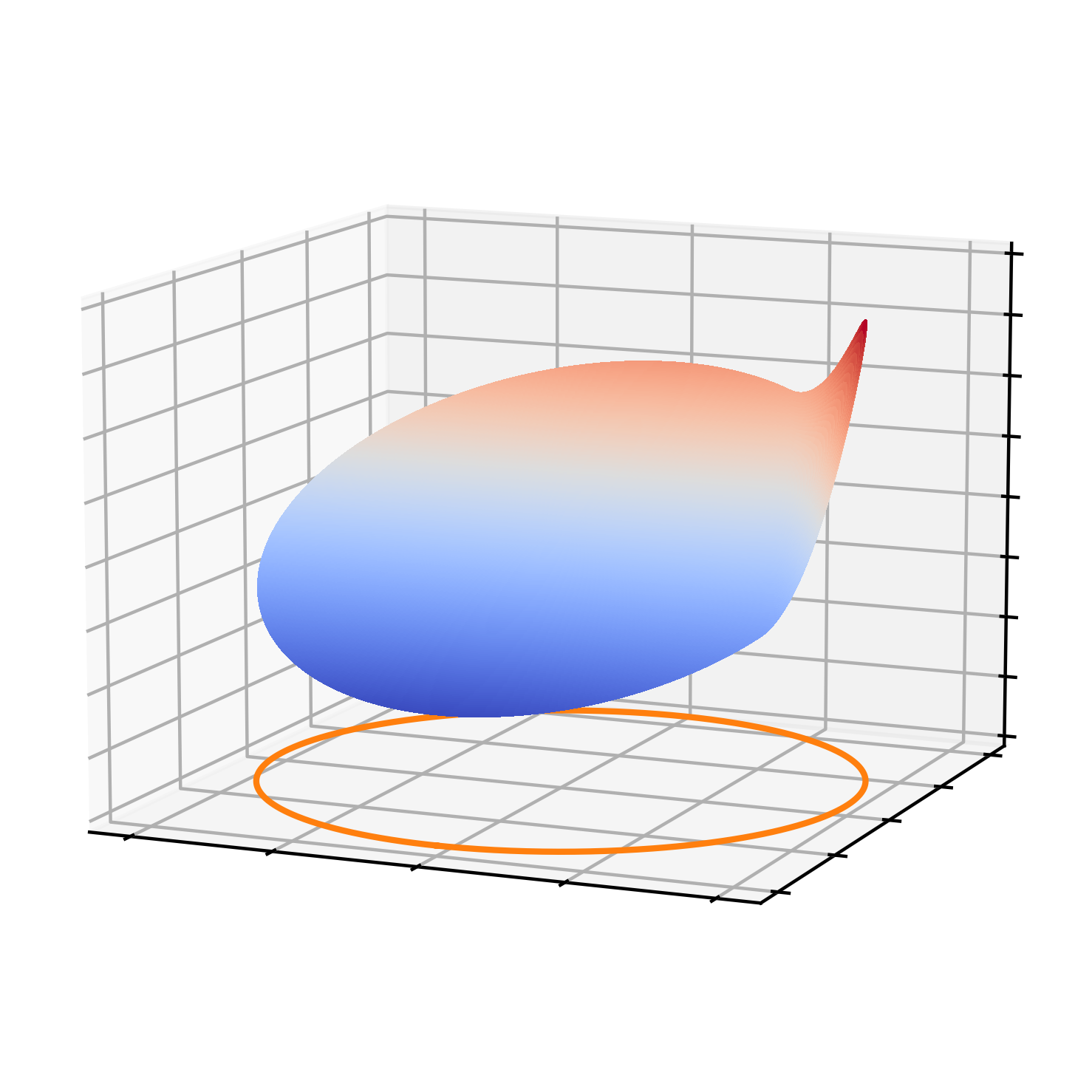}
		\caption{\footnotesize$\langle\theta_1,a\rangle+c\cdot \relu(\langle\theta_2,a\rangle-0.9)^2$}
		\label{fig:local-1c}
	\end{subfigure}
	\caption{Visualization of the reward function in our hard instances for non-linear bandit. \textbf{(a).} Hard instances for Theorem~\ref{thm:lowerbound-sample-comp} and~\ref{thm:lowerbound-eluder}. Reward function is $\eta(a)=\relu(\langle\theta,a\rangle-0.9)$. There is a large flat region with zero reward, and the region with non-zero reward is exponentially small. Therefore finding the global maximum requires exponential  samples. \textbf{(b).} Hard instances for Theorem~\ref{thm:lowerbound-optimism}. The reward is a linear part plus the same nonlinear part so there is no flat region. The UCB algorithm in this case doesn't even converge in polynomial samples because it keeps guessing the nonlinear part. \textbf{(c).} A smoothed version of (b). Our \AlgBan~algorithm converges to local maxima in polynomial samples, but UCB algorithm still doesn't converge. This is also a \textit{smooth} hard instance for the setting of Theorem~\ref{thm:lowerbound-sample-comp}---one can prove with similar techniques that it requires exponential samples to find the global maximum.}
	\label{fig:local-1}
\end{figure}

\paragraph{Inefficiency caused by optimism in nonlinear models.} In the following we revisit the optimism-in-face-of-uncertainty principle. First we recall the UCB algorithm in deterministic environments.

We formalize UCB algorithm under deterministic environments as follows. At every time step $t$, the algorithm maintains a upper confidence bound $C_t:\calA\to \R$. The function $C_t$ satisfies $\eta(\thetas,a)\le C_t(a).$ And then the action for time step $t$ is $a_t\gets \arg\max C_t(a)$. 
Let $\Theta_t$ be the set of parameters that is consistent with $\eta(\thetas,a_1),\cdots,\eta(\thetas,a_{t-1}).$ That is, $\Theta_t=\{\theta\in \Theta:\eta(\theta,a_\tau)=\eta(\thetas,a_\tau),\forall \tau<t\}.$ 
In a deterministic environment, the tightest upper confidence bound is $C_t(a)=\sup_{\theta\in \Theta_t}\eta(\theta,a).$

The next theorem states that the UCB algorithm that uses optimism-in-face-of-uncertainty principle can overly explore in the action space, even if the ground-truth is simple. 
\begin{theorem}\label{thm:lowerbound-optimism}
	Consider the case where the ground-truth reward function is linear: $\dotp{\thetas}{a}$ and the action set is $a\in S^{d-1}.$ If the hypothesis is chosen to be two-layer neural network with width $d$, UCB algorithm with tightest upper confidence bound  suffers exponential sample complexity .
\end{theorem}
Proof of the theorem is deferred to Appendix~\ref{app:proof-hardness-optimism}. 
Informally speaking, we prove the theorem by showing that one optimistic exploration step only eliminates a (exponentially) small portion of the hypothesis, because the optimistic action is less informative. Similar constructions also appear in proving the inefficiency of UCB algorithm on contextual bandits \citep[Proposition 1]{foster2018practical}. \citet[Section 25.2]{lattimore2020bandit} and \citet{hao2020adaptive} also show that optimistic algorithms is suboptimal for linear bandits.

\paragraph{Hardness of stochastic environments.} As a motivation to consider deterministic rewards, the next theorem proves that a $\poly\pbra{\log |\Theta|}$ sample complexity is impossible for finding local optimal action even under mild stochastic environment.

\begin{theorem}\label{thm:lowerbound-stochastic}
	There exists an bandit problem with stochastic reward and hypothesis class with size $\log |\Theta|=\tildeO\pbra{1}$, such that any algorithm requires $\Omega\pbra{d}$ sample to find a $(0.1,1)$-approximate second order stationary point with probability at least $3/4.$
\end{theorem}
A similar theorem is proved in \citet[Section 23.3]{lattimore2020bandit} (in a somewhat different context) with minor differences in the constructed hard instances. Our hard instance is linear bandit with hypothesis $\Theta=\{e_1,\cdots,e_d\}$, action space $\calA=S^{d-1}$ and i.i.d. standard Gaussian noise. Intuitively, the hardness comes from low signal-to-noise ratio because $\min_i\abs{\dotp{a}{e_i}}\le 1/\sqrt{d}$ for any $a\in \calA$. In other words, in the worse case the signal-to-noise ration is $\bigO\pbra{1/\sqrt{d}}$, which leads to a sample complexity that depends on $d$. We defer the formal proof to Appendix~\ref{app:proof-hardness-stochastic}.

\section{Additional Related Work}\label{sec:relatedwork}
There are several provable efficient algorithms without optimism for contextual bandit. 
The algorithms in \citet{dudik2011efficient,agarwal2014taming} instantiate mirror descent for contextual bandit, and the regret bounds depend polynomially on the number of actions. \citet{foster2020beyond} and \citet{simchi2020bypassing} exploit a particular exploration probability that is approximately the inverse of empirical gap. The SquareCB algorithm \citep{foster2020beyond} also extends to infinite action, but with a polynomial dependence on the action dimension in regret bound. The exploration strategy in \citet{foster2020beyond} and \citet{simchi2020bypassing} seems to rely on the structure in the action space whereas ours exploits the extrapolation in the model space. 
Recently, \citet{foster2020instance} prove an instance-dependent regret bound for contextual bandit.

\citet{zhou2020neural} also consider non-linear contextual bandits and propose the NeuralUCB algorithm by leveraging the NTK approach. It converges to a globally optimal solution with number of samples polynomial in the number of contexts and the number of actions. For nonlinear bandit problems with continuous actions, the sample complexity depends on the ``effective dimension'' of some kernel matrix. Because no algorithm can find the global maximum for the hard instances in Theorem~\ref{thm:lowerbound-sample-comp} with polynomial samples, the effective dimension should be exponential in dimension for the hard instances.

The deterministic nonlinear 
bandit problem can also be formulated as zero-order optimization without noise (see \citet{duchi2015optimal, liu2020primer} and references therein), where the reward class is assumed to be all 1-Lipschitz functions. In contrast, our algorithm exploits the knowledge of the reward function parametrization and achieves an action-dimension-free regret.
In the setting of stochastic nonlinear bandit, \citet{filippi2010parametric} consider generalized linear model. \citet{valko2013finite,zhou2020neural} focus on rewards in a Reproducing Kernel Hilbert Space (RKHS) and neural network (in the Neural Tangent Kernel regime) respectively, and provide algorithms with sublinear regret. \citet{yang2020bridging} extends this line of research to reinforcement learning setting.

Another line of research focuses on solving reinforcement learning by running optimization algorithms on the policy space. \citet{agarwal2019theory} prove that natural policy gradient can solve tabular MDPs efficiently. \citet{cai2020provably} incorporate exploration bonus in proximal policy optimization algorithm and achieves polynomial regret in linear MDP setting. \citet{hazan2017efficient} also work on local regret, but in the setting of online non-convex games. Their local regret notation is different from ours.

Beyond linear function approximations, there are also extensive studies on various settings that allow efficient algorithms. For example, rich observation MDPs \citep{krishnamurthy2016pac,dann2018oracle,du2019provably,misra2020kinematic}, state aggregation \citep{dong2019provably,li2009unifying}, Bellman rank \citep{jiang2016contextual,dong2020root} and others \citep{du2021bilinear,littman2001predictive,munos2005error,zanette2020learning,kakade2020information}.
\section{Conclusion}
In this paper, we design new algorithms whose local regrets are bounded by the sequential Rademacher complexity of particular loss functions. By rearranging the priorities of exploration versus exploitation, our algorithms avoid over-aggressive explorations caused by the optimism in the face of uncertainty principle, and hence apply to nonlinear models and dynamics. We raise the following questions as future works:

\begin{itemize}
	\item[1.] Since we mainly focus on proving a regret bound that depends only on the complexity of dynamics/reward class, our convergence rate in $T$ is likely not minimax optimal. Can our algorithms (or analysis) be modified to achieve minimax optimal regret for some of the instantiations such as sparse linear bandit and linear bandit with finite model class?
	\item[2.] In the bandit setting, we focus on deterministic reward because our \AlgBan~algorithm relies on finite difference to estimate the gradient and Hessian of reward function. In fact, Theorem~\ref{thm:lowerbound-stochastic} shows that action-dimension-free regret bound for linear models is impossible under standard Gaussian noise. Can we extend our algorithm to stochastic environments with additional assumptions on noises? 
	\item[3.] In the reinforcement learning setting, we use policy gradient lemma to upper bound the gradient/Hessian loss by the dynamics loss, which inevitable require the policies being stochastic. Despite the success of stochastic policies in deep reinforcement learning, the optimal policy may not be stochastic. Can we extend the \AlgBan~algorithm to reinforcement learning problems with deterministic policy hypothesis?
\end{itemize}

\subsection*{Acknowledgment}
The authors would like to thank Yuanhao Wang, Daogao Liu, Zhizhou Ren, Jason D. Lee, Colin Wei, Akshay Krishnamurthy, Alekh Agarwal and Csaba Szepesv{\'a}ri for helpful discussions. TM is also partially supported by the Google Faculty Award, Lam Research, and JD.com. 

\bibliographystyle{plainnat}
\bibliography{all.bib}

\newpage
\appendix

\section*{List of Appendices}
\startcontents[sections]
\printcontents[sections]{l}{1}{\setcounter{tocdepth}{2}}
\newpage

\section{Missing Proofs in Section~\ref{sec:lowerbounds}}\label{sec:app:pf5}%

In this section, we prove several negative results.

\subsection{Proof of Theorem~\ref{thm:lowerbound-sample-comp}}

\label{app:proof-sample-comp}
\begin{proof}
    We consider the class $\calI = \{I_{\theta, \varepsilon} : \norm{\theta}_2 \le 1,\varepsilon>0\}$ of infinite-armed bandit instances, where in the instance $I_{\theta, \varepsilon}$, the reward of pulling action $x \in \sB_2^d(1)$ is deterministic and is equal to \begin{equation}
    \eta(I_{\theta, \varepsilon}, x) = Ax\{ \langle x , \theta\rangle - 1 + \varepsilon, 0\}.\end{equation}
    
    We prove the theorem by proving the  minimax regret. The sample complexity then follows from the canonical sample complexity-regret reduction \citep[Section 3.1]{jin2018q}. Let $\calA$ denote any algorithm. Let  $R^T_{\calA, I}$ be the $T$-step regret of algorithm $\calA$ under instance $I $. Then we have 
\begin{align*}
    \inf_{\calA} \sup_{I \in \calI} \E[R^T_{\calA, I}] \ge \Omega(T^{\frac{d-2}{d-1}}).
\end{align*} 
    Fix $\varepsilon = c \cdot T^{-1/(d-1)}$. Let $\Theta$ be an $\varepsilon$-packing of the sphere $\{x \in \sR^d : \lVert x \rVert_2 = 1\}$. Then we have $\lvert \Theta \rvert \ge \Omega(\varepsilon^{-(d-1)}).$ So we choose $c > 0$ to be a numeric constant such that $T \le \abs{\Theta} / 2$. Let $\mu$ be the distribution over $\Theta$ such that $\mu(\theta) = \Pr[\exists t \le T \text{ s.t. } \eta(I_{\theta, \varepsilon}, \va_t) \ne 0 \text{ when }r_\tau \equiv 0 \text{ for }\tau = 1, \ldots, T]$. Note that for any action $\va_t \in \sB_2^d(1)$, there is at most one $\theta \in \Theta$ such that $\eta(I_{\theta,\varepsilon}, \va_t) \ne 0$, because $\Theta$ is a packing. Since $T \le \abs{\Theta} / 2$, there exists $\theta^* \in \Theta$ such that $\mu(\theta^*) \le 1/2$. Therefore, with probability $1/2$, the algorithm $\calA$ would obtain reward $r_t = \eta(I_{\theta^*, \varepsilon}, \va_t) = 0$ for every time step $t = 1, \ldots, T$. Note that under instance $I_{\theta^*, \varepsilon}$, the optimal action is to choose $\va_t \equiv \theta^*$, which would give reward $r_t^* \equiv \varepsilon$. Therefore, with probability $1/2$, we have $\E[R^T_{\calA, I_{\theta, \varepsilon}}] \ge \varepsilon T/2 \ge \Omega(T^{\frac{d-2}{d-1}})$.
    
\end{proof}

\subsection{Proof of Theorem~\ref{thm:lowerbound-eluder}} \label{app:proof-hardness-eluder}

\begin{proof} 
    We adopt the notations from Appendix~\ref{app:proof-sample-comp}. We use $\dim_E(\calF, \varepsilon)$ to denote the $\varepsilon$-eluder dimension of the function class $\calF.$ Let $\Theta$ be an $\varepsilon$-packing of the sphere $\{x \in \sR^d : \lVert x \rVert_2 = 1\}$. We write $\Theta = \{\theta_1, \ldots, \theta_n\}.$ Then we have $n \ge \Omega(\varepsilon^{-(d-1)}).$ Next we establish that $\dim_E(\calF, \varepsilon) \ge \Omega(\varepsilon^{-(d-1)}).$ For each $i \in [n],$ we define the function $f_i(\va) = \eta(I_{\theta_i,\varepsilon},\va) \in \calF.$ Then for $i \le n-1,$ we have $f_i(\theta_j) = f_{i+1}(\theta_j)$ for $j \le i-1,$ while $\varepsilon=f_i(\theta_i) \ne f_{i + 1}(\theta_i) = 0.$ Therefore, $\theta_i$ is $\frac{\varepsilon}{2}$-independent of its predecessors. As a result, we have $\dim_E(\calF, \varepsilon) \ge n-1.$
\end{proof}

\subsection{Proof of Theorem~\ref{thm:lowerbound-optimism}}\label{app:proof-hardness-optimism}

We first provide a proof sketch to the theorem. We consider the following reward function. $$\eta((\thetas_1,\thetas_2,\alpha),a)=\frac{1}{64}\dotp{a}{\thetas_1}+\alpha\maxtwo{\dotp{\thetas_2}{a}-\frac{31}{32}}{0}.$$

Note that the reward function $\eta$ can be clearly realized by a two-layer neural network with width $2d$. When $\alpha=0$ we have $\eta((\thetas_1,\thetas_2,\alpha),a)=\frac{1}{64}\dotp{\thetas_1}{a},$ which represents a linear reward. Informally, optimism based algorithm will try to make the second term large (because optimistically the algorithm hopes $\alpha=1$), which leads to an action $a_t$ that is suboptimal for ground-truth reward (in which case $\alpha=0$). In round $t$, the optimism algorithm observes $\dotp{\thetas_2}{a_t}=0$, and can only eliminate an exponentially small fraction of $\thetas_2$ from the hypothesis. Therefore the optimism algorithm needs exponential number of steps to determine $\alpha=0$ and stops exploration. Formally, the prove is given below.

\begin{proof} %
	Consider a bandit problem where $\calA=S^{d-1}$ and  $$\eta((\thetas_1,\thetas_2,\alpha),a)=\frac{1}{64}\dotp{a}{\thetas_1}+\alpha\maxtwo{\dotp{\thetas_2}{a}-\frac{31}{32}}{0}.$$ The hypothesis space is $\Theta=\{\theta_1,\theta_2,\alpha:\normtwo{\theta_1}\le 1,\normtwo{\theta_2}\le 1, \alpha\in[0,1]\}.$ Then the reward function $\eta$ can be clearly realized by a two-layer neural network with width $d$. Note that when $\alpha=0$ we have $\eta((\thetas_1,\thetas_2,\alpha),a)=\frac{1}{64}\dotp{\thetas_1}{a},$ which represents a linear reward. In the following we use $\vthetas=(\thetas_1,\thetas_2,0)$ as a shorthand.
	
	The UCB algorithm is described as follows. At every time step $t$, the algorithm maintains a upper confidence bound $C_t:\calA\to \R$. The function $C_t$ satisfies $\eta(\vthetas,a)\le C_t(a).$ And then the action for time step $t$ is $a_t\gets \argmax C_t(a)$. %
	
	Let $\calP=\{p_1,p_2,\cdots,p_n\}$ be an $\frac{1}{2}$-packing of the sphere $S^{d-1},$ where $n=\Omega(2^{d}).$ Let $\calB(p_i,\frac{1}{4})$ be the ball with radius $1/4$ centered at $p_i,$ and $B_i=\calB(p_i,\frac{1}{4})\cup S^{d-1}.$ We prove the theorem by showing that the UCB algorithm will explore every packing in $\calP$. That is, for any $i\in [n]$, there exists $t$ such that $a_t\in B_i$. Since we have $\sup_{a_j\in B_j}\dotp{p_i}{a_j}\le 31/32$ for all $j\neq i,$ this over-exploration strategy leads to a sample complexity (for finding a $(31/2048)$-suboptimal action) at least $\Omega(2^d)$ when $\vthetas=(p_i,p_i,0).$
	
	Let $\Theta_t$ be the set of parameters that is consistent with $\eta(\vthetas,a_1),\cdots,\eta(\vthetas,a_{t-1}).$ That is, $\Theta_t=\{\vtheta\in \Theta:\eta(\vtheta,a_\tau)=\eta(\vthetas,a_\tau),\forall \tau<t\}.$ Since our environment is deterministic, a tightest upper confidence bound is $C_t(a)=\sup_{\vtheta\in \Theta_t}\eta(\vtheta,a).$ Let $A_t=\{a_1,\cdots,a_{t}\}.$ It can be verified that for any $\theta_2\in S^{d-1},$ $\eta((\thetas_1,\theta_2,1),\cdot)$ is consistent with $\eta(\vthetas,\cdot)$ on $A_{t-1}$ if $\calB(\theta_2,\frac{1}{4})\cup A_{t-1}=\emptyset.$ As a result, for any $\theta_2$ such that $\calB(\theta_2,\frac{1}{4})\cup A_{t-1}=\emptyset$ we have 
	\begin{equation}\label{equ:optimism-1}
		C_t(\theta_2)\ge \frac{1}{32}> \frac{1}{128}+\sup_{a}\eta(\vthetas,a).
	\end{equation}
	
	Next we prove that for any $i\in[n]$, there exists $t$ such that $a_t\in \calB(p_i,\frac{1}{4}).$ Note that $\eta(\vtheta,\cdot)$ is $\frac{65}{64}$ Lipschitz for every $\vtheta\in \Theta.$ 
	As a result, $C_t(a_\tau+\xi)\le C_t(a_\tau)+\frac{65}{64}\normtwo{\xi}=\eta(\vthetas,a_\tau)+\frac{65}{64}\normtwo{\xi}$ for all $\tau<t.$
	Consequently, 
	\begin{align}\label{equ:optimism-2}
		C_t(a_\tau+\xi)\le \sup_{a}\eta(\vthetas,a)+\frac{1}{128}=\frac{3}{128}
	\end{align} for any $\tau<t$ and $\xi$ such that $\normtwo{\xi}\le \frac{1}{130}.$ 
	In other words, Eq.~\eqref{equ:optimism-2} upper bounds the upper confidence bound for actions that is taken by the algorithm, and Eq.~\eqref{equ:optimism-1} lower bounds the upper confidence bound for actions that is not taken.
	
	Now, for the sake of contradiction, assume that actions in $\calB(\theta_2,\frac{1}{4})$ is never taken by the algorithm. By Eq.~\eqref{equ:optimism-2} we have $C_t(\theta_2)\ge\frac{1}{32}$ for all $t.$ Let $\calH_t=\cup_{\tau=1}^{t-1}\calB(a_\tau,\frac{1}{130}).$ By Eq.~\eqref{equ:optimism-2} we have $C_t(a)\le \frac{3}{128}$ for all $a\in \calH_t$. Because $a_t\gets \argmax_{a}C_t(a)$ and $\max_{a\in\calH_t}C_t(a)<C_t(\theta_2)$, we conclude that $a_t\not\in \calH_t.$ Therefore, $\{a_t\}$ is a $(1/130)$-packing. However, the $(1/130)$-packing of $S^{d-1}$ has a size bounded by $130^d$, which leads to contradiction.
	
	For any $\theta_2$ there exists $t\le 130^{d}$ such that $a_t\in \calB(\theta_2,\frac{1}{4}).$
\end{proof}

\subsection{Proof of Theorem~\ref{thm:lowerbound-stochastic}}\label{app:proof-hardness-stochastic}
\begin{proof} %
	We consider a linear bandit problem with hypothesis class $\Theta=\{e_1,\cdots, e_d\}$. The action space is $S^{d-1}$. The stochastic reward function is given by $\eta(\theta,a)=\dotp{\theta}{a}+\xi$ where $\xi=\calN(0,1)$ is the noise. Define the set $A_i=\{a\in S^{d-1}:\abs{\dotp{a}{e_i}}\ge 0.9\}.$ By basic algebra we get, $A_i\cap A_j=\emptyset$ for all $i\neq j.$
	
	The manifold gradient of $\eta(\theta,\cdot)$ on $S^{d-1}$ is $$\grad \eta(\theta,a)=\pbra{I-aa^\top}\theta.$$ By triangular inequality we get $\normtwo{\grad \eta(\theta,a)}\ge \normtwo{\theta}-\abs{\dotp{a}{\theta}}.$ Consequently, $\normtwo{\grad \eta(\theta_i,a)}\ge 0.1$ for $a\not\in A_i.$ In other words, $\pbra{S^{d-1}\setminus A_i}$ does not contain any $(0.1,1)$-approximate second order stationary point for $\eta(\theta_i,\cdot).$
	
	For a fixed algorithm, let $a_1,\cdots,a_T$ be the sequence of actions chosen by the algorithm, and $x_t=\dotp{\thetas}{a_t}+\xi_t$. Next we prove that  with $T\lesssim d$ steps, there exists $i\in [d]$ such that $\Pr_i\bbra{a_T\in A_i}\le 1/2,$ where $\Pr_i$ denotes the probability space generated by $\thetas=\theta_i.$ Let $\Pr_0$ be the probability space generated by $\thetas=0.$ Let $E_{i,T}$ be the event that the algorithm outputs an action $a\in A_i$ at time step $T$. By Pinsker inequality we get,
	\begin{align}
		\E_{i}\bbra{E_{i,T}} \le \E_0\bbra{E_{i,T}}+\sqrt{\frac{1}{2}\KL\pbra{\Pr\nolimits_i,\Pr\nolimits_0}}.
	\end{align}
	Using the chain rule of KL-divergence and the fact that $\KL\pbra{\calN(0,1),\calN(a,1)}=\frac{a^2}{2},$ we get
	\begin{align}
		\E_{i}\bbra{E_{i,T}} \le \E_0\bbra{E_{i,T}}+\sqrt{\frac{1}{4}\E_0\bbra{\sum_{t=1}^{T}\dotp{a_t}{\theta_i}^2}}.
	\end{align}
	Consequently,
	\begin{align}
		\sum_{i=1}^{d}\E_{i}\bbra{E_{i,T}} \le\; &\sum_{i=1}^{d}\E_0\bbra{E_{i,T}}+\sum_{i=1}^{d}\sqrt{\frac{1}{4}\E_0\bbra{\sum_{t=1}^{T}\dotp{a_t}{\theta_i}^2}}\\
		\le\;& 1+\sqrt{\frac{d}{4}\E_0\bbra{\sum_{i=1}^{d}\sum_{t=1}^{T}\dotp{a_t}{\theta_i}^2}}\le 1+\sqrt{\frac{dT}{4}},
	\end{align}
	which means that 
	\begin{align}
		\min_{i\in [d]}\E_{i}\bbra{E_{i,T}} \le \frac{1}{d}+\sqrt{\frac{T}{4d}}.
	\end{align}
	Therefore when $T\le d$, there exists $i\in [d]$ such that $\E_{i}\bbra{E_{i,T}}\le \frac{3}{4}.$
\end{proof}

\section{Missing Proofs in Section~\ref{sec:main-results-bandit}}\label{sec:app:pf3}
In this section, we show missing proofs in Section~\ref{sec:main-results-bandit}.

\subsection{Proof of Lemma~\ref{lem:bandit-improvement}}\label{app:proof-bandit-improvement}
\begin{proof} %
	We prove the lemma by showing that algorithm~\ref{alg:ban} improves reward $\eta(\thetas,a_t)$ in the following two cases:
	\begin{itemize}
		\item[1.] $\normtwo{\ga \eta(\thetas, a_{t-1})}\ge \epsilon$, or
		\item[2.] $\normtwo{\ga \eta(\thetas, a_{t-1})}\le \epsilon$ and $\lmax\pbra{\ha \eta(\thetas, a_{t-1})}\ge 6\sqrt{\tasup\epsilon}$.
	\end{itemize}
	\paragraph{Case 1:} For simplicity, let $g_{t}=\ga \eta(\thetas,a_{t-1}).$ In this case we assume $\normtwo{g_{t}}\ge \epsilon.$ Define function 
	\begin{equation}
		\bar\eta_t(\theta,a)=\eta(\theta,a_{t-1})+\dotp{a-a_{t-1}}{\ga \eta(\theta,a_{t-1})}-\hasup\normtwo{a-a_{t-1}}^2
	\end{equation} to be the local first order approximation of function $\eta(\theta,a)$. By the Lipschitz assumption (namely, Assumption~\ref{assumption:lipschitz}), we have $\eta(\theta,a)\ge \bar{\eta}_t(\theta,a)$ for all $\theta\in \Theta,a\in \calA$. By the definition of $\err_{t,2}$ and $\err_{t,3}$, we get
	\begin{align}\label{equ:err-case-1}
		\bar{\eta}_t(\theta_t,a)\ge \bar\eta_t(\thetas,a)-\err_{t,2}-\normtwo{a-a_{t-1}}\err_{t,3}.
	\end{align}
	In this case we have
	\begin{align*}
		&\eta(\thetas,a_t)\ge\; \E_{\theta_t\sim p_t}\bbra{\eta(\theta_t,a_t)-\err_{t,1}}\\
		\ge\; &\sup_a\E_{\theta_t\sim p_t}\bbra{\eta(\theta_t,a)-\err_{t,1}}\tag{By the optimality of $a_t$}\\
		\ge\; &\sup_a\E_{\theta_t\sim p_t}\bbra{\bar{\eta}_t(\theta_t,a)-\err_{t,1}}\\
		\ge\; &\sup_a\E_{\theta_t\sim p_t}\bbra{\bar\eta_t(\thetas,a)-\err_{t,1}-\err_{t,2}-\normtwo{a-a_{t-1}}\err_{t,3}}\tag{By Eq.~\eqref{equ:err-case-1}}\\
		\ge\; &\E_{\theta_t\sim p_t}\bbra{\eta(\thetas,a_{t-1})+\frac{1}{4\hasup}\normtwo{g_t}^2-\err_{t,1}-\err_{t,2}-\frac{\normtwo{g_t}}{2\hasup}\err_{t,3}}&\tag{Take $a=a_{t-1}+\frac{g_t}{2\hasup}$}\\
		\ge\; &\eta(\thetas,a_{t-1})+\frac{\epsilon^2}{4\hasup}-\E_{\theta_t\sim p_t}\bbra{\pbra{2+\frac{\gasup}{\hasup}}\err_{t}}\tag{By Cauchy-Schwarz}
	\end{align*}
	\paragraph{Case 2:} Let $H_{t}=\ha \eta(\thetas,a_{t-1})$. Define $v_{t}\in \argmax_{v:\normtwo{v}=1}v^\top H_{t}v$. In this case we have $\normtwo{g_{t}}\le \epsilon$ and 
	\begin{align}\label{equ:eigenvalue-case-2}
		v_{t}^\top H_{t}v_{t}\ge 6\sqrt{\tasup\epsilon}\normtwo{v_{t}}^2.
	\end{align}
	Define function 
	\begin{align}
		\hat\eta_t(\theta,a)=\;&\eta(\theta,a_{t-1})+\dotp{a-a_{t-1}}{\ga \eta(\theta,a_{t-1})}\nonumber \\
		&+\frac{1}{2}\dotp{\ha \eta(\theta_t,a_{t-1})(a-a_{t-1})}{a-a_{t-1}}-\frac{\tasup}{2}\normtwo{a-a_{t-1}}^3
	\end{align} to be the local second order approximation of function $\eta(\theta,a)$. By the Lipschitz assumption (namely, Assumption~\ref{assumption:lipschitz}), we have $\eta(\theta,a)\ge \hat{\eta}_t(\theta,a)$ for all $\theta\in \Theta,a\in \calA$. 
	
	By Eq.~\eqref{equ:eigenvalue-case-2}, we can exploit the positive curvature by taking $a'=a_{t-1}+4\sqrt{\frac{\epsilon}{\tasup}}v_t$. Concretely, by basic algebra we get:
	\begin{align}
		\hat\eta_t(\thetas,a')&\ge \eta(\thetas,a_{t-1})-\epsilon\normtwo{a'-a_{t-1}}+3\sqrt{\tasup\epsilon}\normtwo{a'-a_{t-1}}^2-\frac{\tasup}{2}\normtwo{a'-a_{t-1}}^3\nonumber\\
		&\ge \eta(\thetas,a_{t-1})+12\sqrt{\frac{\epsilon^3}{\tasup}}.\label{equ:ascent-case-2}
	\end{align}
	Combining with the definition of $\err_{t,2}$, $\err_{t,3}$ and $\err_{t,4}$, for any $a\in \calA$ we get
	\begin{align}\label{equ:err-case-2t}
		\hat{\eta}_t(\theta_t,a)&\ge \hat\eta_t(\thetas,a)-\err_{t,2}-\normtwo{a-a_{t-1}}\err_{t,3}-\frac{1}{2}\normtwo{a-a_{t-1}}^2\err_{t,4}.
	\end{align}
	As a result, we have
	\begin{align*}
		&\eta(\thetas,a_t)\ge\; \E_{\theta_t\sim p_t}\bbra{\eta(\theta_t,a_t)-\err_{t,1}}\\
		\ge\; &\E_{\theta_t\sim p_t}\bbra{\eta(\theta_t,a')-\err_{t,1}}\tag{By the optimality of $a_t$}\\
		\ge\; &\E_{\theta_t\sim p_t}\bbra{\hat\eta_t(\thetas,a')-\err_{t,1}-\err_{t,2}-\normtwo{a-a_{t-1}}\err_{t,3}-\frac{1}{2}\normtwo{a-a_{t-1}}^2\err_{t,4}}\tag{By Eq.~\eqref{equ:err-case-2t}}\\
		\ge\; &\eta(\thetas,a_{t-1})+12\sqrt{\frac{\epsilon^3}{\tasup}}-\E_{\theta_t\sim p_t}\bbra{\err_{t,1}+\err_{t,2}+4\sqrt{\frac{\epsilon}{\tasup}}\err_{t,3}+\frac{8\epsilon}{\tasup}\err_{t,4}}\tag{By Eq.~\eqref{equ:ascent-case-2}}\\
		\ge\; &\eta(\thetas,a_{t-1})+12\sqrt{\frac{\epsilon^3}{\tasup}}-\E_{\theta_t\sim p_t}\bbra{2\err_{t}}\tag{When $16\epsilon\le \tasup$}.
	\end{align*}
	Combining the two cases together, we get the desired result.
\end{proof}

\subsection{Proof of Lemma~\ref{lem:bandit-concentration}}\label{app:proof-bandit-concentration}
\begin{proof}%
	Define $\calF_{t}$ to be the $\sigma$-field generated by random variable $u_{1:t},v_{1:t},\theta_{1:t}.$ In the following, we use $\E_t[\cdot]$ as a shorthand for $\E[\cdot\mid \calF_{t}].$ 
	
	Let $g_{t}=\ga \eta(\theta_t,a_{t-1})-\ga \eta(\thetas,a_{t-1})$. Note that condition on $\theta_t$ and $\calF_{t-1}$, $\dotp{g_t}{u_t}$ follows the distribution $\calN(0,\normtwo{g_t}^2).$ By Assumption~\ref{assumption:lipschitz}, $\normtwo{g_t}\le 2\gasup=\ubg.$ As a result,
	\begin{align}
		\E_{t-1}\bbra{\mintwo{\ubg^2}{\dotp{g_t}{u_t}^2}\mid \theta_t}\ge \frac{1}{2} \E_{t-1}\bbra{\dotp{g_t}{u_t}^2\mid \theta_t}=\frac{1}{2}\normtwo{g_t}^2.
	\end{align}
	By the tower property of expectation we get 
	\begin{align}\label{equ:bandit-lem2-part1}
		\E_{t-1}\bbra{\mintwo{\ubg^2}{\emperr_{t,3}^2}}\ge \frac{1}{2}\E_{t-1}\bbra{\err_{t,3}^2}.
	\end{align}
	
	Now we turn to the term $\emperr_{t,4}^2.$
	Let $H_t=\ha \eta(\theta_t,a_{t-1})-\ha \eta(\thetas,a_{t-1}).$ Define a random variable $x=\pbra{u_t^\top H_t v_t}^2.$ Note that $u_t,v_t$ are independent, we have
	\begin{align}
		\E_{t-1}[x\mid \theta_t]=\E_{t-1}\bbra{\normtwo{H_{t} v_t}^2\mid \theta_t}=\normF{H_{t}}^2\ge \normsp{H_{t}}^2.
	\end{align}
	Since $u_t,v_t$ are two Gaussian vectors, random variable $x$ has nice concentratebility properties. Therefore we can prove that the $\min$ operator in the definition of $\emperr_t$ does not change the expectation too much. Formally speaking, by Lemma~\ref{lem:concentration-matrix-min}, condition on $\calF_{t-1}$ and $\theta_t$, we have $\E\bbra{\mintwo{\ubh^2}{x}}\ge \frac{1}{2}\mintwo{\hasup^2}{\E\bbra{x}},$ which leads to 
	\begin{align}\label{equ:bandit-lem2-part2}
		&\E_{t-1}\bbra{\min(\ubh^2,\emperr_{t,4}^2)} \ge \frac{1}{2}\E_{t-1}\bbra{\min(\hasup^2,\normF{H_t}^2)}\ge \frac{1}{2}\E_{t-1}\bbra{\min(\hasup^2,\normsp{H_t}^2)}=\frac{1}{2}\E_{t-1}\bbra{\normsp{H_t}^2}.
	\end{align}
	
	Combining Eq.~\eqref{equ:bandit-lem2-part1} and Eq.~\eqref{equ:bandit-lem2-part2}, we get the desired inequality. 
\end{proof}

\subsection{Proof of Theorem~\ref{thm:bandit-main}}\label{app:proof-bandit-main}
\begin{proof}%
	Let $\delta_t=\inf_{a\in \sosp_{(\epsilon),6\sqrt{\tasup\epsilon}}}\eta(\thetas,a)-\eta(\thetas,a_t).$ By the definition of regret we have $\reg_{\epsilon,6\sqrt{\tasup\epsilon}}(T)=\sum_{t=1}^{T}\delta_t.$ Define $\upsilon=\mintwo{\frac{1}{4\hasup}\eps^2}{\frac{1}{\tasup^{1/2}}\epsilon^{3/2}}$ for simplicity. Recall that $C_1=2+\frac{\gasup}{\hasup}.$ In the following we prove by induction that for any $t_0$,
	\begin{align}\label{equ:bandit-induction}
		\E_{t_0-1}\bbra{\sum_{t=t_0}^{T}\delta_t}\le \E_{t_0-1}\bbra{\frac{1}{\upsilon}\pbra{\delta_{t_0}+C_1\sum_{t=t_0+1}^{T}\err_t}}.
	\end{align}
	
	For the base case where $t_0=T$ Eq.~\eqref{equ:bandit-induction} trivially holds because $\upsilon\le 1.$
	
	Now suppose Eq.~\eqref{equ:bandit-induction} holds for any $t>t_0$ and consider time step $t_0.$ When $a_{t_0}\not\in \sosp_{(\epsilon),6\sqrt{\tasup\epsilon}}$, applying Lemma~\ref{lem:bandit-improvement} we get $\eta(\thetas,a_{t_0+1})\ge \eta(\thetas,a_{t_0})+\upsilon-C_1\E_{t_0}[\err_{t_0+1}].$ By basic algebra we get,
	\begin{align}\label{equ:bandit-increase}
		\delta_{t_0+1}\le \delta_{t_0}-\upsilon+C_1\E_{t_0}[\err_{t_0+1}].
	\end{align}
	As a result,
	\begin{align*}
		\E_{t_0-1}\bbra{\sum_{t=t_0}^{T}\delta_t}&=\E_{t_0-1}\bbra{\delta_{t_0}+\sum_{t=t_0+1}^{T}\delta_t}\\
		&\le \E_{t_0-1}\bbra{\delta_{t_0}+\frac{1}{\upsilon}\pbra{\delta_{t_0+1}+C_1\sum_{t=t_0+2}^{T}\err_t}}\tag{By induction hypothesis}\\
		&\le \E_{t_0-1}\bbra{\delta_{t_0}-1+\frac{1}{\upsilon}\pbra{\delta_{t_0}+C_1\err_{t_0+1}+C_1\sum_{t=t_0+2}^{T}\err_t}}\tag{By Eq.~\eqref{equ:bandit-increase}}\\
		&\le\E_{t_0-1}\bbra{\frac{1}{\upsilon}\pbra{\delta_{t_0}+C_1\sum_{t=t_0+1}^{T}\err_t}}.
	\end{align*}
	On the other hand, when $a_{t_0}\in \sosp_{(\epsilon),6\sqrt{\tasup\epsilon}}$ we have 
	\begin{align*}
		&\eta(\thetas,a_{t_0+1})\ge \E_{\theta_{t_0+1}}\bbra{\eta(\theta_{t_0+1}, a_{t_0+1})-\err_{t_0+1,1}}\\
		\ge &\E_{\theta_{t_0+1}}\bbra{\eta(\theta_{t_0+1}, a_{t_0})-\err_{t_0+1,1}}\tag{By the optimality of $a_{t_0+1}$}\\
		\ge &\E_{\theta_{t_0+1}}\bbra{\eta(\thetas, a_{t_0})-\err_{t_0+1,1}-\err_{t_0+1,2}}\ge \eta(\thetas, a_{t_0})-C_1\E_{\theta_{t_0+1}}\bbra{\err_{t_0+1}}.
	\end{align*}
	Consequently, by basic algebra we get $\delta_{t_0+1}\le \delta_{t_0}+C_1\E_{t_0}[\err_{t_0+1}].$ Note that since $a_{t_0}\in \sosp_{(\epsilon),6\sqrt{\tasup\epsilon}}$, we have $\delta_{t_0}\le 0.$ As a result,
	\begin{align*}
		\E_{t_0-1}\bbra{\sum_{t=t_0}^{T}\delta_t}&\le  \E_{t_0-1}\bbra{\delta_{t_0}+\frac{1}{\upsilon}\pbra{\delta_{t_0+1}+C_1\sum_{t=t_0+2}^{T}\err_t}}\tag{By induction hypothesis}\\&\le  \E_{t_0-1}\bbra{\frac{1}{\upsilon}\pbra{\delta_{t_0+1}+C_1\sum_{t=t_0+2}^{T}\err_t}}\tag{$\delta_{t_0}\le 0$}\\
		&\le  \E_{t_0-1}\bbra{\frac{1}{\upsilon}\pbra{\delta_{t_0}+C_1\err_{t_0+1}+C_1\sum_{t=t_0+2}^{T}\err_t}}\\
		&\le  \E_{t_0-1}\bbra{\frac{1}{\upsilon}\pbra{\delta_{t_0}+C_1\sum_{t=t_0+1}^{T}\err_t}}.
	\end{align*}
	
	Combining the two cases together we prove Eq.~\eqref{equ:bandit-induction}. It follows that 
	\begin{align}
		&\E\bbra{\reg_{\epsilon,6\sqrt{\tasup\epsilon}}(T)}=\E\bbra{\sum_{t=0}^{T}\delta_t}\le \E\bbra{\frac{1}{\upsilon}\pbra{\delta_{0}+C_1\sum_{t=1}^{T}\err_t}}\\
		\le &\frac{1}{\upsilon}\pbra{1+C_1\E\bbra{\sqrt{T\sum_{t=1}^{T}\err_t^2}}}
		\le \frac{1}{\upsilon}\pbra{1+C_1\sqrt{T\E\bbra{\sum_{t=1}^{T}\err_t^2}}}.
	\end{align}
	Note that when realizability holds, we have $\inf_{\theta}\sum_{t=1}^{T}\ell((x_t,y_t);\theta)=0$. Therefore, by Lemma~\ref{lem:bandit-concentration} and the definition of online learning regret (see Eq.~\eqref{equ:def-online-learning}) we have
	\begin{align}
		\E\bbra{\reg_{\epsilon,6\sqrt{\tasup\epsilon}}(T)}\le &\frac{1}{\upsilon}\pbra{1+C_1\sqrt{2T\E\bbra{\sum_{t=1}^{T}\emperr_t^2}}}
		\le \frac{1}{\upsilon}\pbra{1+C_1\sqrt{4T\src_T}}.
	\end{align}
\end{proof}

\subsection{Proof of Theorem~\ref{thm:ban-action-space}}\label{app:action-convergence}
In this section we show that Alg.~\ref{alg:ban} finds a $(\epsilon,6\sqrt{\tasup\epsilon})$-approximate local maximum in polynomial steps. In the following, we treat $\gasup, \hasup, \tasup$ as constants.

\begin{proof}[Proof of Theorem~\ref{thm:ban-action-space}]
	We prove this theorem by contradiction. Suppose $\Pr\bbra{a_{t+1}\in \sosp_{\epsilon,6\sqrt{\tasup\epsilon}}}\le 0.5$ for all $t\in[T]$, we prove that $T\lesssim \frac{R}{\epsilon^4}\polylog(R,1/\epsilon).$
	
	Define $\upsilon=\mintwo{\frac{1}{4\hasup}\eps^2}{\frac{1}{\tasup^{1/2}}\epsilon^{3/2}}$. Recall that $C_1=2+\frac{\gasup}{\hasup}.$ By Lemma~\ref{lem:bandit-improvement}, when $a_t$ is not a $(\epsilon,6\sqrt{\tasup\epsilon})$-approximate local maximum we have
	\begin{align}
		\eta(\theta^\star,a_{t+1})\ge \eta(\thetas,a_{t})+\upsilon-C_1\E_{t}[\err_{t+1}].
	\end{align}
	Similar to the proof of Theorem~\ref{thm:bandit-main}, when $a_t$ is a $(\epsilon,6\sqrt{\tasup\epsilon})$-approximate local maximum we have
	\begin{align}
		\eta(\thetas,a_{t+1})\ge \eta(\thetas, a_{t})-C_1\E_{t}\bbra{\err_{t+1}}.
	\end{align}
	As a result, when $\Pr\bbra{a_{t+1}\in \sosp_{\epsilon,6\sqrt{\tasup\epsilon}}}\le 0.5$ we get
	\begin{align}\label{equ:action-conv-1}
		\E[\eta(\theta^\star,a_{t+1})]\ge \E[\eta(\thetas,a_{t})]+\frac{\upsilon}{2}-C_1\E\bbra{\err_{t+1}}.
	\end{align}
	Take summation of Eq.~\eqref{equ:action-conv-1} over $t\in [T]$ leads to 
	\begin{align}\label{equ:action-conv-2}
		\E[\eta(\theta^\star,a_{T})-\eta(\theta^\star,a_{0})]\ge \frac{\upsilon T}{2}-C_1\E\bbra{\sum_{t=1}^{T}\Delta_t}.
	\end{align}
	Lemma~\ref{lem:bandit-concentration} leads to 
	\begin{align}
		\E\bbra{\sum_{t=1}^{T}\Delta_t}\le \sqrt{2T\E\bbra{\sum_{t=1}^{T}\emperr_t^2}}\le 2T^{3/4}\pbra{R(\Theta)\polylog(T)}^{1/4}.
	\end{align}
	Combining with Eq.~\eqref{equ:action-conv-2} we have
	\begin{align}
		1\ge \E[\eta(\theta^\star,a_{T})-\eta(\theta^\star,a_{0})]\ge \frac{\upsilon T}{2}-2C_1T^{3/4}\pbra{R(\Theta)\polylog(T)}^{1/4}.
	\end{align}
	As a result, we can solve an upper bound of $T$. In particular, we get
	\begin{align}
		T\lesssim R(\Theta)\epsilon^{-8}\polylog(R(\Theta),1/\epsilon).
	\end{align}
	Consequently, when $T\gtrsim R(\Theta)\epsilon^{-8}\polylog(R(\Theta),1/\epsilon)$, there exists $t\in[T]$ such that $\Pr\bbra{a_{t+1}\in \sosp_{\epsilon,6\sqrt{\tasup\epsilon}}}> 0.5$.
\end{proof}

\subsection{Instantiations of Theorem~\ref{thm:bandit-main}}\label{app:instance}
In this section we rigorously prove the instantiations discussed in Section~\ref{sec:main-results-bandit}.

\paragraph{Linear bandit with finite model class.} 
Recall that the linear bandit reward is given by $\eta(\theta,a)=\dotp{\theta}{a}$, and the constrained reward is $\tilde{\eta}(\theta,a)=\eta(\theta,a)-\frac{1}{2}\normtwo{a}^2.$ 

In order to deal with $\ell_2$ regularization which violates Assumption~\ref{assumption:lipschitz}, we bound the set of actions Alg.~\ref{alg:ban} takes. Consider the regularized reward $\tilde{\eta}(\theta,a)$. When $\normtwo{a}>2$ we have $\tilde{\eta}(\theta,a)<0.$ Therefore the set of actions taken by Alg.~\ref{alg:ban} satisfies $\normtwo{a_t}\le 2$ for all $t$. Because we only apply Lemma~\ref{lem:bandit-improvement} and Lemma~\ref{lem:bandit-concentration} to actions that is taken by the algorithm, Theorem~\ref{thm:bandit-main} holds even if Assumption~\ref{assumption:lipschitz} is satisfied locally for $\normtwo{a}\lesssim 1.$ Since the gradient and Hessian of regularization term is $a$ and $I_{d}$ respectively, we have $\normtwo{\ga \tilde{\eta}(\theta,a)}\lesssim \normtwo{\ga \eta(\theta,a)}+1$ and $\normsp{\ha \tilde{\eta}(\theta,a)}\lesssim \normsp{\ha \eta(\theta,a)}+1$ when $\normtwo{a}\lesssim 1$, which verifies Assumption~\ref{assumption:lipschitz}.

In the following we translate the regularized local regret to the standard regret. 
Note that $\ga \tilde{\eta}(\theta,a)=\theta-a.$  As a result, the optimal action is given by $a^\star=\thetas.$ In addition, for any $a\in \sosp_{(\epsilon,1)}$ we have $\normtwo{\thetas-a}\le \epsilon.$ By algebraic manipulation we have
\begin{align}
	&\normtwo{\thetas-a}\le \epsilon\\
	\implies &\normtwo{\thetas}^2+\normtwo{a}^2-2\dotp{\thetas}{a}=\normtwo{\thetas-a}^2\le \epsilon^2\\
	\implies &\dotp{\thetas}{a}-\frac{1}{2}\normtwo{a}^2\ge \frac{1}{2}\normtwo{\thetas}^2-\frac{\epsilon^2}{2}=\dotp{\thetas}{\thetas}-\frac{1}{2}\normtwo{\thetas}^2-\frac{\epsilon^2}{2}.
\end{align}
Plug in the definition of regularized reward, for any $a\in\sosp_{(\epsilon,1)}$
\begin{align}
	\tilde{\eta}(\thetas,a)\ge \tilde{\eta}(\thetas,a^\star)-\frac{\epsilon^2}{2}.
\end{align}
Consequently,
\begin{align}
	&\E\bbra{\sum_{t=1}^{T}\pbra{\tilde{\eta}(\thetas,a^\star)-\tilde{\eta}(\thetas,a_t)}}\le \E\bbra{\sum_{t=1}^{T}\pbra{\inf_{a\in \sosp_{(\epsilon,1)}}\tilde{\eta}(\thetas,a)-\tilde{\eta}(\thetas,a_t)}}+\frac{\epsilon^2}{2}T\\
	&\quad \lesssim \epsilon^{-2}\sqrt{T\src_T}+\epsilon^2T,
\end{align}
where the last inequality follows from Theorem~\ref{thm:bandit-main}. Since the loss function $\ell$ is uniformly bounded by $v=\ubh^2+\ubg^2+4=O(1).$ By \citet{rakhlin2011online}, for finite hypothesis we have $\src_T\le v\sqrt{2T\log |\Theta|}.$ 
By choosing $\epsilon=T^{-1/16}(\log|\Theta|)^{1/16}$ we get 
\begin{align}
	\E\bbra{\sum_{t=1}^{T}\pbra{\tilde{\eta}(\thetas,a^\star)-\tilde{\eta}(\thetas,a_t)}}\lesssim T^{7/8}(\log |\Theta|)^{1/8}.
\end{align}

Now we bound the standard regret by the regularized regret. Recall that $\normtwo{\thetas}=1$. As a result,
\begin{align}\label{equ:lb-pf-1}
	\tilde{\eta}(\thetas,a^\star)-\tilde{\eta}(\thetas,a)=\frac{1}{2}-\dotp{\thetas}{a}+\frac{1}{2}\normtwo{a}^2=\frac{1}{2}\normtwo{\thetas-a}^2.
\end{align}
By Lemma~\ref{lem:helper-reduction} we have
\begin{align}\label{equ:lb-pf-2}
	&\normtwo{\thetas-a}^2\ge (1-\dotp{\thetas}{a})^2.
\end{align}
Combining Eq.~\eqref{equ:lb-pf-1} and Eq.~\eqref{equ:lb-pf-2} we get
\begin{align}
	&\E\bbra{\sum_{t=1}^{T}\pbra{\eta(\thetas,a^\star)-\eta(\thetas,a_t)}}=\E\bbra{\sum_{t=1}^{T}\pbra{1-\dotp{\thetas}{a_t}}}\\
	\le\;& \E\bbra{\sqrt{T\sum_{t=1}^{T}\pbra{1-\dotp{\thetas}{a_t}}^2}}
	\le\; \E\bbra{\sqrt{2T\sum_{t=1}^{T}\pbra{\tilde{\eta}(\thetas,a^\star)-\tilde{\eta}(\thetas,a_t)}}}\\
	\le\;& \sqrt{2T\E\bbra{\sum_{t=1}^{T}\pbra{\tilde{\eta}(\thetas,a^\star)-\tilde{\eta}(\thetas,a_t)}}}
	\lesssim\; T^{15/16}(\log|\Theta|)^{1/16}.
\end{align}

\paragraph{Linear bandit with sparse or structured model vectors.} In this case, the reduction is exactly the same as that in linear bandit. In the following we prove that the sparse linear hypothesis has a small covering number. Note that the $\log|\Theta|$ regret bound fits perfectly with the covering number technique. That is, we can discretize the hypothesis $\Theta$ by finding a $1/\poly(d T)$-covering of the loss function $\calL=\{\ell(\cdot,\theta):\theta\in \Theta\}$. And then the regret of our algorithm depends polynomially on the log-covering number. Since the log-covering number of the set of $s$-sparse vectors is bounded by $\bigO\pbra{s\log(d T)},$ we get the desired result. 

For completeness, in the following we prove that the Eluder dimension for sparse linear model is $\Omega(d).$
\begin{lemma}\label{lem:sparse-eluder}
	Let $e_1,\cdots,e_d$ be the basis vectors and $f_i(a)=\dotp{e_i}{a}.$ Specifically, define $f_0(a)=0.$ Define the function class $\calF=\{f_i:0\le i\le d\}.$ The Eluder dimension of $\calF$ is at least $d$.
\end{lemma}
\begin{proof}
	In order to prove the lower bound for Eluder dimension, we only need to find a sequence $a_1,\cdots,a_d$ such that $a_i$ is independent with its predecessors. In the sequel we consider the action sequence $a_1=e_1,a_2=e_2,\cdots,a_d=e_d.$
	
	Now we prove that for any $i\in [d]$, $a_i$ is independent with $a_j$ where $j<i.$ Indeed, consider functions $f_i$ and $f_0.$ By definition we have $f_i(a_j)=f_0(a_j),\forall j<i.$ However, $f_i(a_i)=1\neq 0=f_0(a_i).$
\end{proof}

\paragraph{Deterministic logistic bandits.} Recall that in this case the reward function is given by $\eta(\theta,a)=(1+e^{-\dotp{\theta}{a}})^{-1}$, and the regularized loss is $\tilde{\eta}(\theta,a)=\eta(\theta,a)-\frac{c}{2}\normtwo{a}^2$ where $c=e(e+1)^{-2}.$ 
By basic algebra we get
\begin{align}
	\nabla_a \tilde{\eta}(\theta,a)=\frac{\exp(-\theta^\top a)}{(1+\exp(-\theta^\top a))^2}\theta-c a.
\end{align}
As a result, we have $a^\star=\thetas.$

Note that $\tilde{\eta}(\thetas,\cdot)$ is $(1/20)$-strongly concave. As a result, for any $\epsilon\in \sosp_{(\epsilon,1)}$ we get
\begin{align}
	\tilde{\eta}(\thetas,a^\star)-\tilde{\eta}(\thetas,a)\lesssim \normtwo{\nabla_a \tilde{\eta}(\thetas,a)}^2\lesssim \epsilon^2.
\end{align}
Consequently,
\begin{align}
	&\E\bbra{\sum_{t=1}^{T}\pbra{\tilde{\eta}(\thetas,a^\star)-\tilde{\eta}(\thetas,a_t)}}\lesssim \E\bbra{\sum_{t=1}^{T}\pbra{\inf_{a\in \sosp_{(\epsilon,1)}}\tilde{\eta}(\thetas,a)-\tilde{\eta}(\thetas,a_t)}}+\frac{\epsilon^2}{2}T\\
	&\quad \lesssim \epsilon^{-2}\sqrt{T\src_T}+\epsilon^2T,
\end{align}
Since the loss function $\ell$ is uniformly bounded by $v=\ubh^2+\ubg^2+4=\bigO(1).$ By \citet{rakhlin2011online}, for finite hypothesis we have $\src_T\le v\sqrt{2T\log |\Theta|}.$ Choose $\epsilon=T^{-1/16}\log|\Theta|^{1/16}$ we get
\begin{align}
	\E\bbra{\sum_{t=1}^{T}\pbra{\tilde{\eta}(\thetas,a^\star)-\tilde{\eta}(\thetas,a_t)}}\lesssim T^{7/8}(\log|\Theta|)^{1/8}.
\end{align}

In the following we prove a reduction from standard regret to regularized regret. Define $r(x)\defeq(1+\exp(-x))$ for shorthand. By Taylor expansion, for any $x\in \R$ there exists $\xi\in \R$ such that $r(x)=r(1)+(x-1)r'(x)+(x-1)^2r''(\xi).$ As a result,
\begin{align*}
	&\tilde{\eta}(\thetas,a^\star)-\tilde{\eta}(\thetas,a)=r(1)-\frac{c}{2}-r(\dotp{\thetas}{a})+\frac{c}{2}\normtwo{a}^2\\
	=\;&(1-\dotp{\thetas}{a})r'(1)-(1-\dotp{\thetas}{a})^2r''(\xi)-\frac{c}{2}+\frac{c}{2}\normtwo{a}^2\\
	=\;&\frac{c}{2}-c\dotp{\thetas}{a}+\frac{c}{2}\normtwo{a}^2-(1-\dotp{\thetas}{a})^2r''(\xi) \tag{Recall that $r'(1)=c$.}\\
	=\;&\frac{c}{2}\normtwo{\thetas-a}^2-(1-\dotp{\thetas}{a})^2r''(\xi)\\
	\ge\;&\pbra{\frac{c}{2}-r''(\xi)}(1-\dotp{\thetas}{a})^2\tag{By Lemma~\ref{lem:helper-reduction}}\\
	\ge\;&(1-\dotp{\thetas}{a})^2/50\gtrsim (r(1)-r(\dotp{\thetas}{a}))^2\tag{The reward function is Lipschitz.}\\
	\gtrsim\;&(\eta(\thetas,a^\star)-\eta(\thetas,a))^2.
\end{align*}

As a result we have
\begin{align}
	&\E\bbra{\sum_{t=1}^{T}\pbra{\eta(\thetas,a^\star)-\eta(\thetas,a_t)}}
	\lesssim\; \E\bbra{\sqrt{T\sum_{t=1}^{T}\pbra{\eta(\thetas,a^\star)-\eta(\thetas,a_t)}^2}}\\
	\lesssim\;& \E\bbra{\sqrt{T\sum_{t=1}^{T}\pbra{\tilde{\eta}(\thetas,a^\star)-\tilde{\eta}(\thetas,a_t)}}}
	\lesssim\; \sqrt{T\E\bbra{\sum_{t=1}^{T}\pbra{\tilde{\eta}(\thetas,a^\star)-\tilde{\eta}(\thetas,a_t)}}}\\
	\lesssim\;& T^{15/16}(\log|\Theta|)^{1/16}.
\end{align}

\paragraph{Two-layer neural network.} 
Recall that a two-layer neural network is defined by $\eta((\mW_1,\mW_2),a)=\mW_2\sigma(\mW_1a)$ ,where $\sigma$ is the activation function. For a matrix $\mW_1\in \R^{m\times d},$ the $(1,\infty)$-norm is defined by $\max_{i\in [m]}\sum_{j=1}^{d}\abs{[\mW_1]_{i,j}}.$ We make the following assumptions regarding the activation function.
\begin{assumption}\label{assumption:two-layer}
For any $x,y\in \R$, the activation function $\sigma(\cdot)$ satisfies
\begin{align}
	&\sup_{x}\abs{\sigma(x)}\le 1,\quad\sup_{x}\abs{\sigma'(x)}\le 1,\quad\sup_{x}\abs{\sigma''(x)}\le 1,\\
	&\abs{\sigma''(x)-\sigma''(y)}\le \abs{x-y}.
\end{align}
\end{assumption}

The following theorem summarized our result in this setting.
\begin{theorem}\label{thm:two-layer}
	Let $\Theta=\{(W_1,W_2):\normone{\mW_2}\le 1,\norm{\mW_1}_{1,\infty}\le 1\}$ be the parameter hypothesis. 
	Under the setting of Theorem~\ref{thm:bandit-main} with Assumption~\ref{assumption:two-layer}, the local regret of Alg.~\ref{alg:ban} running on two-layer neural networks can be bounded by $\tildeO\pbra{\epsilon^{-2}T^{3/4}}.$ In addition, if the neural network is input concave, then the global regret of Alg.~\ref{alg:ban} is bounded by $\tildeO\pbra{T^{7/8}}.$ 
\end{theorem}
\begin{proof}
We prove the theorem by first bounding the sequential Rademacher complexity of the loss function, and then applying Theorem~\ref{thm:bandit-main}.
Let $\theta=(\mW_1,\mW_2).$ Recall that $u\odot v$ denotes the element-wise product. By basic algebra we get,
\begin{align}
	&\dotp{\ga \eta(\theta,a)}{u}=\mW_2\pbra{\sigma'(\mW_1 a)\odot \mW_1u},\\
	&u^\top \ha \eta(\theta,a)v=\mW_2\pbra{\sigma''(\mW_1 a)\odot \mW_1 u\odot \mW_1 v}.
\end{align}
First of all, we verify that the regularized reward $\tilde{\eta}(\theta,a)\defeq \eta(\theta,a)-\frac{1}{2}\normtwo{a}^2$ satisfies Assumption~\ref{assumption:lipschitz}. Indeed we have
\begin{align*}
	&\normtwo{\ga \eta(\theta,a)}=\sup_{u\in S^{d-1}}\dotp{\ga \eta(\theta,a)}{u}\le 1,\\
	&\normsp{\ha \eta(\theta,a)}=\sup_{u,v\in S^{d-1}} u^\top \ha \eta(\theta,a)v\le 1,\\
	&\normsp{\ha \eta(\theta,a_1)-\ha \eta(\theta,a_2)}=\sup_{u,v\in S^{d-1}} \mW_2\pbra{\pbra{\sigma''(\mW_1 a_1)-\sigma''(\mW_1a_2)}\odot \mW_1 u\odot \mW_1 v}\le \normtwo{a_1-a_2}.
\end{align*}
Observe that $\abs{\eta(\theta,a)}\le \norm{a}_\infty$, we have $\tilde{\eta}(\theta,a)<0$ when $\normtwo{a}>2.$ As a result, action $a_t$ taken by Alg.~\ref{alg:ban} satisfies $\normtwo{a_t}\le 2$ for all $t$. Since the gradient and Hessian of regularization term is $a$ and $I_{d}$ respectively, we have $\normtwo{\ga \tilde{\eta}(\theta,a)}\lesssim \normtwo{\ga \eta(\theta,a)}+1$ and $\normsp{\ha \tilde{\eta}(\theta,a)}\lesssim \normsp{\ha \eta(\theta,a)}+1.$
It follows that Assumption~\ref{assumption:lipschitz} holds with constant Lipschitzness for actions $a$ such that $\norm{a}\lesssim 1$.

In the following we bound the sequential Rademacher complexity of the loss function. By \citet[Proposition 15]{rakhlin2015online}, we can bound the sequential Rademacher complexity of $\err_{t,1}^2$ and $\err_{t,2}^2$ by $\tildeO\pbra{\sqrt{T\log d}}.$ Next we turn to higher order terms.

First of all, because the $(1,\infty)$ norm of $\mW_1$ is bounded, we have $\norm{\mW_1 u}_{\infty}\le \norm{u}_{\infty}.$ It follows from the upper bound of $\sigma'(x)$ that $\norm{\sigma'(\mW_1a)\odot\mW_1 u}_{\infty}\le \norm{u}_{\infty}.$ Therefore we get \begin{equation}
	\dotp{\ga \eta(\theta,a)}{u}\le \norm{\mW_2}_1\norm{\sigma'(\mW_1a)\odot\mW_1 u}_{\infty}\le \norm{u}_{\infty}.
\end{equation}
Similarly, we get 
\begin{equation}
	u^\top \ha \eta(\theta,a) v\le  \norm{u}_{\infty} \norm{v}_{\infty}.
\end{equation}

Let $B=(1+\norm{u}_{\infty})(1+\norm{v}_{\infty})$ for shorthand. We consider the error term $\emperr_{t,3}^2=\pbra{\dotp{\ga \eta(\theta,a)}{u}-[y_t]_3}^2.$ Let $\calG_1$ be the function class $\{\pbra{\dotp{\ga \eta(\theta,a)}{u}-[y_t]_3}^2:\theta\in \Theta\},$ and $\calG_2=\{\dotp{\ga \eta(\theta,a)}{u}:\theta\in \Theta\}.$ Applying \citet[Lemma 4]{rakhlin2015online} we get
$$\src_T(\calG_1)\lesssim B\log^{3/2}(T^2)\src_T(\calG_2).$$
Define $\calG_3=\{\sigma'(w_1^\top a)\cdot w_1^\top u:w_1\in \R^{d}, \normone{w_1}\le 1\}.$ In the following we show that $\src_T(\calG_2)\lesssim \src_T(\calG_3).$ For any sequence $u_1,\cdots,u_T$ and $\calA$-valued tree $\va$, we have
\begin{align}
	\src_T(\calG_2)&=\E_{\epsilon}\bbra{\sup_{\substack{\mW_2:\normone{\mW_2}\le 1\\g_1,\cdots,g_w\in \calG_3}}\sum_{t=1}^{T}\epsilon_t\pbra{\sum_{j=1}^{w}[\mW_2]_j g_j(\va_t(\epsilon))}}\\
	&\le \E_{\epsilon}\bbra{\sup_{\substack{\mW_2:\normone{\mW_2}\le 1\\g_1,\cdots,g_w\in \calG_3}}\normone{\mW_2}\sup_{j\in [w]}\abs{\sum_{t=1}^{T}\epsilon_t g_j(\va_t(\epsilon))}}\\
	&\le \E_{\epsilon}\bbra{\sup_{g\in \calG_3}\abs{\sum_{t=1}^{T}\epsilon_t\pbra{g_j(\va_t(\epsilon))}}}.
\end{align}
Since we have $0\in\calG_3$ by taking $w_1=0$, by symmetricity we have 
\begin{align}
	&\E_{\epsilon}\bbra{\sup_{g\in \calG_3}\abs{\sum_{t=1}^{T}\epsilon_t\pbra{g_j(\va_t(\epsilon))}}}
	\le 2\E_{\epsilon}\bbra{\sup_{g\in \calG_3}\sum_{t=1}^{T}\epsilon_t\pbra{g_j(\va_t(\epsilon))}}=2\src_T(\calG_3).
\end{align}

Now we bound $\src_T(\calG_3)$ by applying the composition lemma of sequential Rademacher complexity (namely \citet[Lemma 4]{rakhlin2015online}). First of all we define a relaxed function hypothesis
$\calG_4=\{\sigma'((w_1')^\top a)\cdot w_1^\top u:w_1,w_1'\in \R^{d}, \normone{w_1}\le 1, \normone{w_1'}\le 1\}.$ Since $\calG_3\subset \calG_4$ we have $\src_T(\calG_3)\le \src_T(\calG_4).$ 
Note that we have $\abs{\sigma'(w_1^\top a)}\le 1$ and $w_1^\top u\le \norm{u}_\infty.$ Let $\phi(x,y)=xy$, which is $(3c)$-Lipschitz for $\abs{x},\abs{y}\le c.$ Define $\calG_5=\{\sigma'(w_1^\top a):w_1\in \R^{d}, \normone{w_1}\le 1\}$ and $\calG_6=\{w_1^\top u:w_1\in \R^{d}, \normone{w_1}\le 1\}$. \citet[Lemma 4]{rakhlin2015online} gives $\src_T(\calG_4)\lesssim B\log^{3/2}(T^2)\pbra{\src_T(\calG_5)+\src_T(\calG_6)}.$ Note that $\calG_5$ is a generalized linear hypothesis and $\calG_6$ is linear, we have $\src_T(\calG_5)\lesssim B\log^{3/2}(T^2)\sqrt{T\log (d)}$ and $\src_T(\calG_6)\lesssim B\sqrt{T\log (d)}$. 

In summary, we get $\src_T(\calG_1)=\bigO\pbra{\poly(B)\polylog(d, T)\sqrt{T}}.$ Since the input $u_t\sim \calN(0,I_{d\times d})$, we have $B\lesssim \log(dT)$ with probability $1/T.$ As a result, the distribution dependent Rademacher complexity of $\emperr_{t,3}^2$ in this case is bounded by $\bigO\pbra{\polylog(d, T)\sqrt{T}}$.

Similarly, we can bound the sequential Rademacher complexity of the Hessian term $\emperr_{t,4}^2$ by $\bigO\pbra{\polylog(d, T)\sqrt{T}}$ by applying composition lemma with Lipschitz function $\phi(x,y,z)=xyz$ with bounded $\abs{x},\abs{y},\abs{z}.$ By \citet[Lemma 4]{rakhlin2015online}, composing with the min operator only introduces $\poly(\log(T))$ terms in the sequential Rademacher complexity. As a result, the sequential Rademacher complexity of the loss function can be bounded by 
$$\src_T=\bigO\pbra{\polylog(d, T)\sqrt{T}}.$$

Applying Theorem~\ref{thm:bandit-main}, the local regret of Alg.~\ref{alg:ban} is bounded by $\tildeO\pbra{\epsilon^{-2}T^{3/4}}.$ 

When the neural network is input concave (see \cite{amos2017input}), the regularized reward $\tilde{\eta}(\theta,a)$ is $\Omega(1)$-strongly concave. As a result, for any $a\in \sosp_{\epsilon,1}$ we have $\tilde{\eta}(\thetas,a)\ge \tilde{\eta}(\thetas,a^\star)-\bigO(\epsilon^2).$ It follows that,
\begin{equation}
	\reg(T)=\tildeO\pbra{\epsilon^{-2}T^{3/4}+\epsilon^2T}.
\end{equation}
By letting $\epsilon=T^{-1/16}$ we get $\reg(T)=\tildeO\pbra{T^{7/8}}.$
\end{proof} 

\section{Missing Proofs in Section~\ref{sec:main-results-rl}}\label{app:rl}
First of all, we present our algorithm in Alg.~\ref{alg:rl}.
\begin{algorithm}[h]
	\setstretch{0.9}
	\caption{\textbf{Vi}rtual Ascent with \textbf{O}n\textbf{lin}e Model Learner (\AlgBan~for RL)}
	\label{alg:rl}
	\begin{algorithmic}[1]
		\State Let $\calH_0 = \emptyset$; choose $a_0\in \calA$ arbitrarily. 
		\For{$t=1,2,\cdots$}
		\State Run $\ol$ on $\calH_{t-1}$ with loss function $\ell$ (defined in Eq.~\eqref{equ:rl-loss}) and obtain $p_t = \calA(\calH_{t-1})$. 
		\State $\psi_t\gets \argmax_{\psi}\E_{\theta_t\sim p_t}\bbra{\eta(\theta_t,\psi)}$;
		\State Sample one trajectory $\tau_t$ from policy $\pi_{\psi_t}$, and one trajectory $\tau_{t}'$ from policy $\pi_{\psi_{t-1}}.$
		\State Update $\calH_t\gets \calH_{t-1}\cup\{(\tau,\tau')\}$
		\EndFor
	\end{algorithmic}
\end{algorithm}

	In the following we present the proof sketch for Theorem~\ref{thm:rl-main}. Compare to the bandit case, we only need to prove an analog of Lemma~\ref{lem:bandit-concentration}, which means that we need to upper-bound the error term $\err_t$ by the difference of dynamics, as discussed before. Formally speaking, let $\tau_t=(s_1,a_1,\cdots,s_H,a_H)$ be a trajectory sampled from policy $\pi_{\psi_{t}}$ under the ground-truth dynamics $\dy_{\thetas}.$ By telescope lemma (Lemma~\ref{lem:telescope}) we get
	\begin{align}\label{equ:rl-1x}
		\V(s_1)-V^{\psi}_{\thetas}(s_1)=\E_{\tau\sim \rho^{\psi}_{\thetas}}\bbra{\sum_{h=1}^{H}\pbra{\V(\dy_{\theta}(s_h,a_h))-\V(\dy_{\thetas}(s_h,a_h))}}.
	\end{align}
	Lipschitz assumption (Assumption~\ref{assumption:RL-lipschitz}) yields,
	\begin{align}\label{equ:rl-2x}
		\abs{\V(\dy_{\theta}(s_h,a_h))-\V(\dy_{\thetas}(s_h,a_h))}\le L_0\normtwo{\dy_{\theta}(s_h,a_h)-\dy_{\thetas}(s_h,a_h))}.
	\end{align}
	Combining Eq.~\eqref{equ:rl-1x} and Eq.~\eqref{equ:rl-2x} and apply Cauchy-Schwartz inequality gives an upper bound for $[\err_t]_1^2$ and $[\err_{t}]_2^2.$ As for the gradient term, we will take gradient w.r.t. $\psi$ to both sides of Eq.~\eqref{equ:rl-1x}. The gradient inside expectation can be dealt with easily. And the gradient w.r.t. the distribution $\rho^{\psi}_{\thetas}$ can be computed by policy gradient lemma (Lemma~\ref{lem:PG}). As a result we get
	\begin{align}
		&\gpsi\V(s_1)-\gpsi V^{\psi}_{\thetas}(s_1)\nonumber\\
		=\;&\E_{\tau\sim \rho^{\psi}_{\thetas}}\bbra{\pbra{\sum_{h=1}^{H}\gpsi\log\pi_{\psi}(a_h\mid s_h)}\pbra{\sum_{h=1}^{H}\pbra{\V(\dy_{\theta}(s_h,a_h))-\V(\dy_{\thetas}(s_h,a_h))}}}\nonumber\\
		&+\;\E_{\tau\sim \rho^{\psi}_{\thetas}}\bbra{\sum_{h=1}^{H}\pbra{\gpsi\V(\dy_{\theta}(s_h,a_h))-\gpsi\V(\dy_{\thetas}(s_h,a_h))}}.
	\end{align}
	The first term can be bounded by vector-form Cauchy-Schwartz and Assumption~\ref{assumption:RL-smoothness}, and the second term is bounded by Assumption~\ref{assumption:RL-lipschitz}. Similarly, this approach can be extended to second order term. As a result, we have the following lemma.
	
	\begin{lemma}\label{lem:rl-concentration} Under the setting of Theorem~\ref{thm:rl-main}, we have
		\begin{align}
			c_1\E_{\tau_{1:t},\tau_{1:t}',\theta_{1:t}}\bbra{\emperrrl_t^2}\ge \E_{\theta_{1:t}}\bbra{\err_t^2}.
		\end{align}
	\end{lemma}
	
	Proof of Lemma~\ref{lem:rl-concentration} is shown in Appendix~\ref{app:proof-rl-concentration}. Proof of Theorem~\ref{thm:rl-main} is exactly the same as that of Theorem~\ref{thm:bandit-main} except for replacing Lemma~\ref{lem:bandit-concentration} with Lemma~\ref{lem:rl-concentration}.

\subsection{Proof of Lemma~\ref{lem:rl-concentration}}\label{app:proof-rl-concentration}
\begin{proof}%
	The lemma is proven by combining standard telescoping lemma and policy gradient lemma. Specifically, let $\rho^\pi_\dy$ be the distribution of trajectories generated by policy $\pi$ and dynamics $\dy$. By telescoping lemma (Lemma~\ref{lem:telescope}) we have,
	\begin{align}\label{equ:telescope}
		\Vs{t}(s_1)-V^{\psi_t}_{\thetas}(s_1)=\E_{\tau\sim \rho^{\psi_t}_{\thetas}}\bbra{\sum_{h=1}^{H}\pbra{\Vs{t}(\dy_{\theta_t}(s_h,a_h))-\Vs{t}(\dy_{\thetas}(s_h,a_h))}}.
	\end{align}
	By the Lipschitz assumption (Assumption~\ref{assumption:RL-lipschitz}),
	\begin{align}
		\abs{\Vs{t}(\dy_{\theta_t}(s_h,a_h))-\Vs{t}(\dy_{\thetas}(s_h,a_h))}\le L_0\normtwo{\dy_{\theta_t}(s_h,a_h)-\dy_{\thetas}(s_h,a_h))}.
	\end{align}
	Consequently
	\begin{align}
		\err_{t,1}^2= \pbra{\Vs{t}(s_0)-V^{\psi_t}_{\thetas}(s_0)}^2\le HL_0^2\E_{\tau\sim \rho^{\psi_t}_{\thetas}}\bbra{\sum_{h=1}^{H}\normtwo{\dy_{\theta_t}(s_h,a_h)-\dy_{\thetas}(s_h,a_h))}^2}.
	\end{align}
	Similarly we get, 
	\begin{align}
		\err_{t,2}^2= \pbra{V^{\psi_{t-1}}_{\theta_t}(s_0)-V^{\psi_{t-1}}_{\thetas}(s_0)}^2\le HL_0^2\E_{\tau\sim \rho^{\psi_{t-1}}_{\thetas}}\bbra{\sum_{h=1}^{H}\normtwo{\dy_{\theta_{t}}(s_h,a_h)-\dy_{\thetas}(s_h,a_h))}^2}.
	\end{align}
	
	Now we turn to higher order terms. First of all, by H\"{o}lder inequality and Assumption~\ref{assumption:RL-smoothness}, we can prove the following:%
	\begin{itemize}
		\item $\normsp{\E_{\tau\sim \rho^\psi_{\thetas}}\bbra{\pbra{\sum_{h=1}^{H}\gpsi \log \pi_\psi(a_h\mid s_h)}\pbra{\sum_{h=1}^{H}\gpsi \log \pi_\psi(a_h\mid s_h)}^\top}}\le H^2\Pg,\forall \psi\in \Psi;$
		\item $\normsp{\E_{\tau\sim \rho^\psi_{\thetas}}\bbra{\pbra{\sum_{h=1}^{H}\gpsi \log \pi_\psi(a_h\mid s_h)}^{\otimes 4}}}\le H^4\Pt,\forall \psi\in \Psi;$
		\item $\normsp{\E_{\tau\sim \rho^\psi_{\thetas}}\bbra{\pbra{\sum_{h=1}^{H}\hpsi \log \pi_\psi(a_h\mid s_h)}\pbra{\sum_{h=1}^{H}\hpsi \log \pi_\psi(a\mid s)}^\top}}\le H^2\Ph,\forall \psi\in \Psi.$
	\end{itemize}
	Indeed, consider the first statement. Define $g_h=\gpsi \log \pi_\psi(a_h\mid s_h)$ for shorthand. Then we have
	\begin{align}
		&\normsp{\E_{\tau\sim \rho^\psi_{\thetas}}\bbra{\pbra{\sum_{h=1}^{H}g_h}\pbra{\sum_{h=1}^{H}g_h}^\top}}
		= \sup_{u\in S^{d-1}}u^\top\E_{\tau\sim \rho^\psi_{\thetas}}\bbra{\pbra{\sum_{h=1}^{H}g_h}\pbra{\sum_{h=1}^{H}g_h}^\top}u\\
		&\quad=\sup_{u\in S^{d-1}}\E_{\tau\sim \rho^\psi_{\thetas}}\bbra{\dotp{u}{\pbra{\sum_{h=1}^{H}g_h}}^2}
		\le \sup_{u\in S^{d-1}}\E_{\tau\sim \rho^\psi_{\thetas}}\bbra{H\sum_{h=1}^{H}\dotp{u}{g_h}^2}\\
		&\quad \le\E_{\tau\sim \rho^\psi_{\thetas}}\bbra{H\sum_{h=1}^{H}\sup_{u\in S^{d-1}}\dotp{u}{g_h}^2} =\E_{\tau\sim \rho^\psi_{\thetas}}\bbra{H\sum_{h=1}^{H}\normsp{gg^\top}}\le H^2\Pg.
	\end{align}
	Similarly we can get the second and third statement.
	
	For any fixed $\psi$ and $\theta$ we have
	\begin{align}\label{equ:telescope-2}
		\V(s_1)-V^{\psi}_{\thetas}(s_1)=\E_{\tau\sim \rho^{\psi}_{\thetas}}\bbra{\sum_{h=1}^{H}\pbra{\V(\dy_{\theta}(s_h,a_h))-\V(\dy_{\thetas}(s_h,a_h))}}.
	\end{align}
	
	Applying policy gradient lemma (namely, Lemma~\ref{lem:PG}) to RHS of Eq.~\eqref{equ:telescope-2} we get,
	\begin{align}\label{equ:rl-pg-1}
		&\gpsi\V(s_1)-\gpsi V^{\psi}_{\thetas}(s_1)\nonumber\\
		=\;&\E_{\tau\sim \rho^{\psi}_{\thetas}}\bbra{\pbra{\sum_{h=1}^{H}\gpsi\log\pi_{\psi}(a_h\mid s_h)}\pbra{\sum_{h=1}^{H}\pbra{\V(\dy_{\theta}(s_h,a_h))-\V(\dy_{\thetas}(s_h,a_h))}}}\nonumber\\
		&+\;\E_{\tau\sim \rho^{\psi}_{\thetas}}\bbra{\sum_{h=1}^{H}\pbra{\gpsi\V(\dy_{\theta}(s_h,a_h))-\gpsi\V(\dy_{\thetas}(s_h,a_h))}}.
	\end{align}
	Define the following shorthand:
	\begin{align}
		\G(s,a)&=\V(\dy_{\theta}(s,a))-\V(\dy_{\thetas}(s,a)),\\
		f&=\sum_{h=1}^{H}\gpsi\log\pi_\psi(a_h\mid s_h).
	\end{align}
	In the following we also omit the subscription in $\E_{\tau\sim \rho^{\psi}_{\thetas}}$ when the context is clear. It followed by Eq.~\eqref{equ:rl-pg-1} that
	\begin{align*}
		&\normtwo{\gpsi\V(s_1)-\gpsi V^{\psi}_{\thetas}(s_1)}^2\\
		\le\;&2\normtwo{\E\bbra{f\pbra{\sum_{h=1}^{H}\G(s_h,a_h)}}}^2+2\normtwo{\E\bbra{\sum_{h=1}^{H}\gpsi\G(s_h,a_h)}}^2\\
		\le\;&2\normsp{\E\bbra{ff^\top}}\E\bbra{\pbra{\sum_{h=1}^{H}\G(s_h,a_h)}^2}
		+2\normtwo{\E\bbra{\sum_{h=1}^{H}\gpsi\G(s_h,a_h)}}^2\tag{By Lemma~\ref{lem:holder-vector}}\\
		\le\;&2H\normsp{\E\bbra{ff^\top}}\E\bbra{\sum_{h=1}^{H}\G(s_h,a_h)^2}
		+2H\E\bbra{\sum_{h=1}^{H}\normtwo{\gpsi\G(s_h,a_h)}^2}.
	\end{align*}
	Now, plugin $\psi=\psi_{t-1},\theta=\theta_t$ and apply Assumption~\ref{assumption:RL-lipschitz} we get
	\begin{align*}
		&\err_{t,3}^2=\normtwo{\gpsi V^{\psi_{t-1}}_{\theta_t}(s_1)-\gpsi V^{\psi_{t-1}}_{\thetas}(s_1)}^2\\
		\le\;&(2HL_1^2+2H^3\Pg L_0^2)\E_{\tau\sim \rho^{\psi_{t-1}}_{\thetas}}\bbra{\sum_{h=1}^{H}\normtwo{\dy_{\theta_t}(s_h,a_h)-\dy_{\thetas}(s_h,a_h)}^2}.
	\end{align*}
	
	For any fixed $\psi, \theta$, define the following shorthand:
	\begin{align}		
		g&=\sum_{h=1}^{H}\pbra{\V(\dy_{\theta}(s_h,a_h))-\V(\dy_{\thetas}(s_h,a_h))}.
	\end{align}
	Apply policy gradient lemma again to RHS of Eq.~\eqref{equ:rl-pg-1} we get
	\begin{align}
		&\hpsi\V(s_1)-\hpsi V^{\psi}_{\thetas}(s_1)\nonumber\\
		=\;&\E\bbra{\pbra{\gpsi g} f^\top}\nonumber+\E\bbra{f\pbra{\gpsi g}^\top}+\E\bbra{\hpsi g}+\E\bbra{g\pbra{\sum_{h=1}^{H}\hpsi\log\pi_\psi(a_h\mid s_h)}}+\E\bbra{g\pbra{f f^\top}}.
	\end{align}
	As a result of Lemma~\ref{lem:holder-matrix} and Lemma~\ref{lem:holder-2tensor} that,
	\begin{align*}\label{equ:rl-pg-2}
		&\normsp{\hpsi\V(s_1)-\hpsi V^{\psi}_{\thetas}(s_1)}^2\\
		=\;&4\normsp{\E\bbra{\pbra{\gpsi g} f^\top}+\E\bbra{f\pbra{\gpsi g}^\top}}^2+4\normsp{\E\bbra{\hpsi g}}^2+4\normsp{\E\bbra{g\pbra{f f^\top}}}^2\\
		&+4\normsp{\E\bbra{g\pbra{\sum_{h=1}^{H}\hpsi\log\pi_\psi(a_h\mid s_h)}}}^2\\
		\le&\;8\sup_{u,v\in S^{d-1}}\E\bbra{\dotp{\gpsi g}{u}\dotp{f}{v}}^2+4\E\bbra{\normsp{\hpsi g}^2}+4\E\bbra{g^2}\normsp{\E\bbra{f^{\otimes 4}}}\\
		&+4\E\bbra{g^2}\normsp{\E\bbra{\pbra{\sum_{h=1}^{H}\hpsi\log\pi_\psi(a_h\mid s_h)}\pbra{\sum_{h=1}^{H}\hpsi\log\pi_\psi(a_h\mid s_h)}^\top}}\numberthis.
	\end{align*}
	Note that by H\"{o}lder's inequality,
	\begin{align*}
		\sup_{u,v\in S^{d-1}}\E\bbra{\dotp{\gpsi g}{u}\dotp{f}{v}}^2\le \sup_{u,v\in S^{d-1}}\E\bbra{\dotp{\gpsi g}{u}^2}\E\bbra{\dotp{f}{v}^2}\le \E\bbra{\normtwo{\gpsi g}^2}\normsp{\E\bbra{ff^\top}}.
	\end{align*}
	By Assumption~\ref{assumption:RL-lipschitz} we get,
	\begin{align}
		\E[g^2]&=\E\bbra{\pbra{\sum_{h=1}^{H}\pbra{\V(\dy_{\theta}(s_h,a_h))-\V(\dy_{\thetas}(s_h,a_h))}}^2}\\
		&\le H\E\bbra{\sum_{h=1}^{H}\pbra{\V(\dy_{\theta}(s_h,a_h))-\V(\dy_{\thetas}(s_h,a_h))}^2}\\
		&\le HL_0^2\E\bbra{\sum_{h=1}^{H}\normtwo{\dy_{\theta}(s_h,a_h)-\dy_{\thetas}(s_h,a_h)}^2}.
	\end{align}
	Similarly, we have
	\begin{align}
		\E[\normtwo{\gpsi g}^2]
		&\le HL_1^2\E\bbra{\sum_{h=1}^{H}\normtwo{\dy_{\theta}(s_h,a_h)-\dy_{\thetas}(s_h,a_h)}^2},\\
		\E[\normsp{\hpsi g}^2]
		&\le HL_2^2\E\bbra{\sum_{h=1}^{H}\normtwo{\dy_{\theta}(s_h,a_h)-\dy_{\thetas}(s_h,a_h)}^2}.
	\end{align}
	Combining with Eq.~\eqref{equ:rl-pg-2} we get,
	\begin{align*}
		&\err_{t,4}^2=\normsp{\hpsi\Vs{t}(s_1)-\gpsi V^{\psi_t}_{\thetas}(s_1)}^2\\
		\le \;&\pbra{8H^3L_1^2 \Pg +4HL_2^2+4L_0^2(H^3\Ph+H^5\Pt)}\E_{\tau\sim \rho^{\psi_{t-1}}_{\thetas}}\bbra{\sum_{h=1}^{H}\normtwo{\dy_{\theta_t}(s_h,a_h)-\dy_{\thetas}(s_h,a_h)}^2}.
	\end{align*}
	By noting that $\err_t^2= \sum_{i=1}^{4}\err_{i,t}^2,$ we get the desired upper bound.
\end{proof}

\section{Analysis of Example~\ref{example:1}}\label{app:rl-instance}

Recall that our RL instance is given as follows:
\begin{align}
	\dy(s,a)&=\NN_\theta(s+a),\\
	\pi_\psi(s)&=\calN(\psi s,\sigma^2 I).\label{equ:app-policy}
\end{align}
And the assumptions are listed below.
\begin{itemize}
	\item Lipschitzness of reward function: $\abs{r(s_1,a_1)-r(s_2,a_2)}\le L_r(\normtwo{s_1-s_2}+\normtwo{a_1-a_2}).$
	\item Bounded Parameter: we assume $\normop{\psi}\le \bigO(1).$
\end{itemize}

In the sequel we verify the assumptions of Theorem~\ref{thm:rl-main}.
\subsection{Verifying Assumption~\ref{assumption:RL-smoothness}.} 
\paragraph{Verifying item 1.}
Recall that $\psi\in \R^{d\times d}$. By algebraic manipulation, for all $s,a$ we get,
\begin{align}\label{equ:app:pg-1}
	\gpsi \log \pi_\psi(a\mid s)=\frac{1}{\sigma^2}\vec\pbra{(a-\psi s)\otimes s}
\end{align}
where $\vec(x)$ denotes the vectorization of tensor $x$. Define random variable $u=a-\psi s.$ By the definition of policy $\pi_\psi(s)$ we have $u\sim \calN(0,\sigma^2I).$ As a result, 
\begin{align}
	&\normspsm{\E_{a\sim \pi_\psi(\cdot\mid s)}[(\gpsi \log \pi_\psi(a\mid s))(\gpsi \log \pi_\psi(a\mid s))^\top]}\\
	=&\frac{1}{\sigma^4}\sup_{v\in S^{d\times d-1}}\E_{u\sim \calN(0,\sigma^2I)}\bbra{\dotp{v}{\vec(u\otimes s)}^2}.\label{equ:app-4-2-1-1}
\end{align}
Note that $\dotp{v}{\vec(u\otimes s)}=\sum_{1\le i,j\le d}[u]_i [s]_j[v]_{i,j}.$ Because $u$ is isotropic, $[u]_i$ are independent random variables where $[u]_i\sim \calN(0,\sigma^2).$ Therefore $\dotp{v}{\vec(u\otimes s)}\sim \calN\pbra{0,\sigma^2\sum_{i=1}^{d}\pbra{\sum_{j=1}^{d}s_j v_{i,j}}^2}.$ Combining with Eq.~\eqref{equ:app-4-2-1-1} we get,
\begin{align}
	&\E_{u\sim \calN(0,\sigma^2I)}\bbra{\dotp{v}{\vec(u\otimes s)}^2}=\sigma^2\sum_{i=1}^{d}\pbra{\sum_{j=1}^{d}s_j v_{i,j}}^2\\
	\le\; &\sigma^2\sum_{i=1}^{d}\pbra{\sum_{j=1}^{d}s_j^2}\pbra{\sum_{j=1}^{d}v_{i,j}^2}\le \normtwo{s}^2\normtwo{v}^2.
\end{align}
Consequently we have 
\begin{align}
	\normspsm{\E_{a\sim \pi_\psi(\cdot\mid s)}[(\gpsi \log \pi_\psi(a\mid s))(\gpsi \log \pi_\psi(a\mid s))^\top]}\le \frac{1}{\sigma^2}\defeq \Pg.
\end{align}

\paragraph{Verifying item 2.}
Similarly, using the equation where $\E_{x\sim \calN(0,\sigma^2)}[x^4]=3\sigma^4$ we have
\begin{align}
	&\normspsm{\E_{a\sim \pi_\psi(\cdot\mid s)}[(\gpsi \log \pi_\psi(a\mid s))^{\otimes 4}]}
	=\frac{1}{\sigma^8}\sup_{v\in S^{d\times d-1}}\E_{u\sim \calN(0,\sigma^2I)}\bbra{\dotp{v}{\vec(u\otimes s)}^4}.\label{equ:app-4-2-1-2}\\
	\le\;&\frac{3}{\sigma^8}\pbra{\sigma^2\sum_{i=1}^{d}\pbra{\sum_{j=1}^{d}s_j v_{i,j}}^2}^2\le \frac{3}{\sigma^4}\normtwo{s}^4\normtwo{v}^4\le \frac{3}{\sigma^4}\defeq \Pt.
\end{align}

\paragraph{Verifying item 3.}
Since $\hpsi \log\pi_\psi(a\mid s)$ is PSD, we have
\begin{align}
	&\normspsm{\E_{a\sim \pi_\psi(\cdot\mid s)}[(\hpsi \log \pi_\psi(a\mid s))(\hpsi \log \pi_\psi(a\mid s))^\top]}\\
	=\;&\sup_v \E\bbra{v^\top (\hpsi \log \pi_\psi(a\mid s))(\hpsi \log \pi_\psi(a\mid s))^\top v}\\
	=\;&\sup_v \E\bbra{\normtwo{(\hpsi \log \pi_\psi(a\mid s))^\top v}^2}
	=\sup_v \E\bbra{\pbra{v^\top (\hpsi \log \pi_\psi(a\mid s)) v}^4}.
\end{align}
By algebraic manipulation, for all $s,a\in \R^{d}$ and $v\in \R^{d\times d}$ we have 
\begin{align}
	v^\top (\hpsi \log \pi_\psi(a\mid s))v=-\sum_{i=1}^{d}\pbra{\sum_{j=1}^{d}v_{i,j}s_j}^2.
\end{align}
Consequently,
\begin{align}
	\pbra{v^\top (\hpsi \log \pi_\psi(a\mid s)) v}^4\le \normtwo{s}^4\normtwo{v}^4\le 1\defeq \Ph.
\end{align}

\subsection{Verifying Assumption~\ref{assumption:lipschitz}.} 
\paragraph{Verifying item 1.}
We verify Assumption~\ref{assumption:lipschitz} by applying policy gradient lemma. Recall that 
\begin{align}
	\eta(\theta,\psi)=\E_{\tau\sim \rho^\psi_\theta}\bbra{\sum_{h=1}^{H}r(s_h,a_h)}.
\end{align}
By policy gradient lemma (Lemma~\ref{lem:PG}) we have
\begin{align}\label{equ:app:PG-1}
	\gpsi \eta(\theta,\psi)=\E_{\tau\sim \rho^\psi_\theta}\bbra{\pbra{\sum_{h=1}^{H}\gpsi \log\pi_\psi(a_h\mid s_h)}\pbra{\sum_{h=1}^{H}r(s_h,a_h)}}.
\end{align}

By Eq.~\eqref{equ:app:pg-1}, condition on $s_h$ we get
\begin{align}
	\gpsi \log \pi_\psi(a_h\mid s_h)=\frac{1}{\sigma^2}\vec(u\otimes s_h)
\end{align}
where $u=a_h-\psi s_h\sim \calN(0,\sigma^2 I).$ Define the shorthand $g=\sum_{h=1}^{H}r(s_h,a_h).$ Note that by H\"{o}lder inequality,
\begin{align}
	&\normtwo{\E\bbra{\gpsi \log \pi_\psi(a_h\mid s_h)g}}^2=\sup_{v\in \R^{d\times d},\normtwo{v}=1}\E\bbra{\dotp{\gpsi \log \pi_\psi(a_h\mid s_h)}{v}g}^2\\
	\le &\sup_{v\in \R^{d\times d},\normtwo{v}=1}\E\bbra{\dotp{\gpsi \log \pi_\psi(a_h\mid s_h)}{v}^2}\E\bbra{g^2}.
\end{align}
Since $v\in \R^{d\times d}$, if we view $v$ as a $d\times d$ matrix then
$\dotp{\gpsi \log \pi_\psi(a_h\mid s_h)}{v}=\frac{1}{\sigma^2}\dotp{vs_h}{u}.$ Because $u$ is an isotropic Gaussian random vector, $\dotp{vs_h}{u}\sim \calN(0, \sigma^2\normtwo{vs_h}^2).$ Consequently, 
\begin{align}
	\E\bbra{\dotp{\gpsi \log \pi_\psi(a_h\mid s_h)}{v}^2}=\frac{1}{\sigma^2}\normtwo{vs_h}^2\le \frac{1}{\sigma^2}\normF{v}^2\normtwo{s_h}^2\le \frac{1}{\sigma^2}.
\end{align}
It follows that $\normtwo{\E\bbra{\gpsi \log \pi_\psi(a_h\mid s_h)g}}^2\le \frac{H^2}{\sigma^2}.$
By triangular inequality and Eq.~\eqref{equ:app:PG-1} we get
\begin{align}
	\normtwo{\gpsi \eta(\theta,\psi)}\le H^2/\sigma.
\end{align}

\paragraph{Verifying item 2.}
Define the shorthand $f=\sum_{h=1}^{H}\gpsi \log\pi_\psi(a_h\mid s_h)$. Use policy gradient lemma on Eq.~\eqref{equ:app:PG-1} again we get, for any $v,w \in \R^{d\times d}$,
\begin{align}\label{equ:app:PG-2}
	v^\top \hpsi \eta(\theta,\psi) w=\E_{\tau\sim \rho^\psi_\theta}\bbra{\dotp{f}{v}\dotp{f}{w}g+\pbra{\sum_{h=1}^{H}v^\top \hpsi \log\pi_\psi(a_h\mid s_h)w}g}.
\end{align} 
For the first term inside the expectation, we bound it by using H\"{o}lder inequality twice. Specifically, for any $h,h'\in [H]$ we have
\begin{align}
	&\E\bbra{\dotp{\gpsi \log\pi_\psi(a_h\mid s_h)}{v}\dotp{\gpsi \log\pi_\psi(a_{h'}\mid s_{h'})}{w}g}\\
	\le&\E\bbra{\dotp{\gpsi \log\pi_\psi(a_h\mid s_h)}{v}^{4}}^{1/4}\E\bbra{\dotp{\gpsi \log\pi_\psi(a_{h'}\mid s_{h'})}{w}^4}^{1/4}\E\bbra{g^2}^{1/2}.
\end{align}
Similarly, $\dotp{\gpsi \log\pi_\psi(a_h\mid s_h)}{v}\sim \frac{1}{\sigma^2}\calN(0, \sigma^2\normtwo{vs_h}^2)$ and $\dotp{\gpsi \log\pi_\psi(a_{h'}\mid s_{h'})}{w}\sim \frac{1}{\sigma^2}\calN(0, \sigma^2\normtwo{ws_{h'}}^2).$ As a result,
\begin{align}
	&\E\bbra{\dotp{\gpsi \log\pi_\psi(a_h\mid s_h)}{v}\dotp{\gpsi \log\pi_\psi(a_{h'}\mid s_{h'})}{w}g}\\
	\le&3\frac{1}{\sigma^2}\normtwo{v}\normtwo{s_h}\normtwo{w}\normtwo{s_{h'}}H\le \frac{3H}{\sigma^2}.
\end{align}
Therefore the first term of Eq.~\eqref{equ:app:PG-2} can be bounded by $\frac{3H^3}{\sigma^2}$. Now we bound the second term of Eq.~\eqref{equ:app:PG-2}.By algebraic manipulation we have
\begin{align}\label{equ:app:pg-2}
	v^\top \hpsi \log\pi_\psi(a_h\mid s_h)w=-\frac{1}{\sigma^2}\dotp{ws_h}{vs_h}.
\end{align}
Consequently, 
\begin{align}
	\E\bbra{\pbra{\sum_{h=1}^{H}v^\top \hpsi \log\pi_\psi(a_h\mid s_h)w}g}\le \frac{H^2}{\sigma^2}\normtwo{w}\normtwo{v}\le \frac{H^2}{\sigma^2}.
\end{align}
In summary, we have $\normop{\hpsi \eta(\theta,\psi)}\le \frac{4H^3}{\sigma^2}.$

\paragraph{Verifying item 3.}
Now we turn to the last item in Assumption~\ref{assumption:lipschitz}. First of all, following Eq.~\eqref{equ:app:pg-2}, we have $\gpsi^3 \log\pi_\psi(a_h\mid s_h)=0.$ As a result, applying policy gradient lemma to Eq.~\eqref{equ:app:PG-2} again we get
\begin{align}\label{equ:app:PG-3}
	\dotp{\gpsi^3 \eta(\theta,\psi)}{v\otimes w\otimes x}=&\E_{\tau\sim \rho^\psi_\theta}\bbra{\dotp{f}{v}\dotp{f}{w}\dotp{f}{w}g}\\
	&+\E_{\tau\sim \rho^\psi_\theta}\bbra{\dotp{f}{x}\pbra{\sum_{h=1}^{H}v^\top \hpsi \log\pi_\psi(a_h\mid s_h)w}g}\\
	&+\E_{\tau\sim \rho^\psi_\theta}\bbra{\dotp{f}{v}\pbra{\sum_{h=1}^{H}x^\top \hpsi \log\pi_\psi(a_h\mid s_h)w}g}\\
	&+\E_{\tau\sim \rho^\psi_\theta}\bbra{\dotp{f}{w}\pbra{\sum_{h=1}^{H}v^\top \hpsi \log\pi_\psi(a_h\mid s_h)x}g}.
\end{align}
Following the same argument, by H\"{o}lder inequality, for any $h_1,h_2,h_3\in [H]$ we have
\begin{align*}
	&\E\bbra{\dotp{\gpsi \log\pi_\psi(a_h\mid s_h)}{v}\dotp{\gpsi \log\pi_\psi(a_{h'}\mid s_{h'})}{w}\dotp{\gpsi \log\pi_\psi(a_{h'}\mid s_{h'})}{x}g}\\
	\le&\E\bbra{\dotp{\gpsi \log\pi_\psi(a_h\mid s_h)}{v}^{6}}^{1/6}\E\bbra{\dotp{\gpsi \log\pi_\psi(a_{h'}\mid s_{h'})}{w}^6}^{1/6}\E\bbra{\dotp{\gpsi \log\pi_\psi(a_{h'}\mid s_{h'})}{x}^6}^{1/6}H\\
	\le&\frac{\sqrt{15}H}{\sigma^3}.
\end{align*}
On the other hand, 
\begin{align*}
	&\E\bbra{\dotp{f}{x}\pbra{\sum_{h=1}^{H}v^\top \hpsi \log\pi_\psi(a_h\mid s_h)w}g}\\
	\le &\E\bbra{\dotp{f}{x}^2}^{1/2}\E\bbra{\pbra{\pbra{\sum_{h=1}^{H}v^\top \hpsi \log\pi_\psi(a_h\mid s_h)w}g}^2}^{1/2}\\
	\le &\frac{H^3}{\sigma^3}.
\end{align*}
By symmetricity, Eq.~\eqref{equ:app:PG-3} can be upper bounded by 
\begin{align}
	\dotp{\gpsi^3 \eta(\theta,\psi)}{v\otimes w\otimes x}\le \frac{7H^4}{\sigma^3}.
\end{align}

\subsection{Verifying Assumption~\ref{assumption:RL-lipschitz}.} 
\paragraph{Verifying item 1.} 
We verify Assumption~\ref{assumption:RL-lipschitz} by coupling argument. First of all, consider the Lipschitzness of value function. By Bellman equation we have
\begin{align}
	V_\theta^\psi(s)=&\E_{a\sim \pi_\psi(s)}\bbra{r(s,a)+V_\theta^\psi(\dy(s,a))}\\
	=&\E_{u\sim \calN(0,\sigma^2 I)}\bbra{r(s,\psi s + u)+V_\theta^\psi(\NN_\theta(s+\psi s + u))}.\label{equ:app:lip-1}
\end{align}
Define $B=1+\normop{\psi}$ for shorthand. For two states $s_1,s_2\in \calS$, by the Lipschitz assumption on reward function we have
\begin{align}
	\abs{r(s_1,\psi s_1 + u)-r(s_2, \psi s_2 + u)}\le L_rB\normtwo{s_1-s_2}.
\end{align}
Then consider the second term in Eq.~\eqref{equ:app:lip-1}. Since we have $\abs{V^\pi_\theta}\le H$ and $$\TV{\calN(s_1+\psi s_1,\sigma^2 I)}{\calN(s_2+\psi s_2,\sigma^2 I)}\le \frac{1}{2\sigma}\normtwo{s_1+\psi s_1 - s_2 - \psi s_2}\le \frac{B\normtwo{s_1-s_2}}{2\sigma},$$ it follows that 
\begin{align*}
	\abs{\E_{u\sim \calN(0,\sigma^2 I)}\bbra{V_\theta^\psi(\NN_\theta(s_1+\psi s_1 + u))}-\E_{u\sim \calN(0,\sigma^2 I)}\bbra{V_\theta^\psi(\NN_\theta(s_2+\psi s_2 + u))}}\le \frac{HB}{2\sigma}\normtwo{s_1-s_2}.
\end{align*}
As a result, item 1 of Assumption~\ref{assumption:RL-lipschitz} holds as follows
\begin{align}
	\abs{V_\theta^\psi(s_1)-V_\theta^\psi(s_2)}\le \pbra{\frac{HB}{2\sigma}+L_rB}\normtwo{s_1-s_2}.
\end{align}

\paragraph{Verifying item 2.} 
Now we turn to verifying the Lipschitzness of gradient term. Recall that by policy gradient lemma we have for every $v\in \R^{d\times d}$,
\begin{align}\label{equ:app:lip-2}
	&\dotp{\gpsi V^\psi_\theta(s)}{v}\\
	=&\;\E_{a\sim \pi_\psi (s)}\bbra{\dotp{\gpsi V^\psi_\theta(\NN_\theta(s+a))}{v}}\\
	&+\E_{a\sim \pi_\psi (s)}\bbra{\dotp{\gpsi \log\pi_\psi(a\mid s)}{v}\pbra{r(s,a)+V^\psi_\theta(\NN_\theta(s+a))}}\\
	=&\;\E_{u\sim \calN(0,\sigma^2 I)}\bbra{\dotp{\gpsi V^\psi_\theta(\NN_\theta(s+\psi s + u))}{v}}\\
	&+\E_{u\sim \calN(0,\sigma^2 I)}\bbra{\dotp{\gpsi \log\pi_\psi(\psi s + u\mid s)}{v}\pbra{r(s,\psi s + u)+V^\psi_\theta(\NN_\theta(s+\psi s + u))}}.\label{equ:app:lip-2-2}
\end{align}
Because for any two vectors $g_1,g_2\in \R^{d\times d}$
$\normtwo{g_1-g_2}=\sup_{v\in S^{d\times d-1}}\dotp{g_1-g_2}{v}$, Lipschitzness of Eq.~\eqref{equ:app:lip-2} for every $v\in \R^{d\times d},\normtwo{v}=1$ implies Lipschitzness of $\gpsi \V(s)$.

By the boundness of $\normtwo{\gpsi V^\psi_\theta(s)}$ (specifically, item 1 of Assumption~\ref{assumption:lipschitz}), we have 
\begin{align*}
	&\abs{\E_{u\sim \calN(0,\sigma^2 I)}\bbra{\dotp{\gpsi V^\psi_\theta(\NN_\theta(s_1+\psi s_1 + u))}{v}}-\E_{u\sim \calN(0,\sigma^2 I)}\bbra{\dotp{\gpsi V^\psi_\theta(\NN_\theta(s_2+\psi s_2 + u))}{v}}}\\
	\le\;&\frac{H^2}{\sigma}\TV{\calN(s_1+\psi s_1,\sigma^2I)}{\calN(s_2+\psi s_2,\sigma^2 I)}\le \frac{H^2}{\sigma}\frac{B}{2\sigma}\normtwo{s_1-s_2}.
\end{align*}
For the reward term in Eq.~\eqref{equ:app:lip-2-2}, recalling $v\in \R^{d\times d}$ we have
\begin{align*}
	\E_{a\sim \pi_\psi(s)}\bbra{\dotp{\gpsi \log\pi_\psi(\psi s + u\mid s)}{v}r(s,\psi s + u)}=\E_{u\sim\calN(0,\sigma^2 I)}\bbra{\dotp{vs}{u}r(s,\psi s + u)}.
\end{align*}
Note that
\begin{align}\label{equ:app:lip-2-1-1}
	&\E_{u\sim\calN(0,\sigma^2 I)}\bbra{\dotp{vs_1}{u}r(s_1,\psi s_1 + u)}-\E_{u\sim\calN(0,\sigma^2 I)}\bbra{\dotp{vs_2}{u}r(s_2,\psi s_2 + u)}\\
	=\;&\E_{u\sim\calN(0,\sigma^2 I)}\bbra{\dotp{vs_1}{u}\pbra{r(s_1,\psi s_1 + u)-r(s_2,\psi s_2 + u)}}\label{equ:app:lip-2-1-2}\\
	&+\E_{u\sim\calN(0,\sigma^2 I)}\bbra{\pbra{\dotp{vs_1}{u}-\dotp{vs_2}{u}}r(s_2,\psi s_2 + u)}.\label{equ:app:lip-2-1-3}
\end{align}
Note that $u$ is isotropic. Applying Lemma~\ref{lem:lip-1-bounded} we have
\begin{align}
	&\E_{u\sim\calN(0,\sigma^2 I)}\bbra{\pbra{\dotp{vs_1}{u}-\dotp{vs_2}{u}}r(s_2,\psi s_2 + u)}
	\le \sigma\normtwo{vs_1-vs_2}\le \sigma\normtwo{s_1-s_2}.
\end{align}
We can also bound the term in Eq.~\eqref{equ:app:lip-2-1-2} by
\begin{align}
	&\E_{u\sim\calN(0,\sigma^2 I)}\bbra{\dotp{vs_1}{u}\pbra{r(s_1,\psi s_1 + u)-r(s_2,\psi s_2 + u)}}\\
	\le\;&\E_{u\sim\calN(0,\sigma^2 I)}\bbra{\dotp{vs_1}{u}^2}^{1/2}\E_{u\sim\calN(0,\sigma^2 I)}\bbra{\pbra{r(s_1,\psi s_1 + u)-r(s_2,\psi s_2 + u)}^2}^{1/2}\\
	\le\;&\sigma\E_{u\sim\calN(0,\sigma^2 I)}\bbra{L_r^2B^2\normtwo{s_1-s_2}^2}^{1/2}\le \sigma L_rB\normtwo{s_1-s_2}.
\end{align}
Now we deal with the last term in Eq.~\eqref{equ:app:lip-2-2}. Let $f(s,u)=V^\psi_\theta(\NN_\theta(s+\psi s + u))$ for shorthand. Similarly we have
\begin{align}
	\E_{u\sim \calN(0,\sigma^2 I)}\bbra{\dotp{\gpsi \log\pi_\psi(\psi s + u\mid s)}{v}V^\psi_\theta(\NN_\theta(s+\psi s + u))}
	=\;&\E_{u\sim\calN(0,\sigma^2 I)}\bbra{\dotp{vs}{u}f(s,u)}.
\end{align}
By the same telescope sum we get,
\begin{align}\label{equ:app:lip-2-2-1}
	&\E_{u\sim\calN(0,\sigma^2 I)}\bbra{\dotp{vs_1}{u}f(s_1,u)}-\E_{u\sim\calN(0,\sigma^2 I)}\bbra{\dotp{vs_2}{u}f(s_2,u)}\\
	=\;&\E_{u\sim\calN(0,\sigma^2 I)}\bbra{\dotp{vs_1}{u}\pbra{f(s_2,u)-f(s_2,u)}}\label{equ:app:lip-2-2-2}\\
	&+\E_{u\sim\calN(0,\sigma^2 I)}\bbra{\pbra{\dotp{vs_1}{u}-\dotp{vs_2}{u}}f(s_2,u)}.\label{equ:app:lip-2-2-3}
\end{align}
Applying Lemma~\ref{lem:lip-1-bounded} we have
\begin{align}
&\E_{u\sim\calN(0,\sigma^2 I)}\bbra{\pbra{\dotp{vs_1}{u}-\dotp{vs_2}{u}}f(s_2,u)}
\le \sigma H\normtwo{vs_1-vs_2}\le \sigma H\normtwo{s_1-s_2}.
\end{align}
Applying Lemma~\ref{lem:lip-1-coupling} we have 
\begin{align}
	&\E_{u\sim\calN(0,\sigma^2 I)}\bbra{\dotp{vs_1}{u}\pbra{f(s_2,u)-f(s_2,u)}}\le 6BH\normtwo{s_1-s_2}\pbra{1+\frac{1}{\sigma}}.
\end{align}
In summary, we have
\begin{align}
	\normtwo{\gpsi V^\psi_\theta(s_1)-\gpsi V^\psi_\theta(s_2)}\le \poly(H,B,\sigma,1/\sigma,L_r)\normtwo{s_1-s_2}.
\end{align}

\paragraph{Verifying item 3.} Lastly, we verify the Lipschitzness of Hessian term. Applying policy gradient lemma to Eq.~\eqref{equ:app:lip-2} again we have
\begin{align}
	&w^\top \hpsi V^\psi_\theta(s)v\\
	=&\;\E_{a\sim \pi_\psi (s)}\bbra{w^\top \hpsi V^\psi_\theta(\NN_\theta(s+a))v}\\
	&+\E_{a\sim \pi_\psi (s)}\bbra{\dotp{\gpsi V^\psi_\theta(\NN_\theta(s+a))}{v}\dotp{\gpsi \log\pi_\psi(a\mid s)}{w}} \label{equ:lip-3-1}\\
	&+\E_{a\sim \pi_\psi (s)}\bbra{\dotp{\gpsi \log\pi_\psi(a\mid s)}{v}\dotp{\gpsi V^\psi_\theta(\NN_\theta(s+a))}{w}}\label{equ:lip-3-2}\\
	&+\E_{a\sim \pi_\psi (s)}\bbra{\dotp{\gpsi \log\pi_\psi(a\mid s)}{v}\dotp{\gpsi \log\pi_\psi(a\mid s)}{w}\pbra{r(s,a)+V^\psi_\theta(\NN_\theta(s+a))}}.\label{equ:lip-3-3}
\end{align}
Recall that $a\sim \psi s + \calN(0,\sigma^2 I).$ In the sequel, we bound the Lipschitzness of above four terms separately.

By the upper bound of $\normop{\hpsi V^\psi_\theta(\NN_\theta(s+a))}$ (specifically, item 2 of Assumption~\ref{assumption:lipschitz}) we have
\begin{align*}
	&\abs{\E_{u\sim \calN(0,\sigma^2 I)}\bbra{w^\top \hpsi V^\psi_\theta(\NN_\theta(s_1+\psi s_1 + u))v}-\E_{u\sim \calN(0,\sigma^2 I)}\bbra{w^\top \hpsi V^\psi_\theta(\NN_\theta(s_1+\psi s_1 + u))v}}\\
	\le\;&\frac{4H^3}{\sigma^2}\TV{\calN(s_1+\psi s_1,\sigma^2I)}{\calN(s_2+\psi s_2,\sigma^2 I)}\le \frac{3H^3}{\sigma^2}\frac{B}{2\sigma}\normtwo{s_1-s_2}.
\end{align*}

For the terms in Eq.~\eqref{equ:lip-3-1}, let $f(s,u)=\dotp{\gpsi V^\psi_\theta(\NN_\theta(s+\psi s + a))}{v}.$ Repeat the same argument when verifying item 2 again, we have
\begin{align}
	&\E_{u\sim\calN(0,\sigma^2 I)}\bbra{\dotp{ws_1}{u}\pbra{f(s_2,u)-f(s_2,u)}}\le 6B\frac{H^2}{\sigma}\normtwo{s_1-s_2}\pbra{1+\frac{1}{\sigma}}.
\end{align}
Similarly, term in Eq.~\eqref{equ:lip-3-2} also has the same Lipschitz constant.

Finally, we bound the term in Eq.~\eqref{equ:lip-3-3}. For the reward term in Eq.~\eqref{equ:lip-3-3}, recalling $v,w\in \R^{d\times d}$ we have
\begin{align*}
	\E_{a\sim \pi_\psi (s)}\bbra{\dotp{\gpsi \log\pi_\psi(a\mid s)}{v}\dotp{\gpsi \log\pi_\psi(a\mid s)}{w}r(s,a)}=\E_{u\sim\calN(0,\sigma^2 I)}\bbra{\dotp{vs}{u}\dotp{ws}{u}r(s,\psi s + u)}.
\end{align*}
Note that
\begin{align}\label{equ:app:lip-3-1-1}
	&\E_{u\sim\calN(0,\sigma^2 I)}\bbra{\dotp{vs_1}{u}\dotp{ws_1}{u}r(s_1,\psi s_1 + u)}-\E_{u\sim\calN(0,\sigma^2 I)}\bbra{\dotp{vs_2}{u}\dotp{ws_2}{u}r(s_2,\psi s_2 + u)}\\
	=\;&\E_{u\sim\calN(0,\sigma^2 I)}\bbra{\dotp{vs_1}{u}\dotp{ws_1}{u}\pbra{r(s_1,\psi s_1 + u)-r(s_2,\psi s_2 + u)}}\label{equ:app:lip-3-1-2}\\
	&+\E_{u\sim\calN(0,\sigma^2 I)}\bbra{\pbra{\dotp{vs_1}{u}\dotp{ws_1}{u}-\dotp{vs_2}{u}\dotp{ws_2}{u}}r(s_2,\psi s_2 + u)}.\label{equ:app:lip-3-1-3}
\end{align}
Note that $u$ is isotropic. Applying Lemma~\ref{lem:lip-2-bounded} we have
\begin{align*}
	&\E_{u\sim\calN(0,\sigma^2 I)}\bbra{\pbra{\dotp{vs_1}{u}\dotp{ws_1}{u}-\dotp{vs_2}{u}\dotp{ws_2}{u}}r(s_2,\psi s_2 + u)}\\
	\le&\; \sqrt{3}\sigma^2\pbra{\normtwo{vs_1-vs_2}+\normtwo{ws_1-ws_2}}\le 2\sqrt{3}\sigma^2\normtwo{s_1-s_2},
\end{align*}
We can also bound the term in Eq.~\eqref{equ:app:lip-3-1-2} by
\begin{align*}
	&\E_{u\sim\calN(0,\sigma^2 I)}\bbra{\dotp{vs_1}{u}\dotp{ws_1}{u}\pbra{r(s_1,\psi s_1 + u)-r(s_2,\psi s_2 + u)}}\\
	\le\;&\E_{u\sim\calN(0,\sigma^2 I)}\bbra{\dotp{vs_1}{u}^4}^{1/4}\E_{u\sim\calN(0,\sigma^2 I)}\bbra{\dotp{ws_1}{u}^4}^{1/4}\E_{u\sim\calN(0,\sigma^2 I)}\bbra{\pbra{r(s_1,\psi s_1 + u)-r(s_2,\psi s_2 + u)}^2}^{1/2}\\
	\le\;&\sqrt{3}\sigma^2\E_{u\sim\calN(0,\sigma^2 I)}\bbra{L_r^2B^2\normtwo{s_1-s_2}^2}^{1/2}\le \sqrt{3}\sigma^2 L_rB\normtwo{s_1-s_2}.
\end{align*}

Now we deal with the last term in Eq.~\eqref{equ:lip-3-3}. Let $f(s,u)=V^\psi_\theta(\NN_\theta(s+\psi s + u))$ for shorthand. Similarly we have
\begin{align}
	&\E_{a\sim \pi_\psi (s)}\bbra{\dotp{\gpsi \log\pi_\psi(a\mid s)}{v}\dotp{\gpsi \log\pi_\psi(a\mid s)}{w}V^\psi_\theta(\NN_\theta(s+a))}\\
	=\;&\E_{u\sim\calN(0,\sigma^2 I)}\bbra{\dotp{vs}{u}\dotp{ws}{u}f(s,u)}.
\end{align}
By the same telescope sum we get,
\begin{align}\label{equ:app:lip-3-2-1}
	&\E_{u\sim\calN(0,\sigma^2 I)}\bbra{\dotp{vs_1}{u}\dotp{ws_1}{u}f(s_1,u)}-\E_{u\sim\calN(0,\sigma^2 I)}\bbra{\dotp{vs_2}{u}\dotp{ws_2}{u}f(s_2,u)}\\
	=\;&\E_{u\sim\calN(0,\sigma^2 I)}\bbra{\dotp{vs_1}{u}\dotp{ws_1}{u}\pbra{f(s_1,u)-f(s_2,u)}}\label{equ:app:lip-3-2-2}\\
	&+\E_{u\sim\calN(0,\sigma^2 I)}\bbra{\pbra{\dotp{vs_1}{u}\dotp{ws_1}{u}-\dotp{vs_2}{u}\dotp{ws_2}{u}}f(s_2,u)}.\label{equ:app:lip-3-2-3}
\end{align}
Applying Lemma~\ref{lem:lip-2-bounded} we have
\begin{align}
	&\E_{u\sim\calN(0,\sigma^2 I)}\bbra{\pbra{\dotp{vs_1}{u}\dotp{ws_1}{u}-\dotp{vs_2}{u}\dotp{ws_2}{u}}f(s_2,u)}\\
	\le\; &\sqrt{3}\sigma^2H\pbra{\normtwo{vs_1-vs_2}+\normtwo{ws_1-ws_2}}\le 2\sqrt{3}\sigma^2H\normtwo{s_1-s_2}.
\end{align}
Applying Lemma~\ref{lem:lip-2-coupling} we have 
\begin{align}
	&\E_{u\sim\calN(0,\sigma^2 I)}\bbra{\dotp{vs_1}{u}\dotp{ws_1}{u}\pbra{f(s_2,u)-f(s_2,u)}}\le \poly(H,\sigma,1/\sigma)B\normtwo{s_1-s_2}.
\end{align}
In summary, we have
\begin{align}
	\normop{\hpsi V^\psi_\theta(s_1)-\hpsi V^\psi_\theta(s_2)}\le \poly(H,B,\sigma,1/\sigma,L_r)\normtwo{s_1-s_2}.
\end{align}

\section{Helper Lemmas}\label{sec:app:helperlemmas}
In this section, we list helper lemmas that are used in previous sections.
\subsection{Helper Lemmas on Probability Analysis}

The following lemma provides a concentration inequality on the norm of linear transformation of a Gaussian vector, which is used to prove Lemma~\ref{lem:square-bilinear-tail-bound}.
\begin{lemma}[Theorem 1 of \citet{hsu2012tail}]\label{lem:hason-wright}
	For $v\sim \calN(0,I)$ be a $n$ dimensional Gaussian vector, and $A\in \R^{n\times n}$. Let $\Sigma=A^\top A$, then
	\begin{align}\label{equ:hason-wright}
		\forall t>0,\Pr\bbra{\normtwo{A v}^2\ge \Tr(\Sigma)+ 2\sqrt{\Tr(\Sigma^2)t}+2\normop{\Sigma}t}\le \exp(-t).
	\end{align}
\end{lemma}
\begin{corollary}\label{cor:hason-wright}
	Under the same settings of Lemma~\ref{lem:hason-wright}, 
	\begin{align}
		\forall t>1,\Pr\bbra{\normtwo{A v}^2\ge \normF{A}^2+ 4\normF{A}^2t}\le \exp(-t).
	\end{align}
\end{corollary}
\begin{proof}
	Let $\lambda_i$ be the $i$-th eigenvalue of $\Sigma$. By the definition of $\Sigma$ we have $\lambda_i\ge 0$. Then we have \begin{align*}
		\Tr(\Sigma)&=\sum_{i=1}^{n}\lambda_i=\normF{A}^2,\\
		\Tr(\Sigma^2)&=\sum_{i=1}^{n}\lambda_i^2\le\pbra{\sum_{i=1}^{n}\lambda_i}^2=\normF{A}^4,\\
		\normop{\Sigma}&=Ax_{i\in [n]}\lambda_i\le \sum_{i=1}^{n}\lambda_i=\normF{A}^2.
	\end{align*}
	Plug in Eq.~\eqref{equ:hason-wright}, we get the desired equation.
\end{proof}

Next lemma proves a concentration inequality on which Lemma~\ref{lem:bandit-concentration} relies.
\begin{lemma}\label{lem:square-bilinear-tail-bound}
	Given a symmetric matrix $H$, let $u,v\sim \calN(0,I)$ be two independent random vectors, we have
	\begin{align}
		\forall t\ge 1, \Pr\bbra{(u^\top H v)^2\ge t\normF{H}^2}\le 3\exp(-\sqrt{t}/4).
	\end{align}
\end{lemma}
\begin{proof}
	Condition on $v$, $u^\top H v$ is a Gaussian random variable with mean zero and variance $\normtwo{H v}^2.$ Therefore we have,
	\begin{align}\label{equ:tail-lemma-1}
		\forall v, \Pr\bbra{\pbra{u^\top H v}^2\ge \sqrt{t}\normtwo{H v}^2}\le \exp(-\sqrt{t}/2).
	\end{align}
	By Corollary~\ref{cor:hason-wright} and basic algebra we get,
	\begin{align}\label{equ:tail-lemma-2}
		\Pr\bbra{\normtwo{H v}^2\ge \sqrt{t}\normF{H}^2}\le 2\exp(-\sqrt{t}/4).
	\end{align}
	Consequently,
	\begin{align*}
		&\E\bbra{\ind{ (u^\top H v)^2\ge t\normF{H}^2 }}\\
		\le\;& \E\bbra{\ind{ (u^\top H v)^2\ge \sqrt{t}\normtwo{H v}^2\text{ or } \normtwo{H v}^2\ge \sqrt{t}\normF{H}^2}}\\
		\le\;& \E\bbra{\ind{ (u^\top H v)^2\ge \sqrt{t}\normtwo{H v}^2}\mid v}+\E\bbra{ \ind{\normtwo{H v}^2\ge \sqrt{t}\normF{H}^2}}\\
		\le\;& 3\exp(-\sqrt{t}/4).\tag{Combining Eq.~\eqref{equ:tail-lemma-1} and Eq.~\eqref{equ:tail-lemma-2}}
	\end{align*}
\end{proof}

The next two lemmas are dedicated to prove anti-concentration inequalities that is used in Lemma~\ref{lem:bandit-concentration}.
\begin{lemma}[Lemma 1 of \citet{laurent2000adaptive}]\label{lem:laurent-massart} Let $(y_1, \cdots, y_n)$ be i.i.d. $\calN(0,1)$ Gaussian variables. Let $a=(a_1,\cdots,a_n)$ be non-negative coefficient. Let 
	$$\normtwo{a}^2=\sum_{i=1}^{n}a_i^2.$$ Then for any positive $t$,
	\begin{align}
		\Pr\pbra{\sum_{i=1}^{n}a_i y_i^2\le \sum_{i=1}^{n}a_i-2\normtwo{a}\sqrt{t}}\le \exp(-t).
	\end{align}
\end{lemma}

\begin{lemma}\label{lem:square-bilinear-anti-concentration}
	Given a symmetric matrix $H\in \R^{n\times n}$, let $u,v\sim \calN(0,I)$ be two independent random vectors. Then
	\begin{align}
		\Pr\bbra{(u^\top H v)^2\ge \frac{1}{8}\normF{H}^2}\ge \frac{1}{64}.
	\end{align}
\end{lemma}
\begin{proof}
	Since $u,v$ are independent, by the isotropy of Guassian vectors we can assume that $H=\diag(\lambda_1,\cdots,\lambda_n).$ Note that condition on $v$, $u^\top H v$ is a Gaussian random variable with mean zero and variance $\normtwo{H v}^2.$ As a result, 
	\begin{align}\label{equ:anti-concentration-1}
		\forall v, \Pr\bbra{\pbra{u^\top H v}^2\ge \frac{1}{4}\normtwo{H v}^2\mid v}\ge \frac{1}{2}.
	\end{align}
	On the other hand, $\normtwo{H v}^2=\sum_{i=1}^{n}\lambda_i^2v_i^2.$ Invoking Lemma~\ref{lem:laurent-massart} we have
	\begin{align}
		&\Pr\bbra{\normtwo{H v}^2\ge \frac{1}{2}\normF{H}^2}\nonumber \\
		\ge\; &\Pr\bbra{\normtwo{H v}^2\ge \normF{H}^2-\frac{1}{2}\sqrt{\sum_{i=1}^{n}\lambda_i^4}}\nonumber \\
		=\; &\Pr\bbra{\sum_{i=1}^{n}\lambda_i^2v_i^2\ge \sum_{i=1}^{n}\lambda_i^2-\frac{1}{2}\sqrt{\sum_{i=1}^{n}\lambda_i^4}}\tag{By definition} \\
		\ge\; &1-\exp(-1/16)\ge \frac{1}{32}.\label{equ:anti-concentration-2}
	\end{align}
	Combining Eq.~\eqref{equ:anti-concentration-1} and Eq.~\eqref{equ:anti-concentration-2} we get,
	\begin{align*}
		\Pr\bbra{(u^\top H v)^2\ge \frac{1}{8}\normF{H}^2}
		\ge\; \Pr\bbra{(u^\top H v)^2\ge \frac{1}{4}\normtwo{H v}^2,\normF{H v}^2\ge \frac{1}{2}\normF{H}^2}
		\ge\; \frac{1}{64}.
	\end{align*}
\end{proof}

The following lemma justifies the cap in the loss function.
\begin{lemma}\label{lem:concentration-matrix-min}
	Given a symmetric matrix $H$, let $u,v\sim \calN(0,I)$ be two independent random vectors. Let $\ubh,c_1\in \R_+$ be two numbers satisfying $\ubh\ge 640\sqrt{2}c_1$, then
	\begin{align}\label{equ:concentration-matrix-min}
		\mintwo{c_1^2}{\normF{H}^2}\le 2\E\bbra{\mintwo{\ubh^2}{\pbra{u^\top H v}^2}}.
	\end{align}
\end{lemma}
\begin{proof}
	Let $x=\pbra{u^\top H v}^2$ for simplicity. Consider the following two cases:
	
	\paragraph{Case 1:} $\normF{H}\le \ubh/40.$ In this case we exploit the tail bound of random variable $x$. Specifically,
	\begin{align*}
		&\E\bbra{\pbra{u^\top H v}^2}-\E\bbra{\mintwo{\ubh^2}{\pbra{u^\top H v}^2}}\\
		=\; &\int_{\ubh^2}^{\infty}\Pr\bbra{x\ge t}dt\\
		\le\; &3\int_{\ubh^2}^{\infty}\exp\pbra{-\frac{1}{4}\sqrt{\frac{t}{\normF{H}^2}}}dt\tag{By Lemma~\ref{lem:square-bilinear-tail-bound}}\\
		=\; &24\exp\pbra{-\frac{\ubh}{4\normF{H}}}\normF{H}\pbra{\ubh+4\normF{H}}\\
		\le\; &48\exp\pbra{-\frac{\ubh}{4\normF{H}}}\normF{H}\ubh\tag{$4\normF{H}\le \ubh$ in this case}\\
		\le\; &48\cdot\frac{4\normF{H}}{384\ubh}\normF{H}\ubh\tag{$\exp(-x)\le \frac{1}{384x}$ when $x\ge 10$}\\
		\le\; &\frac{\normF{H}^2}{2}.
	\end{align*}
	As a result,
	\begin{align}
		\E\bbra{\mintwo{\ubh^2}{\pbra{u^\top H v}^2}}\ge \E\bbra{\pbra{u^\top H v}^2}-\frac{\normF{H}^2}{2}=\frac{\normF{H}^2}{2}.
	\end{align}
	
	\paragraph{Case 2:} $\normF{H}> \ubh/40.$ In this case, we exploit the anti-concentration result of random variable $x$. Note that by the choice of $\ubh$, we have $$\normF{H}> \ubh/40\implies \frac{1}{8}\normF{H}^2\ge 64c_1^2.$$ As a result,
	\begin{align*}
		&\E\bbra{\mintwo{\ubh^2}{\pbra{u^\top H v}^2}}\\
		\ge\; &64c_1^2\Pr\bbra{\mintwo{\ubh^2}{\pbra{u^\top H v}^2}\ge 64c_1^2}\\
		\ge\; &64c_1^2\Pr\bbra{\pbra{u^\top H v}^2\ge 64c_1^2}\tag{By definition of $\ubh$}\\
		\ge\; &64c_1^2\Pr\bbra{\pbra{u^\top H v}^2\ge \frac{1}{8}\normF{H}^2}\\
		\ge\; &c_1^2.\tag{By Lemma~\ref{lem:square-bilinear-anti-concentration}}
	\end{align*}
	
	Therefore, in both cases we get 
	\begin{align}
		\E\bbra{\mintwo{\ubh^2}{\pbra{u^\top H v}^2}}\ge \frac{1}{2}\mintwo{c_1^2}{\normF{H}^2},
	\end{align}
	which proofs Eq.~\eqref{equ:concentration-matrix-min}.
\end{proof}

Following lemmas are analogs to Cauchy-Schwartz inequality (in vector/matrix forms), which are used to prove Lemma~\ref{lem:rl-concentration} for reinforcement learning case.
\begin{lemma}\label{lem:holder-vector}
	For a random vector $x\in \R^{d}$ and random variable $r$, we have
	\begin{align}
		\normtwo{\E\bbra{rx}}^2\le \normop{\E\bbra{xx^\top}}\E\bbra{r^2}.
	\end{align}
\end{lemma}
\begin{proof}
	Note that for any vector $g\in \R^{d}$, $\normtwo{g}^2=\sup_{u\in S^{d-1}}\dotp{u}{g}^2.$ As a result,
	\begin{align*}
		&\normtwo{\E\bbra{rx}}^2=\sup_{u\in S^{d-1}}\dotp{u}{\E\bbra{rx}}^2=\sup_{u\in S^{d-1}}\E\bbra{r\dotp{u}{x}}^2\\
		\le &\;\sup_{u\in S^{d-1}}\E\bbra{\dotp{u}{x}^2}\E\bbra{r^2}\tag{H\"{o}lder Ineqaulity}\\
		= &\;\normop{\E\bbra{xx^\top}}\E\bbra{r^2}.
	\end{align*}
\end{proof}

\begin{lemma}\label{lem:holder-matrix}
For a symmetric random matrix $H\in \R^{d\times d}$ and random variable $r$, we have
\begin{align}
	\normsp{\E\bbra{rH}}^2\le \normsp{\E\bbra{HH^\top}}\E\bbra{r^2}.
\end{align}
\end{lemma}
\begin{proof}
Note that for any matrix $\mG\in \R^{d}$, $\normsp{H}^2=\sup_{u,v\in S^{d-1}}\pbra{u^\top\mG v}^2.$ As a result,
\begin{align*}
	&\normtwo{\E\bbra{rH}}^2=\sup_{u,v\in S^{d-1}}\pbra{u^\top\E\bbra{rH}v}^2=\sup_{u,v\in S^{d-1}}\E\bbra{r\pbra{u^\top H v}}^2\\
	\le &\;\sup_{u,v\in S^{d-1}}\E\bbra{\pbra{u^\top H v}^2}\E\bbra{r^2}\tag{H\"{o}lder Ineqaulity}\\
	= &\sup_{u,v\in S^{d-1}}\E\bbra{u^\top H v v^\top H^\top u}\E\bbra{r^2}\\
	\le &\sup_{u\in S^{d-1}}\E\bbra{u^\top H H^\top u}\E\bbra{r^2}\\
	= &\normsp{\E\bbra{H H^\top}}\E\bbra{r^2}.
\end{align*}
\end{proof}

\begin{lemma}\label{lem:holder-2tensor}
	For a random matrix $x\in \R^{d}$ and a positive random variable $r$, we have
	\begin{align}
		\normsp{\E\bbra{rxx^\top}}^2\le \normsp{\E[x^{\otimes 4}]}\E\bbra{r^2}.
	\end{align}
\end{lemma}
\begin{proof}
	Since $r$ is non-negative, we have $\E\bbra{rxx^\top}\succeq 0.$ As a result, $$\normsp{\E\bbra{rxx^\top}}=\sup_{u\in S^{d-1}}u^\top \E\bbra{rxx^\top} u.$$ It follows that 
	\begin{align*}
		&\normsp{\E\bbra{rxx^\top}}^2=\sup_{u\in S^{d-1}}\pbra{u^\top\E\bbra{rxx^\top}u}^2=\sup_{u\in S^{d-1}}\E\bbra{r\dotp{u}{x}^2}^2\\
		\le &\;\sup_{u\in S^{d-1}}\E\bbra{\dotp{u}{x}^4}\E\bbra{r^2}\tag{H\"{o}lder Inequality}\\
		= &\sup_{u\in S^{d-1}}\dotp{u^{\otimes 4}}{\E[x^{\otimes 4}]}\E\bbra{r^2}\\
		= &\normsp{\E\bbra{x^{\otimes 4}}}\E\bbra{r^2}.
	\end{align*}
\end{proof}

Following lemmas exploit the isotropism of Gaussian vectors, and are used to verify the Lipschitzness assumption of Example~\ref{example:1}. In fact, we heavily rely on the fact that, for a fixed vector $g\in \R^{d}$, $\dotp{g}{u}\sim \calN(0,\normtwo{g}^2)$ when $u\sim \calN(0,I),$
\begin{lemma}\label{lem:lip-1-bounded}
	For two vectors $p,q\in \R^{d}$ and a bounded function $f:\R^{d}\to [-B,B],$ we have
	\begin{align}
		\E_{u\sim \calN(0,\sigma^2 I)}\bbra{\pbra{\dotp{p}{u}-\dotp{q}{u}}f(u)}\le \sigma B\normtwo{p-q}.
	\end{align} 
\end{lemma}
\begin{proof}
	By H\"{o}lder inequality we have 
	\begin{align}
		&\E_{u\sim \calN(0,\sigma^2 I)}\bbra{\pbra{\dotp{p}{u}-\dotp{q}{u}}f(u)}\\
		\le&\;\E_{u\sim \calN(0,\sigma^2 I)}\bbra{\pbra{\dotp{p}{u}-\dotp{q}{u}}^2}^{1/2}\E_{u\sim \calN(0,\sigma^2 I)}\pbra{f(u)^2}^{1/2}.
	\end{align}
	Note that $u$ is isotropic. As a result $\dotp{p-q}{u}\sim \calN(0,\sigma^2\normtwo{p-q}^2).$ It follows that 
	\begin{align}
		&\E_{u\sim \calN(0,\sigma^2 I)}\bbra{\pbra{\dotp{p}{u}-\dotp{q}{u}}^2}^{1/2}\E_{u\sim \calN(0,\sigma^2 I)}\pbra{f(u)^2}^{1/2}\\
		\le\;&\sigma\normtwo{p-q}B.
	\end{align}
\end{proof}

\begin{lemma}\label{lem:helper-reduction}
	For two vectors $x,y\in \R^{d}$, if $\normtwo{x}=1$ we have
	\begin{align}
		\normtwo{x-y}^2\ge (1-\dotp{x}{y})^2.
	\end{align}
\end{lemma}
\begin{proof}
	By basic algebra we get
	\begin{align}
		&\normtwo{x-y}^2=\normtwo{x-\dotp{x}{y}x+\dotp{x}{y}x-y}^2\\
		=&\normtwo{x-\dotp{x}{y}x}^2+\normtwo{\dotp{x}{y}x-y}^2-2(1-\dotp{x}{y})\dotp{x}{\dotp{x}{y}x-y}\\
		=&\normtwo{x-\dotp{x}{y}x}^2+\normtwo{\dotp{x}{y}x-y}^2\\
		= &(1-\dotp{x}{y})^2+\normtwo{\dotp{x}{y}x-y}^2\ge (1-\dotp{x}{y})^2.
	\end{align}
\end{proof}

\begin{lemma}\label{lem:lip-1-coupling}
For vectors $p,x_1,x_2\in \R^{d}$ and a bounded function $f:\R^{d}\to [0,B],$ we have
\begin{align}
	\E_{u\sim \calN(0,\sigma^2 I)}\bbra{\dotp{p}{u}\pbra{f(x_1+u)-f(x_2+u)}}\le B\normtwo{p}\normtwo{x_1-x_2}\pbra{6+\frac{3(\normtwo{x_1}+\normtwo{x_2})}{\sigma}}.
\end{align} 
\end{lemma}
\begin{proof}
	The lemma is proved by coupling argument. With out loss of generality, we assume that $x_1=C e_1$ and $x_2=-Ce_1$ where $e_1$ is the first basis vector. That is, $\normtwo{x_1-x_2}=2C.$ For a vector $x\in \R^{d}$, let $F(x)$ be the density of distribution $\calN(0,\sigma^2 I)$ at $x$. Then we have,
	\begin{align}
		\E_{u\sim \calN(0,\sigma^2 I)}\bbra{\dotp{p}{u}f(x_1+u)}=\int_{y\in \R^{d}}F(y-x_1)\dotp{p}{y-x_1}f(y)\dd y.
	\end{align}
	As a result,
	\begin{align}
		&\E_{u\sim \calN(0,\sigma^2 I)}\bbra{\dotp{p}{u}\pbra{f(x_1+u)-f(x_2+u)}}\\
		=&\;\int_{y\in \R^{d}}F(y-x_1)\dotp{p}{y-x_1}f(y)\dd y-F(y-x_2)\dotp{p}{y-x_2}f(y)\dd y.
	\end{align}
	Define $G(y)=\min(F(y-x_1),F(y-x_2)).$ It follows that,
	\begin{align}
		&\int_{y\in \R^{d}}F(y-x_1)\dotp{p}{y-x_1}f(y)\dd y-F(y-x_2)\dotp{p}{y-x_2}f(y)\dd y\\
		\le&\; \int_{y\in \R^{d}}G(y)\abs{\dotp{p}{y-x_1}-\dotp{p}{y-x_2}}f(y)\dd y \label{equ:d11-1}\\ 
		&\;+\int_{y\in \R^{d}}(F(y-x_1)-G(y))\abs{\dotp{p}{y-x_1}}f(y)\dd y \label{equ:d11-2}\\ 
		&\;+\int_{y\in \R^{d}}(F(y-x_2)-G(y))\abs{\dotp{p}{y-x_2}}f(y)\dd y. \label{equ:d11-3}
	\end{align}
	The term in Eq.~\eqref{equ:d11-1} can be bounded by 
	\begin{align}
		&\int_{y\in \R^{d}}G(y)\abs{\dotp{p}{y-x_1}-\dotp{p}{y-x_2}}f(y)\dd y\\
		\le\; &\int_{y\in \R^{d}}G(y)\dd y \sup_{y\in \R^{d}}\abs{\dotp{p}{x_2-x_1}f(y)}\le\normtwo{x_2-x_1}\normtwo{p}B.
	\end{align}
	Note that the terms in Eq.~\eqref{equ:d11-2} and Eq.~\eqref{equ:d11-3} are symmetric. Therefore in the following we only prove an upper bound for Eq.~\eqref{equ:d11-2}. In the following, we use the notation $[y]_{-1}$ to denote the $(d-1)$-dimensional vector generated by removing the first coordinate of $y$. Let $P(x)$ be the density of distribution $\calN(0,\sigma^2)$ at point $x\in \R$. By the symmetricity of Gaussian distribution, $F(y)=P([y]_1)F([y]_{-1}).$
	
	By definition, $F(y-x_1)-G(y)=0$ for $y$ such that $[y]_1\le 0.$ As a result, 
	\begin{align}
		&\int_{y:[y]_1>0}(F(y-x_1)-G(y))\abs{\dotp{p}{y-x_1}}f(y)\dd y\\
		=\;&\int_{y:[y]_1>0}(F(y-x_1)-F(y-x_2))\abs{\dotp{p}{y-x_1}}f(y)\dd y\\
		\le\;&\int_{[y]_1>0}\dd [y]_1(P([y]_1-[x_1]_1)-P([y_1]-[x_2]_1))\E_{[y]_{-1}}\bbra{\pbra{\abs{\dotp{[p]_{-1}}{[y-x_1]_{-1}}+[p]_1[y-x_1]_1}}f(y)}.
	\end{align}
	Note that conditioned on $[y]_1$, $[y-x_1]_{-1}\sim \calN(0,\sigma^2 I).$ Consequently,
	\begin{align}
		&\E_{[y]_{-1}}\bbra{\abs{\dotp{[p]_{-1}}{[y-x_1]_{-1}}}f(y)}\\
		\le\;&\E_{[y]_{-1}}\bbra{\dotp{[p]_{-1}}{[y-x_1]_{-1}}^2}^{1/2}\E_{[y]_{-1}}\bbra{f(y)^2}^{1/2}\\
		\le\;&B\sigma\normtwo{[p]_{-1}}\le B\sigma\normtwo{p}.
	\end{align}
	On the other hand, we have 
	\begin{align}\label{equ:d11-TV}
		\int_{[y]_1>0}\dd [y]_1(P([y]_1-[x_1]_1)-P([y_1]-[x_2]_1))\le \TV{\calN(-C,\sigma^2)}{\calN(C,\sigma^2)}\le \frac{1}{\sigma}\normtwo{x_1-x_2}.
	\end{align} It follows that,
	\begin{align}
		&\int_{y:[y]_1>0}(F(y-x_1)-G(y))\abs{\dotp{p}{y-x_1}}f(y)\dd y\\
		\le\;&B\normtwo{p}\normtwo{x_1-x_2}+B\int_{[y]_1>0}\dd [y]_1(P([y]_1-[x_1]_1)-P([y_1]-[x_2]_1))\abs{[p]_1[y-x_1]_1}.\label{equ:d11-4}
	\end{align}
	Note that the second term in Eq.~\eqref{equ:d11-4} involves only one dimensional Gaussian distribution. Invoking Lemma~\ref{lem:lip-1-1dim}, the second term can be bounded by $\normtwo{p}\pbra{\frac{3\normtwo{x_1-x_2}}{\sigma}\normtwo{x_1}+4\normtwo{x_1-x_2}}$. Therefore, we have
	\begin{align}
		&\int_{y:[y]_1>0}(F(y-x_1)-G(y))\abs{\dotp{p}{y-x_1}}f(y)\dd y \le 5B\normtwo{p}\normtwo{x_1-x_2}+\normtwo{p}\frac{3B\normtwo{x_1-x_2}}{\sigma}.
	\end{align}
	
\end{proof}

\begin{lemma}\label{lem:lip-2-coupling}
	For vectors $p,q,x_1,x_2\in \R^{d}$ with $\norm{x_1\le 1}\le 1\normtwo{x_2}\le 1$ and a bounded function $f:\R^{d}\to [0,B],$ we have
	\begin{align}
		&\E_{u\sim \calN(0,\sigma^2 I)}\bbra{\dotp{p}{u}\dotp{q}{u}\pbra{f(x_1+u)-f(x_2+u)}}\le \poly(B,\sigma,1/\sigma,\normtwo{p},\normtwo{q})\normtwo{x_1-x_2}.
	\end{align} 
\end{lemma}
\begin{proof}
	Proof of this lemma is similar to that of Lemma~\ref{lem:lip-2-coupling}. With out loss of generality, we assume that $x_1=C e_1$ and $x_2=-Ce_1$ where $e_1$ is the first basis vector. That is, $\normtwo{x_1-x_2}=2C.$ For a vector $x\in \R^{d}$, let $F(x)$ be the density of distribution $\calN(0,\sigma^2 I)$ at $x$. Then we have,
	\begin{align}
		\E_{u\sim \calN(0,\sigma^2 I)}\bbra{\dotp{p}{u}\dotp{q}{u}f(x_1+u)}=\int_{y\in \R^{d}}F(y-x_1)\dotp{p}{y-x_1}\dotp{q}{y-x_1}f(y)\dd y.
	\end{align}
	As a result,
	\begin{align}
		&\E_{u\sim \calN(0,\sigma^2 I)}\bbra{\dotp{p}{u}\dotp{q}{u}\pbra{f(x_1+u)-f(x_2+u)}}\\
		=&\;\int_{y\in \R^{d}}F(y-x_1)\dotp{p}{y-x_1}\dotp{q}{y-x_1}f(y)\dd y-F(y-x_2)\dotp{p}{y-x_2}\dotp{q}{y-x_2}f(y)\dd y.
	\end{align}
	Define $G(y)=\min(F(y-x_1),F(y-x_2)).$ It follows that,
	\begin{align}
		&\int_{y\in \R^{d}}F(y-x_1)\dotp{p}{y-x_1}\dotp{q}{y-x_1}f(y)\dd y-F(y-x_2)\dotp{p}{y-x_2}\dotp{q}{y-x_2}f(y)\dd y\\
		\le&\; \int_{y\in \R^{d}}G(y)\abs{\dotp{p}{y-x_1}\dotp{q}{y-x_1}-\dotp{p}{y-x_2}\dotp{q}{y-x_2}}f(y)\dd y \label{equ:d12-1}\\ 
		&\;+\int_{y\in \R^{d}}(F(y-x_1)-G(y))\abs{\dotp{p}{y-x_1}\dotp{q}{y-x_1}}f(y)\dd y \label{equ:d12-2}\\ 
		&\;+\int_{y\in \R^{d}}(F(y-x_2)-G(y))\abs{\dotp{p}{y-x_2}\dotp{q}{y-x_2}}f(y)\dd y. \label{equ:d12-3}
	\end{align}
	By basic algebra we have
	\begin{align}
		&\int_{y\in \R^{d}}G(y)\abs{\dotp{p}{y-x_1}\dotp{q}{y-x_1}-\dotp{p}{y-x_2}\dotp{q}{y-x_2}}f(y)\dd y\\
		\le&\;\int_{y\in \R^{d}}G(y)\abs{\dotp{p}{x_2-x_1}\dotp{q}{y-x_1}}f(y)\dd y\\
		&+\int_{y\in \R^{d}}G(y)\abs{\dotp{p}{y-x_2}\dotp{q}{x_2-x_1}}f(y)\dd y.\label{equ:d12-4}
	\end{align}
	Continue with the first term we get
	\begin{align}
		&\int_{y\in \R^{d}}G(y)\abs{\dotp{p}{x_2-x_1}\dotp{q}{y-x_1}}f(y)\dd y\\
		\le&\;\normtwo{x_2-x_1}\normtwo{p}\int_{y\in \R^{d}}G(y)\abs{\dotp{q}{y-x_1}}f(y)\dd y\\
		\le&\;\normtwo{x_2-x_1}\normtwo{p}\int_{y\in \R^{d}}F(y-x_1)\abs{\dotp{q}{y-x_1}}f(y)\dd y\\
		=&\;\normtwo{x_2-x_1}\normtwo{p}\E_{u\sim \calN(0,\sigma^2 I)}\bbra{\abs{\dotp{q}{u}}f(u+x_1)}\\
		\le&\;\normtwo{x_2-x_1}\normtwo{p}\E_{u\sim \calN(0,\sigma^2 I)}\bbra{\dotp{q}{u}^2}^{1/2}\E_{u\sim \calN(0,\sigma^2 I)}\bbra{f(u+x_1)^{2}}^{1/2}\\
		\le&\;\sigma B\normtwo{x_2-x_1}\normtwo{p}\normtwo{q}.
	\end{align}
	For the same reason, the second term in Eq.~\eqref{equ:d12-4} is also bounded by $\sigma B\normtwo{x_2-x_1}\normtwo{p}\normtwo{q}.$ As a result, 
	\begin{align}
		\int_{y\in \R^{d}}G(y)\abs{\dotp{p}{y-x_1}\dotp{q}{y-x_1}-\dotp{p}{y-x_2}\dotp{q}{y-x_2}}f(y)\dd y\le 2\sigma B\normtwo{x_2-x_1}\normtwo{p}\normtwo{q}.
	\end{align}

	Now we turn to the term in Eq.~\eqref{equ:d12-2}. Note that the terms in Eq.~\eqref{equ:d12-2} and Eq.~\eqref{equ:d12-3} are symmetric. Therefore in the following we only prove an upper bound for Eq.~\eqref{equ:d12-2}. In the following, we use the notation $[y]_{-1}$ to denote the $(d-1)$-dimensional vector generated by removing the first coordinate of $y$. Let $P(x)$ be the density of distribution $\calN(0,\sigma^2)$ at point $x\in \R$. By the symmetricity of Gaussian distribution, $F(y)=P([y]_1)F([y]_{-1}).$
	
	By definition, $F(y-x_1)-G(y)=0$ for $y$ such that $[y]_1\le 0.$ Define the shorthand $I=\abs{[p]_1[y-x_1]_1},J=\abs{\dotp{[p]_{-1}}{[y-x_1]_{-1}}},C=\abs{[q]_1[y-x_1]_1},D=\abs{\dotp{[q]_{-1}}{[y-x_1]_{-1}}}$. 
	When condition on $[y]_1$, $A,C$ are constants. As a result, 
	\begin{align}
		&\int_{y:[y]_1>0}(F(y-x_1)-G(y))\abs{\dotp{p}{y-x_1}\dotp{q}{y-x_1}}f(y)\dd y\\
		=\;&\int_{y:[y]_1>0}(F(y-x_1)-F(y-x_2))\abs{\dotp{p}{y-x_1}\dotp{q}{y-x_1}}f(y)\dd y\\
		\le\;&\int_{[y]_1>0}\dd [y]_1(P([y]_1-[x_1]_1)-P([y_1]-[x_2]_1)) \E_{[y]_{-1}}\bbra{\abs{(I+J)(C+D)}f(y)}\\
		\le\;&\int_{[y]_1>0}\dd [y]_1(P([y]_1-[x_1]_1)-P([y_1]-[x_2]_1)) IC\E_{[y]_{-1}}\bbra{f(y)}\\
		&+\int_{[y]_1>0}\dd [y]_1(P([y]_1-[x_1]_1)-P([y_1]-[x_2]_1)) I\E_{[y]_{-1}}\bbra{Df(y)}\\
		&+\int_{[y]_1>0}\dd [y]_1(P([y]_1-[x_1]_1)-P([y_1]-[x_2]_1)) C\E_{[y]_{-1}}\bbra{Jf(y)}\\
		&+\int_{[y]_1>0}\dd [y]_1(P([y]_1-[x_1]_1)-P([y_1]-[x_2]_1)) \E_{[y]_{-1}}\bbra{JDf(y)}.
	\end{align}
	Note that conditioned on $[y]_1$, $[y-x_1]_{-1}\sim \calN(0,\sigma^2 I).$ Consequently,
	\begin{align}
		&\E_{[y]_{-1}}\bbra{Df(y)}\le\;\E_{[y]_{-1}}\bbra{D^2}^{1/2}\E_{[y]_{-1}}\bbra{f(y)^2}^{1/2}\le\;B\sigma\normtwo{q}\\
		&\E_{[y]_{-1}}\bbra{Jf(y)}\le\;\E_{[y]_{-1}}\bbra{J^2}^{1/2}\E_{[y]_{-1}}\bbra{f(y)^2}^{1/2}\le\;B\sigma\normtwo{p}\\
		&\E_{[y]_{-1}}\bbra{JDf(y)}\le\;\E_{[y]_{-1}}\bbra{J^4}^{1/4}\E_{[y]_{-1}}\bbra{D^4}^{1/4}\E_{[y]_{-1}}\bbra{f(y)^2}^{1/2}\le\;\sqrt{3}B\sigma^2\normtwo{p}\normtwo{q}.
	\end{align}
	Invoking Lemma~\ref{lem:lip-1-1dim} we get
	\begin{align}
		&\int_{[y]_1>0}\dd [y]_1(P([y]_1-[x_1]_1)-P([y_1]-[x_2]_1)) \le \frac{1}{\sigma}\normtwo{x_1-x_2},\\
		&\int_{[y]_1>0}\dd [y]_1(P([y]_1-[x_1]_1)-P([y_1]-[x_2]_1))I \le \poly(B,\sigma,1/\sigma)\normtwo{p}\normtwo{x_1-x_2},\\
		&\int_{[y]_1>0}\dd [y]_1(P([y]_1-[x_1]_1)-P([y_1]-[x_2]_1))C \le \poly(B,\sigma,1/\sigma)\normtwo{q}\normtwo{x_1-x_2},\\
		&\int_{[y]_1>0}\dd [y]_1(P([y]_1-[x_1]_1)-P([y_1]-[x_2]_1))IC \le \poly(B,\sigma,1/\sigma)\normtwo{p}\normtwo{q}\normtwo{x_1-x_2}.
	\end{align}
	As a result, we get
	\begin{align*}
		&\int_{y:[y]_1>0}(F(y-x_1)-G(y))\abs{\dotp{p}{y-x_1}\dotp{q}{y-x_1}}f(y)\dd y\le \poly(B,\sigma,1/\sigma,\normtwo{p},\normtwo{q})\normtwo{x_1-x_2}.
	\end{align*}
	
\end{proof}

\begin{lemma}\label{lem:lip-1-1dim}
	Let $P(x)$ be the density function of $\calN(0,\sigma^2)$. Given a scalar $x\ge 0$. we have
	\begin{align}
		&\int_{y>0}(P(y-x)-P(y+x))\dd y\le \frac{x}{\sigma},\\
		&\int_{y>0}(P(y-x)-P(y+x))\abs{y-x}\dd y\le \frac{3x^2}{\sigma}+4x,\\
		&\int_{y>0}(P(y-x)-P(y+x))\abs{y-x}^2\dd y\le \frac{x^2}{\sigma}+4\sigma x.
	\end{align}
\end{lemma}
\begin{proof}
	Note that $\TV{\calN(-x,\sigma^2)}{\calN(x,\sigma^2)}\le \frac{x}{\sigma}.$ Consequently, 
	\begin{align}
		\int_{y>0}(P(y-x)-P(y+x))\dd y\le \frac{x}{\sigma}.
	\end{align}
	Using the same TV-distance bound, we get
	\begin{align}
		&\int_{y>0}(P(y-x)-P(y+x))\abs{y-x}\dd y\\
		\le &\int_{y>0}(P(y-x)-P(y+x))((y-x)+2x)\dd y\\
		\le &\int_{y>2x}(P(y-x)-P(y+x))(y-x)\dd y+\frac{3x^2}{\sigma}.
	\end{align}
	Recall $P(x)=\frac{1}{\sqrt{2\pi}}\exp\pbra{-\frac{x^2}{2\sigma^2}}.$ By algebraic manipulation we have
	\begin{align}
		&\int_{y>2x}\pbra{\exp\pbra{-\frac{(y-x)^2}{2\sigma^2}}-\exp\pbra{-\frac{(y+x)^2}{2\sigma^2}}}(y-x)\dd y\\
		\le\;&\int_{y>2x}\exp\pbra{-\frac{(y-x)^2}{2\sigma^2}}\pbra{1-\exp\pbra{-\frac{4xy}{2\sigma^2}}}(y-x)\dd y\\
		\le\;&\int_{y>2x}\exp\pbra{-\frac{(y-x)^2}{2\sigma^2}}\frac{4xy}{2\sigma^2}(y-x)\dd y\\
		\le\;&\int_{y>2x}\exp\pbra{-\frac{(y-x)^2}{2\sigma^2}}\frac{4x}{\sigma^2}(y-x)^2\dd y\\
		\le\;&4x.
	\end{align}
	Now we turn to the third inequality. 
	Because $\abs{y-x}\le x$ for $y\in[0,2x]$, using the TV-distance bound we get
	\begin{align}
		\int_{y>0}(P(y-x)-P(y+x))\abs{y-x}^2\dd y\le \int_{y>2x}(P(y-x)-P(y+x))(y-x)^2\dd y+\frac{x^2}{\sigma}.
	\end{align}
	By algebraic manipulation we have
	\begin{align}
		&\int_{y>2x}\pbra{\exp\pbra{-\frac{(y-x)^2}{2\sigma^2}}-\exp\pbra{-\frac{(y+x)^2}{2\sigma^2}}}(y-x)^2\dd y\\
		\le\;&\int_{y>2x}\exp\pbra{-\frac{(y-x)^2}{2\sigma^2}}\pbra{1-\exp\pbra{-\frac{4xy}{2\sigma^2}}}(y-x)^2\dd y\\
		\le\;&\int_{y>2x}\exp\pbra{-\frac{(y-x)^2}{2\sigma^2}}\frac{4xy}{2\sigma^2}(y-x)^2\dd y\\
		\le\;&\int_{y>2x}\exp\pbra{-\frac{(y-x)^2}{2\sigma^2}}\frac{4x}{\sigma^2}(y-x)^3\dd y\\
		\le\;&4\sigma x.
	\end{align}
\end{proof}

\begin{lemma}\label{lem:lip-2-bounded}
	For four vectors $p,q,v,w\in \R^{d}$ with unit norm and a bounded function $f:\R^{d}\to [-B,B],$ we have
	\begin{align}
		\E_{u\sim \calN(0,\sigma^2 I)}\bbra{\pbra{\dotp{p}{u}\dotp{v}{u}-\dotp{q}{u}\dotp{w}{u}}f(u)}\le \sqrt{3}\sigma^2 B\pbra{\normtwo{p-q}+\normtwo{v-w}}.
	\end{align} 
\end{lemma}
\begin{proof}
	First of all, by telescope sum we get
	\begin{align}
		&\E_{u\sim \calN(0,\sigma^2 I)}\bbra{\pbra{\dotp{p}{u}\dotp{v}{u}-\dotp{q}{u}\dotp{w}{u}}f(u)}\\
		=&\;\E_{u\sim \calN(0,\sigma^2 I)}\bbra{\pbra{\dotp{p}{u}\dotp{v}{u}-\dotp{p}{u}\dotp{w}{u}}f(u)}\\
		&+\E_{u\sim \calN(0,\sigma^2 I)}\bbra{\pbra{\dotp{p}{u}\dotp{w}{u}-\dotp{q}{u}\dotp{w}{u}}f(u)}.
	\end{align}
	By H\"{o}lder inequality we have 
	\begin{align}
		&\E_{u\sim \calN(0,\sigma^2 I)}\bbra{\pbra{\dotp{p}{u}\dotp{v}{u}-\dotp{p}{u}\dotp{w}{u}}f(u)}\\
		\le&\;\E_{u\sim \calN(0,\sigma^2 I)}\bbra{\dotp{p}{u}^{4}}^{1/4}\E_{u\sim \calN(0,\sigma^2 I)}\bbra{\dotp{v-w}{u}^4}^{1/4}\E_{u\sim \calN(0,\sigma^2 I)}\bbra{f(u)^2}^{1/2}\\
		\le&\;\sqrt{3}\sigma^2\normtwo{p}\normtwo{v-w}B.
	\end{align}
	Similarly we have
	\begin{align}
		\E_{u\sim \calN(0,\sigma^2 I)}\bbra{\pbra{\dotp{p}{u}\dotp{w}{u}-\dotp{q}{u}\dotp{w}{u}}f(u)}\le \sqrt{3}\sigma^2\normtwo{p-q}\normtwo{w}B.
	\end{align}
\end{proof}

\subsection{Helper Lemmas on Reinforcement Learning}
\begin{lemma}[Telescoping or Simulation Lemma, see \citet{luo2019algorithmic,agarwal2019reinforcement}]\label{lem:telescope}
	For any policy $\pi$ and deterministic dynamical model $\dy,\hat{\dy}$, we have
	\begin{align}
		V^{\pi}_{\hat{\dy}}(s_1)-V^\pi_{\dy}(s_1)=\E_{\tau\sim \rho^\pi_\dy}\bbra{\sum_{h=1}^{H}\pbra{V^{\pi}_{\hat{\dy}}(\hat{\dy}(s_h,a_h))-V^{\pi}_{\hat{\dy}}(\dy(s_h,a_h))}}.
	\end{align}
\end{lemma}

\begin{lemma}[Policy Gradient Lemma, see \citet{sutton2011reinforcement}]\label{lem:PG}
	For any policy $\pi_\psi$, deterministic dynamical model $\dy$ and reward function $r(s_h,a_h)$, we have
	\begin{align}
		\gpsi V^{\pi_\psi}_\dy=\E_{\tau\sim \rho^{\pi_\psi}_\dy}\bbra{\pbra{\sum_{h=1}^{H}\gpsi \log \pi_\psi(a_h\mid s_h)}\pbra{\sum_{h=1}^{H}r(s_h,a_h)}}
	\end{align}
\end{lemma}
\begin{proof} 
	Note that 
	\begin{align*}
		V_\dy^{\pi_\psi}=\int_{\tau}\Pr\bbra{\tau}\sum_{h=1}^{H}r(s_h,a_h)\;d\tau.
	\end{align*}
	Take gradient w.r.t. $\psi$ in both sides, we have
	\begin{align*}
		\gpsi V^{\pi_\psi}_\dy&=\gpsi \int_{\tau}\Pr\bbra{\tau}\sum_{h=1}^{H}r(s_h,a_h)\;d\tau\\
		&=\int_{\tau}\pbra{\gpsi\Pr\bbra{\tau}}\sum_{h=1}^{H}r(s_h,a_h)\;d\tau\\
		&=\int_{\tau}\Pr\bbra{\tau}\pbra{\gpsi\log\Pr\bbra{\tau}}\sum_{h=1}^{H}r(s_h,a_h)\;d\tau\\
		&=\int_{\tau}\Pr\bbra{\tau}\pbra{\gpsi\log\prod_{h=1}^{H}\pi_\psi(a_h\mid s_h)}\sum_{h=1}^{H}r(s_h,a_h)\;d\tau\\
		&=\int_{\tau}\Pr\bbra{\tau}\pbra{\sum_{h=1}^{H}\gpsi\log\pi_\psi(a_h\mid s_h)}\sum_{h=1}^{H}r(s_h,a_h)\;d\tau\\
		&=\E_{\tau\sim \rho_\dy^{\pi_\psi}}\bbra{\pbra{\sum_{h=1}^{H}\gpsi\log\pi_\psi(a_h\mid s_h)}\sum_{h=1}^{H}r(s_h,a_h)}\\
	\end{align*}
\end{proof}

\end{document}